\documentclass[10pt,twocolumn,journal]{IEEEtran}

\ifCLASSINFOpdf
\else
\fi

\usepackage{setspace}
\usepackage{cite}
\usepackage{graphicx}
\graphicspath{{../Images/}}

\usepackage[cmex10]{amsmath}

\usepackage{float, epstopdf}
\usepackage{verbatim}
\usepackage{relsize,bbm}
\usepackage{amsfonts}
\usepackage{amsthm}
\usepackage{amssymb}
\usepackage{latexsym}
\usepackage{url}
\usepackage{epstopdf}
\usepackage{enumerate}
\usepackage{mathrsfs}
\usepackage[bottom]{footmisc}
\usepackage{tabulary, hyperref}
\usepackage{color}
\usepackage{balance}
\usepackage{amsmath}
\usepackage{cases}

\usepackage{multirow}
\usepackage{caption}
\newcommand{\BREAK}{\STATE \textbf{break} }
\usepackage{thmtools} 
\usepackage{thm-restate}
\usepackage{mathtools}
\mathtoolsset{showonlyrefs}
\usepackage{microtype}
\usepackage{subfigure}
\usepackage{booktabs} 
\usepackage{enumitem}
\setitemize{label=\textbullet, leftmargin=*, nolistsep}
\newcommand{\var}{\mathrm{Var}}
\newcommand{\Hard}{\psi}

\usepackage{tikz}
\usepackage{tkz-euclide}
\allowdisplaybreaks[1]

\DeclareMathOperator*{\argmax}{arg\,max}

\newtheorem{lem}{Lemma}

\newtheorem{remark}{Remark}

\let\oldref\ref
\renewcommand{\ref}[1]{(\oldref{#1})}

\usepackage{algorithm}
\usepackage{algorithmic}
\usepackage{tikz}
\usepackage{tkz-euclide}

\usepackage{caption}
\usepackage{xcolor}

\allowdisplaybreaks

\begin{document}
	%
	\title{Almost Optimal Variance-Constrained Best Arm Identification}
	%
	%
	%
	
	\author{Yunlong Hou, Vincent Y.\ F.\ Tan, \emph{Senior Member, IEEE}, and Zixin Zhong
	\thanks{This research was supported by a Singapore National Research Foundation (NRF) Fellowship (A-0005077-01-00), and by two Singapore Ministry of
Education (MOE) AcRF Tier 1 Grants (A-0009042-01-00 and A-8000189-01-00).}
	\thanks{Y.~Hou is with the  Department of Mathematics, National University of Singapore, Singapore (email: yhou@u.nus.edu).}%
	\thanks{V.~Y.~F.~Tan is with the Department of Mathematics and the Department of Electrical and Computer Engineering, National University of Singapore, Singapore (email: vtan@nus.edu.sg).} 
	\thanks{Z.~Zhong ({\em Corresponding author}) is with the Department of Computing Science at the University of Alberta, Canada (email: zixin.zhong@u.nus.edu).}
	}

%
%

\markboth{}%
{Shell \MakeLowercase{\textit{et al.}}: Bare Demo of IEEEtran.cls for IEEE Journals}
%



\maketitle

\begin{abstract}
	We design and analyze {\em Variance-Aware-Lower and Upper Confidence Bound} (VA-LUCB), a parameter-free algorithm, for identifying the best arm under the fixed-confidence setup and under a stringent constraint that the variance of the chosen arm is strictly smaller than a given threshold. An upper bound on VA-LUCB's sample complexity is  shown to be characterized by a fundamental variance-aware hardness quantity $H_{\mathrm{VA}}$. By proving an information-theoretic lower bound, we show that sample complexity of VA-LUCB is optimal up to  a  factor logarithmic in   $H_{\mathrm{VA}}$. Extensive experiments corroborate the dependence of the sample complexity on the various terms in  $H_{\mathrm{VA}}$. By comparing VA-LUCB's empirical performance to a close competitor  RiskAverse-UCB-BAI by David {\em et al.}~\cite{David2018},   our experiments suggest that VA-LUCB has the lowest sample complexity for this class of risk-constrained best arm identification problems, especially for the riskiest instances. 
\end{abstract}

\begin{IEEEkeywords}
	Stochastic Multi-Armed Bandits, Best Arm Identification, Risk-aware Bandits.
\end{IEEEkeywords}

\section{Introduction}
The stochastic multi-armed bandit (MAB) problem~\cite{lattimore2020bandit} is a classical framework for online decision-making problems with extensive applications, e.g.,  clinical trials and financial portfolio.
In a conventional stochastic MAB problem, given several arms with each of them associated with a fixed but unknown reward distribution, an agent selects an arm and observes a random reward returned from the corresponding distribution at each round. There are two complementary tasks in MAB problems. Firstly, the {\em regret minimization problem} aims to maximize the expected cumulative reward. The second task, the main focus of the present paper, is the {\em best arm identification} or BAI problem that aims to devise a strategy to identify the arm with the largest expected reward.

While the expected reward is a key indication of the quality of an arm, its {\em risk} should also be taken into consideration, e.g., in clinical trials where the effects of experimental drugs exhibit variability over different individuals and in financial portfolio where conservative investors seek a beneficial and also safe product. Instead of pursuing the highest payoff, one may wish to mitigate the underlying  risks of certain arms by balancing between the reward and the potential risk.
Various measures of risk~\cite{Cassel2018,lee2020learning, chang2021unifying} have been adopted, such as the variance, Value-at-Risk (VaR or $ \alpha $-quantile), and the  Conditional Value-at-Risk (CVaR). 
We adopt the  variance as the  risk measure, but our techniques are {\em also}   applicable to other risk measures if suitable concentration bounds are available. 
We design and analyze the VA-LUCB algorithm which is shown to be almost optimal in terms of the sample complexity, and we identify a key fundamental hardness quantity $H_{\mathrm{VA}}$. VA-LUCB also significantly outperforms a suitably modified  algorithm of \cite{David2018}.

\subsection{Literature Review}
There are  three main families of algorithms for the standard fixed-confidence BAI problem---confidence bound-based (CBB) algorithms \cite{Even2006action,Audibert2010,jamieson14lil,Kalyanakrishnan2012},  tracking-based (TB) algorithms  \cite{Kaufmann2016}, and Bayesian-style (BS) algorithms \cite{russo2016simple}. Jamieson and Nowak~\cite{Jamieson2014} provide a comprehensive survey for   CBB algorithms, which includes the Action Elimination algorithm \cite{Even2006action}, the Upper Confidence Bound (UCB) algorithm \cite{Audibert2010} and the LUCB algorithm \cite{Kalyanakrishnan2012}.  While LUCB \cite{Kalyanakrishnan2012} is originally designed for top-$k$ arm identification, Jamieson and Nowak~\cite{Jamieson2014} claimed that LUCB-based methods perform well both theoretically and empirically for BAI task (thus we build our algorithm upon LUCB). LUCB samples the arm with the largest sample mean $i_t$ and another arm with the largest upper confidence bound $j_t$ within the remaining arms. It terminates when the lower confidence bound of $i_t$ is greater than the upper confidence bound of $j_t$. The family of LIL techniques \cite{jamieson14lil,howard2021time} which provide uniform (in time) bounds on the deviation of an empirical statistic from the true quantity can boost the performance of these CBB methods.  TB algorithms such as Track and Stop \cite{Kaufmann2016} track the proportion of arm pulls and achieves asymptotic optimality. BS algorithms such as Top-Two Thompson sampling \cite{russo2016simple} are easy to implement, asymptotically optimal, and yield good theoretical and empirical results.

For the risk-aware BAI problem,
there is a large body of  literature that measures the quality of an arm by {\em general functions of its distribution instead of the expectation}. 
The mean-variance paradigm is studied by \cite{Sani2013}, \cite{Vakili2016} and \cite{Zhu2020} under the regret minimization framework.
Sani {\em et al.}~\cite{Sani2013} regarded the variance as the measure of risk and proposed the MV-LCB algorithm. 
The  regret analysis of MV-LCB \cite{Sani2013} was  improved by \cite{Vakili2016}.
Zhu and Tan~\cite{Zhu2020} and Chang {\em et al.}~\cite{chang2021risk} proposed Thompson sampling-based algorithms that are optimal under different regimes for  the mean-variance and CVaR criteria respectively.
The mean-variance paradigm was generalized by~\cite{Zimin2014} where the quality of an arm is measured by some  functions of the mean and the variance. 
Another class of risk measures that is widely studied consists of the VaR and CVaR.
Under the BAI framework, Prashanth {\em et al.}~\cite{LA2020} adapted the successive rejects algorithm of \cite{Audibert2010} for optimizing the CVaR.  
Kagrecha {\em et al.}~\cite{Kagrecha2019} utilized a linear combination of the reward and the CVaR as the measure of quality of  the arms and relaxed the prior knowledge of the reward distribution; this was generalized recently to general risk measures~\cite{kagrecha2022statistically}. 
David and Shimkin~\cite{David2016} aimed at finding the arm with the maximum $ \alpha $-quantile.
Under the regret minimization framework, Kagrecha {\em et al.}~\cite{Kagrecha2020} and Baudry {\em et al.}~\cite{baudry2020thompson} regarded the CVaR as a risk measure and proposed the RC-LCB algorithm and  Thompson sampling-based algorithms respectively.
Other risk measures have  also been considered. 
For example, the {\em Sharpe ratio}, together with the mean-variance, was adopted by \cite{EvenDar2006} to balance the tradeoff between return and risk. 
Maillard~\cite{Maillard2013} proposed RA-UCB which considers the measure of {\em entropic risk} with a parameter $ \lambda $. Cassel {\em et al.}~\cite{Cassel2018} presented a general and systematic approach to analyzing risk-aware MABs. They adopted the  {\em Empirical Distribution Performance Measure} and proposed the U-UCB algorithm to perform ``proxy regret minimization''.

Another approach casts the risk-aware MAB problem as a {\em constrained MAB problem}, i.e., the allowable risk that the agent can tolerate is formulated as a constraint in the online optimization problem. This is of practical interest in high-risk settings (such as clinical trials) in which the agent demands that the arm (treatment) to be eventually selected has a risk that is strictly below a permissible threshold.
David {\em et al.}~\cite{David2018} focused on identifying an arm with {\em almost} the largest mean among those {\em almost} satisfying an $ \alpha $-quantile constraint under the fixed confidence setting. The authors presented a UCB-based algorithm named RiskAverse-UCB-m-best.
Chang~\cite{Chang2020} considered an average cost constraint where each arm is associated with a cost variable that is independent of the reward   and analyzed the   probability of pulling optimal arms.
This approach is also related to safe bandits \cite{Wu2016,Amani2019}, where the arms are conservatively pulled to meet the safety constraint. However, safe bandits are often considered in a cumulative regret setting and the pulled arms should be safe with high probability (w.h.p.). A brief and current survey of taking risk into account in the study of multi-armed bandits is presented in~\cite{tan2022}. 

The variance-constrained BAI problem consists of two distinct tasks---we seek   {\em optimality} in the mean and {\em feasibility} in the variance. This is  different from the Pareto-front identification with bandit feedback problem~\cite{auer16,turgay18a,Zuluaga}, which seeks optimality in {\em both} objectives, i.e., it seeks a solution/arm that has high mean and low variance simultaneously. While the best feasible arm, if it exists, belongs to the Pareto-Front, we still need to identify the best feasible arm among all arms on the Pareto-Front. These two problems are relevant but are essentially different. The problem is also related to identifying the best arm among the feasible arms. In~\cite{samuels19a}, the arms follow multi-dimensional distributions and the feasible arms are defined to be arms whose mean vectors lie in a polyhedron. It only involves a single mean vector and its projection onto either a subspace (for the objective) and a polyhedron (for the feasibility constraint), while here we have to consider two different statistics---the mean and the variance.

There are works associated with the variance estimation~\cite{Lu,faella2020rapidly} in the BAI problem. However, the variance estimation is done to {\em improve} the algorithms for the {\em standard} BAI objective in both works. Our feasibility constraint in terms of the variance, in conjunction with the standard BAI objective, is a  novel problem setting.

\subsection{Contributions}
We consider the {\em variance-constrained} BAI problem under the fixed confidence setting, i.e., we wish to identify the arm which satisfies a certain variance constraint and has the largest expectation w.h.p. Different from \cite{David2018}, we aim to identify the best arm strictly satisfying the risk constraint {\em without any slack} or {\em suboptimality}. We discuss more differences of our setting and our algorithm vis-\`a-vis \cite{David2018} in Section~\ref{sec:comp}.

We design  VA-LUCB$ (\bar{\sigma}^2,\delta) $   and derive an upper bound on its time or sample complexity. VA-LUCB is an LUCB-based~\cite{Kalyanakrishnan2012} algorithm that is generally better than UCB-based algorithms for BAI problems \cite{Jamieson2014}. It particularizes to LUCB when the constraint is inactive. A hardness parameter $H_{\mathrm{VA}}$ is identified as a fundamental limit; $H_{\mathrm{VA}}$ also reduces to  $ H_1 $ \cite{Audibert2010} when the constraint is inactive. Furthermore, the framework and analysis of VA-LUCB can be extended to other risk measures as long as there are appropriate concentration bounds, e.g.,   Bhat and Prashanth~\cite{Bhat2019} or Chang and Tan~\cite{chang2021unifying} enables us to use CVaR or certain continuous functions as  risk measures within the generic VA-LUCB framework. Different from the work of \cite{David2018} which addresses a similar problem, our algorithm is {\em completely parameter free}, in the sense that Algorithm~\ref{VALUCB} can output the best feasible arm $i^\star$ without knowledge of any parameters that define the instance (e.g., the  suboptimality gaps).

To assess the optimality of VA-LUCB, we prove an accompanying  information-theoretic lower bound on the optimal expected sample complexity of {\em any}  variance-constrained BAI algorithm. We show that VA-LUCB's sample complexity is {\em optimal} up to a logarithmic factor in $H_{\mathrm{VA}}$. 

Lastly, we present extensive experiments in which we examine the effect of each term in $ H_\mathrm{VA} $. We compare VA-LUCB to a na\"ive  algorithm  based on  uniform sampling and a variant of the algorithm in David {\em et al.}~\cite{David2018}  which can only be applied if some unknown parameters  (such as the suboptimality gaps) are known  (see App.~\ref{sec:RiskAverse}). Our experiments suggest that VA-LUCB is the gold standard for this class of constrained BAI problems, reducing the sample complexity significantly, especially for the   riskiest instances. 

\section{Problem Setup}\label{sec:prob_setup}
Given a positive integer $ n $, let $ [n] =\{1,2,\ldots,n\} $. We assume that there are $ N $ arms and  arm $i \in [N]$ corresponds to a reward distribution $ \nu_i$. For each $ i$, the reward of arm $ i $ is denoted by $ X_i$ with $ X_i\sim \nu_i$, which is independent of $ X_j\sim\nu_j$ for all $ j\in[N]\backslash\{i\} $. The expectation and variance of $ X_i $ are denoted by $ \mu_i $ and $ \sigma_i^2 $ respectively. The permissible upper bound on the variance is denoted by $ \bar{\sigma}^2 > 0$. An  {\em instance} $ (\nu=(\nu_1,\ldots,\nu_N),\bar{\sigma}^2) $, consists of $ N $ reward distributions and the upper bound on the variance $ \bar{\sigma}^2 $. Given any instance $ (\nu,\bar{\sigma}^2) $, arm $ i $ is said to be {\em feasible} if  $ \sigma_i^2\leq \bar{\sigma}^2 $. We define $ \mathcal{F}:=\{i\in[N]:\sigma_i^2\leq \bar{\sigma}^2\}$ to be the {\em feasible set} which contains all the feasible arms. Let $ \bar{\mathcal{F}}^c:=[N]\backslash\mathcal{F} $ be the set of all the {\em infeasible} arms. We say an instance is {\em feasible} if $ \mathcal{F} $ is nonempty and we say it is {\em infeasible} otherwise. For a feasible instance, the feasibility flag $ \mathsf{f}=1 $ and the best feasible arm $ i^\star:=\mathop{\rm argmax}\{\mu_i:i\in\mathcal{F}\} $, where $\mathop{\rm argmax}$ returns the smallest index that achieves the maximum. For an infeasible instance, the feasibility flag $ \mathsf{f}$ is set to be $0 $.

An arm $ i $ is said to be {\em suboptimal} if $ \mu_i<\mu_{i^\star} $ and {\em risky} otherwise. We define the {\em suboptimal set} $\mathcal{S} :=\{i\in[N]:\mu_i<\mu_{i^\star} \} $ if $\mathcal{F}\ne \emptyset$ and $\mathcal{S}:=\emptyset$ if $\mathcal{F} = \emptyset$.
The {\em risky set} $ \mathcal{R}:=[N]\setminus\mathcal{S} $ consists of   arms whose expectations are not smaller than $\mu_{i^\star}$.
Define $ i^{\star\star}:=\mathop{\rm argmax}\{\mu_i:i\in\mathcal{S}\}$ to be the arm with greatest expectation among all the suboptimal arms if $ \mathcal{S}\neq\emptyset $. Denote the {\em mean gap} for arms $i\in\mathcal{S}$ as $\Delta_i=\mu_{i^\star}-\mu_i$ if $\mathcal{F}\ne\emptyset$.
Denote the mean gap for arm $i^\star$ as $\Delta_{i^{\star\star}}$ if $\mathcal{F}\ne\emptyset$ and $\mathcal{S}\ne \emptyset$ and $+\infty$ if $\mathcal{S}=\emptyset$.  
Let the {\em variance gaps} for all arms $i\in [N]$ be $\Delta_i^{\mathrm{v}} := |\sigma_i^2-\bar{\sigma}^2|$.
The {\em separator} between $ i^\star $ and the suboptimal arms is denoted by  $\bar{\mu} := (\mu_{i^\star} + \mu_{i^{\star\star}})/2$ if $\mathcal{F}\ne\emptyset$ and $\mathcal{S}\ne\emptyset$ and $\bar{\mu} :=-\infty$ otherwise. 
These sets and quantities are illustrated in Figure~\ref{fig:F}.
\begin{figure}[t]
	\centering
	\includegraphics[width=.45\textwidth]{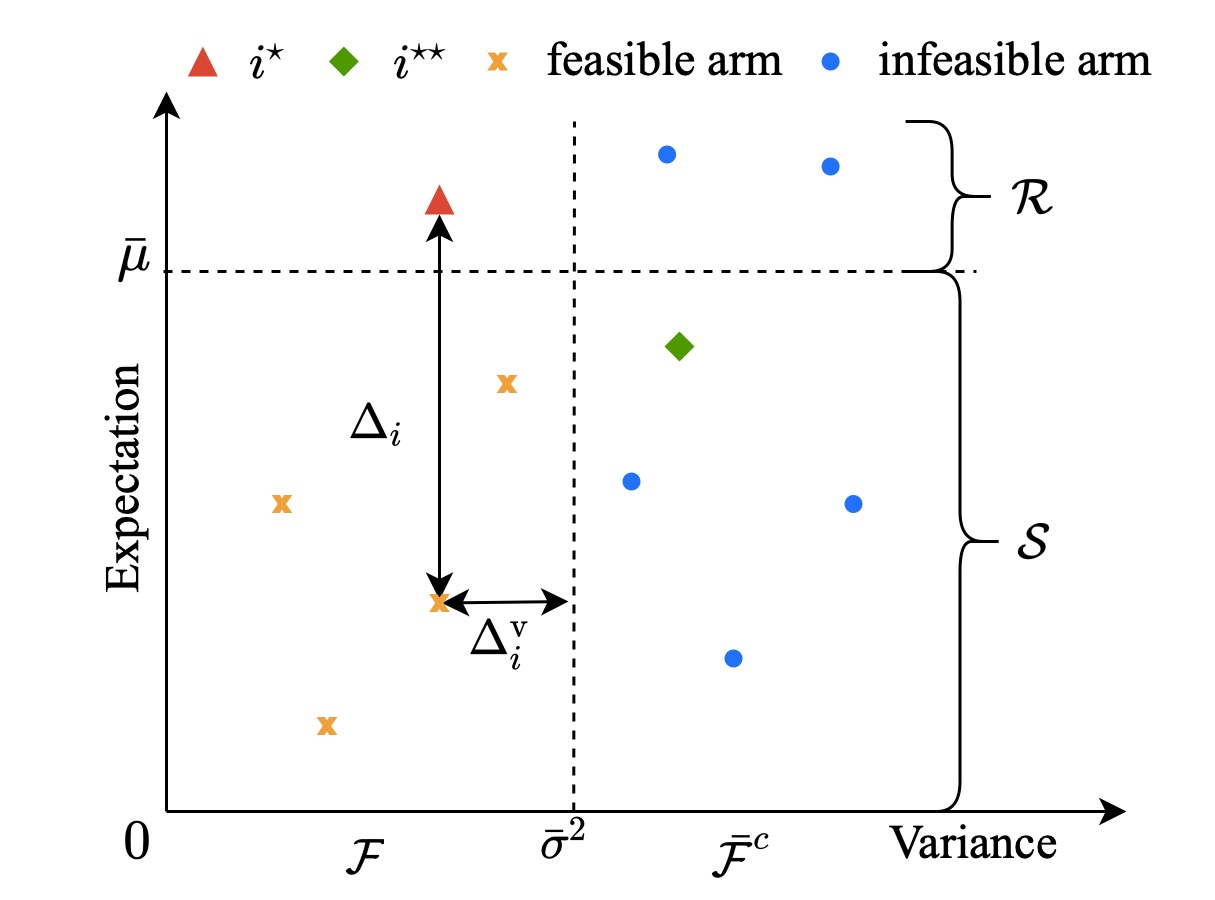}
	\caption{A diagram of the arms. Each dot represents the expectation and variance of an arm.} 
	\label{fig:F}
\end{figure}

At round $ r $, the agent pulls an arm $ i_r \in[N]$ based on the observation history $ ((i_1,X_{1,i_1}),\ldots,(i_{r-1},X_{r-1,i_{r-1}})) $. The agent then observes $ X_{r,i_r}\sim \nu_{i_r} $.  
The rewards sampled from the same arm at different rounds are i.i.d., i.e., $ \{X_{r,i}:r\in\mathbb{N} \}$ are i.i.d.\ samples drawn from $ \nu_i $.

We assume that, if it exists, the best feasible arm is {\em unique} and the variance of the best feasible arm is strictly smaller than  $ \bar{\sigma}^2 $, i.e., $ \sigma_{i^\star}^2<\bar{\sigma}^2 $. We discuss the case  $ \sigma_{i^\star}^2=\bar{\sigma}^2 $ in App.~\ref{sec:equality_case}.  For the sake of clarity, we consider bounded rewards, which are sub-Gaussian. 
Without loss of generality, the reward distributions are supported on $ [0,1] $. We describe extensions  to sub-Gaussian rewards in   App.~\ref{sec:subgaussian}.

Given an instance $ (\nu,\bar{\sigma}^2) $, we would like to design and analyze  an algorithm  that {\em succeeds} w.h.p., i.e., to identify   whether the instance is feasible, and if so, identify the best feasible arm $ i^\star $ in the fewest number of rounds. An algorithm $  \pi:=((\pi_r)_{r\in\mathbb{N}},(\Gamma_r^\pi)_{r\in\mathbb{N}},\phi^\pi) $ determines which arm to pull, when to stop, whether the instance is feasible, and which arm to recommend. More precisely, 
\begin{itemize}
	\item The \textit{sampling strategy} $ \pi_r:([N]\times[0,1])^{r-1} \rightarrow [N]$ decides which arm to sample at round $ r $ based on the observation history, i.e.
	\begin{equation*}
		\pi_r((i_1^\pi,X_{1,i_1^\pi}),\ldots,(i_{r-1}^\pi,X_{r-1,i_{r-1}^\pi}))= i_r^\pi.
	\end{equation*}
	Let $\mathcal{H}_r:=\sigma(i_1^\pi,X_{1,i_1^\pi},\ldots,i_{t}^\pi,X_{r,i_{r}^\pi}) $ be the history of arm pulls and rewards. Then  $\pi_r$ is $ \mathcal{H}_{r-1} $-measurable.
	\item The \textit{stopping rule} $\Gamma_r^\pi$ where $\Gamma_r^\pi$ is $\mathcal{H}_r$-measurable and $\{\text{stop},\text{continue}\}$-valued, decides whether to stop the algorithm at each round $r$.
	The stopping round is denoted by $ \tau^\pi $ if the algorithm stops.
	\item The \textit{recommendation rule} $ \phi^\pi:([N]\times[0,1])^{\tau^\pi} \rightarrow \{0,1\}\times ([N]\cup\{\emptyset\})$ finally gives an estimated flag $ \hat{\mathsf{f}}^\pi\in\{0,1\} $ and an arm $ i_\mathrm{out}^\pi \in [N]$ if $ \hat{\mathsf{f}}^\pi=1 $ based on the observation history (i.e. $ \phi^\pi $ is $ \mathcal{H}_{\tau^\pi} $-measurable):
	\begin{equation*}
		\phi^\pi	 (( i_1^\pi,X_{1,i_1^\pi}),\ldots,(i_{\tau^\pi}^\pi,X_{\tau^\pi,i_{\tau^\pi}^\pi} ))= (\hat{\mathsf{f}}^\pi, i_\mathrm{out}^\pi),
	\end{equation*}
\end{itemize}
The sample complexity of the algorithm $ \pi $ is denoted as $ \tau^\pi $. In the fixed confidence setting, we say that an algorithm $ \pi $ is $\delta$-PAC if the following two conditions hold
\begin{align}
		&\mathbb{P}_\nu[\hat{\mathsf{f}}^\pi=1,i_\mathrm{out}^\pi=i^\star\mid\mathsf{f}=1]\geq1-\delta\;\;\mbox{ and }
		\\
		&\mathbb{P}_\nu[\hat{\mathsf{f}}^\pi=0\mid\mathsf{f}=0]\geq1-\delta.
\end{align} 
The above conditions imply that $\pi$ succeeds with probability at least $1-\delta$. 
Our aim is to design and analyze a $ \delta $-PAC algorithm $ \pi $ that minimizes the sample complexity $ \tau^\pi  $ in expectation and w.h.p. 
We define the optimal expected sample complexity as 
\begin{equation*}
	\tau_\delta^\star= \tau_\delta^\star(\nu,\bar{\sigma}^2) :=\inf\big\{ \mathbb{E}[\tau^\pi]:\pi \mbox{ is } \delta\mbox{-PAC}\big\},
\end{equation*}
where the infimum is taken over all $\delta$-PAC algorithms $\pi$ (as defined above).
For simplicity, we omit the superscripts $ \pi $ in $ \tau^\pi $, $ \phi^\pi $ and $ \hat{\mathsf{f}}^\pi $ if there is no risk of confusion. 

\section{The VA-LUCB   Algorithm}
We present our algorithm which is named {\em Variance-Aware-Lower and Upper Confidence Bound} (or {\em VA-LUCB}) in Algorithm~\ref{VALUCB}. Given an instance $ (\nu,\bar{\sigma}^2) $, the agent pulls each arm according to the VA-LUCB policy to ascertain whether the instance is feasible and to determine which arm is the best feasible arm if the instance is ascertained to be feasible.
\begin{algorithm}
	\caption{Variance-Aware LUCB (VA-LUCB)}
	\begin{algorithmic}[1]
		\STATE \textbf{Input}: threshold $\bar{\sigma}^2 \!>\!0$ and confidence parameter $ \delta\! \in\! (0,1)$. 
		\STATE Sample each of the $ N $ arms twice and set $\bar{\mathcal{F}}_{N}=[N]$.
		\FOR{ time step $ t=N+1, N+2\ldots $}
		\STATE Compute the sample mean using \eqref{sm} and sample variance using \eqref{sv} for $i\in\bar{\mathcal{F}}_{t-1}$.
		\STATE Update the confidence bounds for the mean and variance by \eqref{cb} and \eqref{cbv} for $i\in\bar{\mathcal{F}}_{t-1}$.
		\STATE Update $ \mathcal{F}_t$ and $ \bar{\mathcal{F}}_t$ (see \eqref{f} and \eqref{ni}).
		\STATE Find $ 	i_t^\star:=\mathop{\rm argmax}\{\hat{\mu}_i(t):i\in\mathcal{F}_t\}$ if $ \mathcal{F}_t\neq\emptyset $.
		\STATE Update $ \mathcal{P}_t$ according to \eqref{p}.
		\IF{$\bar{\mathcal{F}}_t\cap\mathcal{P}_t=\emptyset$}
		\STATE \textbf{if} $ \mathcal{F}_t\neq\emptyset $ \textbf{then} Set $ i_\mathrm{out}=i_{t} $ using \eqref{it} and $ \hat{\mathsf{f}}=1 $. \textbf{else} Set $ \hat{\mathsf{f}}=0 $. \textbf{end if}
		\BREAK
		\ENDIF
		\IF{$ |\bar{\mathcal{F} }_t|=1$}
		\STATE Sample arm  $i_{t}$ using \eqref{it} (in one round).
		\ELSE
		\STATE Find $i_{t}$ and competitor arm $c_{t}$ according to \eqref{ct}.
		\STATE  \textbf{if} $ U_{c_{t}}^\mu(t) \geq L_{i_{t}}^\mu(t) $ \textbf{then} Sample arms $ i_{t} $ and $ c_{t} $ (in two rounds).
		\STATE \textbf{else} Sample arm $ i_{t} $ (in one round). \textbf{end if}
		\ENDIF 
		\ENDFOR
	\end{algorithmic}\label{VALUCB}
\end{algorithm}

Each {\em time step} (Lines $ 3 $ to $ 19 $) in our algorithm consists of one or two {\em rounds}, i.e., the agent may pull one or two arms at each time step. 
The algorithm warms up by pulling each of the arms twice (Line $ 2 $). 
At time step $ t $, we first update the sample means, the sample variances and the confidence bounds of the arms that require exploration (Lines $4$ and~$5$); these are the arms in the so-called possibly feasible set $\bar{\mathcal{F}}_{t-1}$, which will be defined formally in~\eqref{ni}. Let $\mathcal{J}_t$ denote the set of arms sampled at time step $t$.
Define $T_i(t):= \sum_{s=1}^{t-1}\mathbbm{1}\{i\in\mathcal{J}_s \}$ to be the number of times arm $ i $ is pulled before time step $t$.
For arm $ i $ that requires exploration, the sample mean and sample variance before time step $t$ 
are 
\begin{align}
	\hat{\mu}_i(t) &:=\frac{1}{T_i(t)}\sum_{s=1}^{t-1} X_{s,i}\mathbbm{1}\{i \in\mathcal{J}_s\},\quad\mbox{ and}\label{sm}\\
	\hat{\sigma}^2_i(t)& :=
	\frac{\sum_{s=1}^{t-1}\left(X_{s,i}-\hat{\mu}_i(t) \right)^2\mathbbm{1}\{i\in\mathcal{J}_s\}} {T_i(t)-1} .
	\label{sv}
\end{align} 
We define the {\em confidence radii} for the mean and variance as 
\begin{equation}\label{radius}
	\alpha(t,T)=\beta(t,T):=\sqrt{\frac{1}{2 T} \ln \left(\frac{2 N t^{4}}{\delta}\right)}. 
\end{equation}
We denote the {\em lower} and {\em upper confidence bounds} (LCB and UCB) for the empirical mean  of arm  $i$   as 
\begin{align}
		&L_i^\mu(t):= \hat{\mu}_i(t)-\alpha(t,T_i(t))\quad\mbox{and}\\
		&U_i^\mu(t):= \hat{\mu}_i(t)+\alpha(t,T_i(t))\label{cb}
\end{align} 
respectively, as well as the LCB and UCB for the  empirical  variance respectively as 
\begin{align}
		&L_i^{\mathrm{v}}(t) := \hat{\sigma}^2_i(t)-\beta(t,T_i(t))\quad\mbox{and}\\
		&U_i^{\mathrm{v}}(t) := \hat{\sigma}^2_i(t)+\beta(t,T_i(t)).\label{cbv}
\end{align}
\subsection{Partition of the Arms} \label{sec:partition}
Based on the empirical variances, at each time step $ t $, we partition the arms into three disjoint subsets based on the confidence bounds on the variance (Line $ 6 $ of Algorithm~\ref{VALUCB}). 
The first set is the {\em empirically feasible set} at time step $ t $,
\begin{equation}\label{f}
	\mathcal{F}_t:=\{i: U_i^{\mathrm{v}}(t)\leq \bar{\sigma}^2\}.
\end{equation}
The second set  is the {\em empirically almost feasible set}, 
\begin{equation}\label{af}
	\partial\mathcal{F}_t:=\{i:L_i^{\mathrm{v}}(t) \leq \bar{\sigma}^2 < U_i^{\mathrm{v}}(t) \}.
\end{equation}
We define the union of the above two sets as the {\em possibly feasible set}, 
\begin{equation}\label{ni}
	\bar{\mathcal{F}}_t:=\mathcal{F}_t\cup\partial\mathcal{F}_t.
\end{equation}
The {\em empirically infeasible set} at time step $ t $ is
\begin{equation}\label{if}
	\bar{\mathcal{F}}_t^c:=\{i:L_i^{\mathrm{v}}(t) > \bar{\sigma}^2\} .
\end{equation}
These sets are illustrated in Figure \ref{fig:Ft}.
\begin{figure}[t]
	\centering
	\includegraphics[width=.45\textwidth]{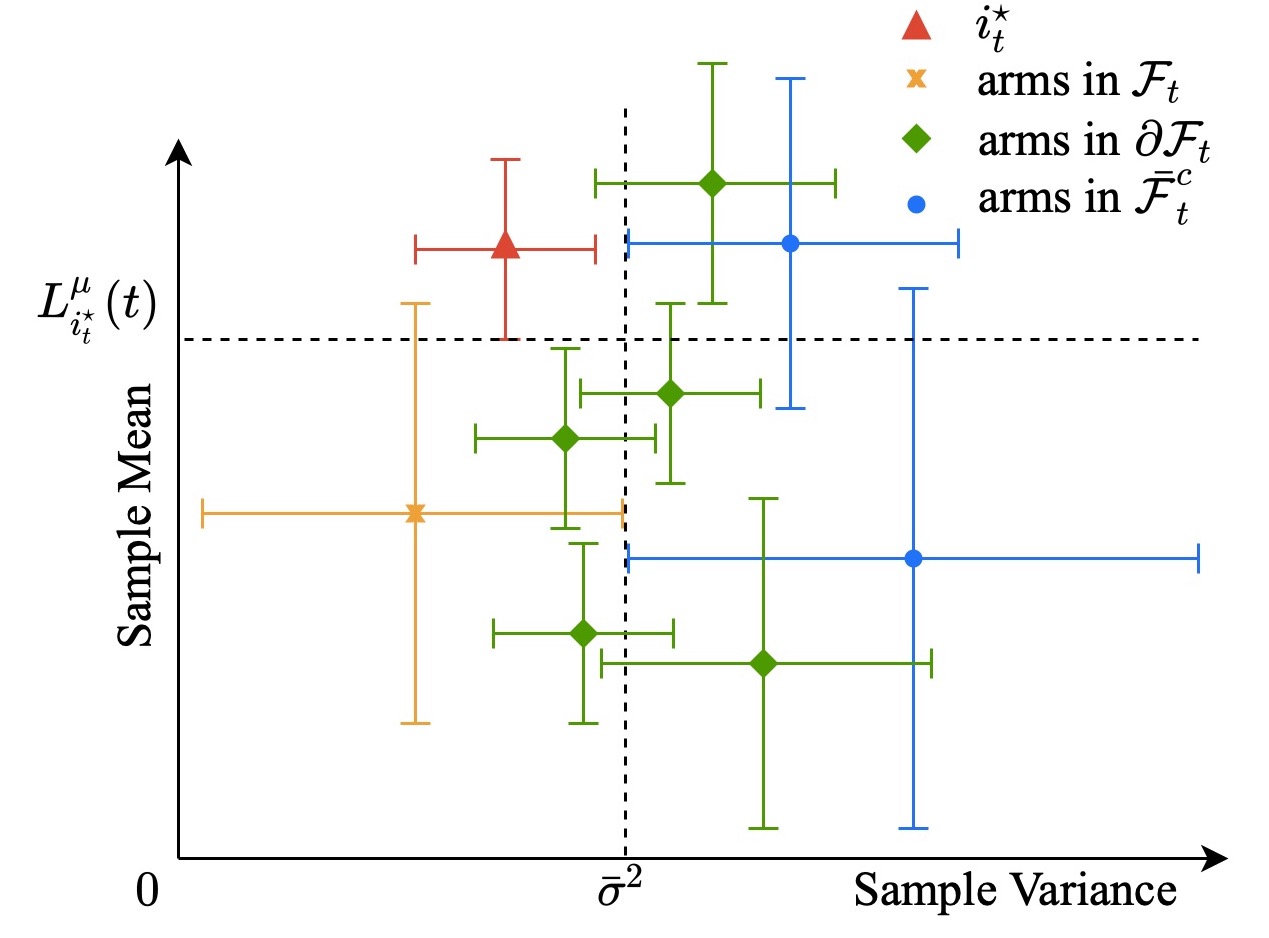}
	\caption{Illustration of the empirical sets. Each dot represents the estimated mean and variance of each arm at time step $t$. The horizontal (resp.\ vertical) component and the horizontal (resp.\ vertical)  crossbar indicates the sample variance  (resp.\ mean) and the confidence interval for the true variance (resp.\ mean). $ i_t $ is the green arm at the top and $ c_t = i_t^\star $. Arms whose  UCBs of the empirical means $U_i^\mu(t)$ are greater than $ L_{i_t^\star}^\mu(t) $ are in~$ \mathcal{P}_t $.}
	\label{fig:Ft}
\end{figure}
The arms that require exploration are the arms in the possibly feasible set $\bar{\mathcal{F}}_t$.
Our intuition is that w.h.p., the true variance of each arm is bounded by the corresponding LCB and UCB, i.e., $ \sigma_i^2\in[L_i^{\mathrm{v}}(t),U_i^{\mathrm{v}}(t)] $; this is stated precisely in Lemma~\ref{Ehappens}. 
If one arm lies in $ \mathcal{F}_t$, it is feasible ($ \sigma_i^2 \leq\bar{\sigma}^2$) w.h.p. Thus only its sample mean needs to be further examined. If one arm lies in $\partial\mathcal{F}_t$, its true feasibility remains unclear, which indicates that this arm needs to be pulled more. If one arm lies in $  \bar{\mathcal{F}}_t^c $, it is infeasible ($ \sigma_i^2 >\bar{\sigma}^2$) w.h.p. Hence, it will not be pulled in future. In summary, only the arms in the possibly feasible set $\bar{\mathcal{F}}_t$ need to be explored more. This justifies the update rules in Lines~$4$ and~$5$ of Algorithm~\ref{VALUCB}.

In terms of the sample mean, if $ \mathcal{F}_t\neq\emptyset $, there is an {\em empirically best feasible arm} at time step $t$ (Line $ 7 $)
\begin{equation}\label{itstar}
	i_t^\star:=\mathop{\rm argmax}\{\hat{\mu}_i(t):i\in\mathcal{F}_t\} .
\end{equation}
Define the {\em potential set} at time step $ t$ (Line $ 8 $) as:  
\begin{equation}\label{p}
	\mathcal{P}_t:=\left\{\begin{aligned}
		&\{i:L_{i_t^\star}^\mu(t)\leq U_i^\mu(t),i\neq i_t^\star\}  ,&\mathcal{F}_t\neq \emptyset\\
		&[N],&\mathcal{F}_t=\emptyset
	\end{aligned}\right. .
\end{equation} 
The potential set contains those arms which potentially have greater expectations than $ i_t^\star $, regardless of their feasibility. 

Considering both the sample variance and sample mean, arms in $ \bar{\mathcal{F}}_t\cap\mathcal{P}_t  $ are said to be  {\em competitor arms}, in the sense that they are possibly feasible and potentially have greater means than the empirically best feasible arm $ i_t^\star $. In conclusion, only the competitor arms in the set $ \bar{\mathcal{F}}_t\cap\mathcal{P}_t  $ need to be pulled more, which also motivates our stopping rule.

\subsection{Stopping Rule}
The intuition for the stopping rule in Lines $ 9$ to $12 $ of VA-LUCB is straightforward. 
If the given instance is infeasible, after pulling all arms sufficiently many times, we have $\mathcal{F}_t= \bar{\mathcal{F}}_t=\emptyset$ and $\mathcal{P}_t=[N] $ (i.e., all the arms are deemed to be   infeasible) w.h.p.  Thus we set the flag $ \hat{\mathsf{f} }=0$ for this instance. 
If the given instance is feasible, after sufficiently many arm pulls,
the arms in $\mathcal{R} \backslash\mathcal{F}$, the best feasible arm $i^\star$ and the arms in $\mathcal{S}$ will be ascertained to be infeasible, 
feasible, 
and  suboptimal respectively. 
At the stopping time step $ \tau $, $ \mathcal{F}_\tau\neq \emptyset$, and we   expect that $i_{\tau}=i_\tau^\star=i^\star $, where 
\begin{equation}\label{it}
	i_{t}:=\mathop{\rm argmax}\big\{\hat{\mu}_i(t):i\in \bar{\mathcal{F}}_t\big\}.
\end{equation}
We formalize this intuition in Lemma \ref{ioutisistar} in Section~\ref{proofsketch}.
When $\bar{\mathcal{F}}_t \cap\mathcal{P}_t    =\emptyset $, there are no competitor arms and we are confident in asserting that the instance is feasible, i.e., $\hat{\mathsf{f}}=1 $ and $ i_t $ is the best feasible arm.

\subsection{Sampling Strategy}
When the algorithm has not terminated, $ \mathcal{P}_t$ and $\bar{\mathcal{F}}_t  $ are not empty.  The intuition for the sampling strategy in Lines~$13$ to~$20 $ of VA-LUCB can be justified as follows.
If the given instance is feasible,
firstly, when $ i_{t} $ is a truly infeasible arm, its infeasibility needs to be ascertained, and secondly,  when $ i_{t}$ is a truly feasible arm, we need to check both of its feasibility and optimality. Thus, in either case, arm $i_t$ requires more pulls.
When $|\bar{\mathcal{F}}_t|>1  $, define the {\em best competitor arm} to $i_t$ as\footnote{Note the {\em competitor arms} are defined for $i_t^\star$ and the {\em best competitor arm} is defined for $i_t$. However, when the given instance is feasible and the arms in $ \bar{\mathcal{F}}^c\cap\mathcal{R} $ are identified as infeasible w.h.p., $ i_t $ will likely be $ i_t^\star $ and $ c_t $ will likely be an suboptimal arm in $\mathcal{S}$. Thus $c_t\in\mathcal{P}_t$ and is a competitor arm to $i_t^\star$ w.h.p.}
\begin{equation}\label{ct}
	c_{t}:=\mathop{\rm argmax}\{U_i^\mu(t):i\in \bar{\mathcal{F}}_t,i\neq i_{t}\}.
\end{equation} 
The fact that $c_t\in\mathcal{P}_t\cup\{i_t^\star\}$ can be justified by Lemma~\ref{ctinpt} in App.~\ref{upperbound}.
Thus, more pulls of $ c_{t} $ are   needed to ascertain which of $ i_{t} $ and $ c_{t} $ has a larger true mean.
If the given instance is infeasible, all the arms in $ \bar{\mathcal{F}}_t $, including $ i_t $ and $c_t $, need to be sampled more times to assert they are indeed infeasible.

We remark that in VA-LUCB, we are {\em interleaving} the verification of optimality (in the mean aspect) and feasibility (in the variance aspect). This is in stark contrast to a  na\"ive but suboptimal strategy in which one uses a two-phase strategy to first identify the feasible arms, then search among these arms for the one with the largest mean. 

\section{Bounds on the Time Complexity}
We state an  upper bound on the sample complexity of our VA-LUCB algorithm and a lower bound on the optimal expected sample complexity over all algorithms.
\subsection{Time Complexity of VA-LUCB}
Given an instance $ (\nu,\bar{\sigma}^2) $, define the {\em variance-aware hardness parameter}
\begin{align}
	H_{\mathrm{VA}} &:= \frac{1}{\min\{\frac{\Delta_{i^\star}}{2},\Delta_{i^\star}^{\mathrm{v}}\}^2 }
	+\sum_{i\in\mathcal{F}\cap\mathcal{S}}\frac{1}{(\frac{\Delta_i}{2})^2}\\*
	&\quad+\sum_{i\in\bar{\mathcal{F}}^c\cap\mathcal{R}}\frac{1}{(\Delta_{i}^{\mathrm{v}})^2 }
	\!+\!\sum_{i\in\bar{\mathcal{F}}^c\cap\mathcal{S}}\frac{1}{\max\{\frac{\Delta_i}{2},\Delta_i^{\mathrm{v}}\}^2}. \label{eqn:hardness}
\end{align}
Our main result is stated as follows.
\begin{restatable}[Upper bound]{thm}{thmupperbound}\label{thm:up_bd}
	Given an instance $ (\nu,\bar{\sigma}^2) $ and confidence parameter $ \delta $, 
	with probability at least $1-\delta$, 
	VA-LUCB succeeds and terminates in
	\begin{equation}
		O\left( H_{\mathrm{VA}}\ln\frac{H_{\mathrm{VA}}}{\delta}\right)\quad \mbox{time steps}. \label{eqn:time_steps}
	\end{equation} 
\end{restatable}
The implied constant in the O-notation can be taken to be no more than $304$. The mean gap $\Delta_i$ and variance gap $\Delta_i^{\mathrm{v}}$ of arm $ i $ are indicative of the hardness of ascertaining its optimality and feasibility respectively. It is easy to see that when the threshold $\bar{\sigma}^2\to \infty$ (in fact, $\bar{\sigma}^2\geq3/4$ suffices), $H_\mathrm{VA}$ reduces to the hardness parameter $H_1:=\sum_{i\neq i^\star}\Delta_i^{-2}$ in the conventional (unconstrained) BAI problem \cite{Audibert2010}.

The intuitions for the four terms  in $ H_\mathrm{VA} $ are as follows:
Firstly, to identify the best feasible arm $i^\star$, both of its feasibility and optimality need to be ascertained, which leads to the first term. 
Secondly, for the arms in $ \mathcal{F}\cap\mathcal{S} $, we can identify them once we have established that they are indeed suboptimal, explaining the dependence  on $\Delta_i^{-2},i \in \mathcal{F}\cap\mathcal{S}$.
Thirdly, since the arms in $ \bar{\mathcal{F}}^c\cap\mathcal{R} $ have larger means than the best feasible arm, the algorithm needs to sample them sufficiently many times to learn they are infeasible, which contributes to the third term in $ H_\mathrm{VA} $. 
Finally, when either the suboptimality or the infeasibility of the arms in $ \bar{\mathcal{F}}^c\cap\mathcal{S} $ is ascertained, we can eliminate them, which explains the last term in $H_\mathrm{VA} $.
The proof of Theorem~\ref{thm:up_bd} is presented in  App.~\ref{upperbound}. 

\begin{remark} \label{rmk:LIL}
    We highlight that Algorithm~\ref{VALUCB} constitutes a convenient framework to tackle any risk-aware BAI problem in the sense that it  is compatible with other concentration bounds.
     For example, one can define alternative confidence radii, different from those specified in~\eqref{radius}, based on the (non-asymptotic)
    Law of  the Iterated Logarithms (LIL)  \cite{jamieson14lil,howard2021time,simchowitz17the,tanczos2017a}. 
     We adopt a simple non-asymptotic LIL concentration bound from Jamieson {\em et al.}~\cite{jamieson14lil} to show that different confidence bounds utilized in  VA-LUCB (Algorithm~\ref{VALUCB}) can lead to slightly different upper bounds on the stopping time with high-probability.
    First, we replace the unbiased sample variance $\hat{\sigma}^2(t)$ by  a biased counterpart
    \begin{align}
        &\tilde{\sigma}^2_i(t) 
		:=
	    \frac{\sum_{s=1}^{t-1}\left(X_{s,i}-\hat{\mu}_i(t) \right)^2\mathbbm{1}\{i\in\mathcal{J}_s\}} {T_i(t)} \\
		&= \frac{1}{T_i(t)}\sum_{s=1}^{t-1} X_{s,i}^2\mathbbm{1}\{i\in\mathcal{J}_s\} -\bigg(\frac{1}{T_i(t)}\sum_{s=1}^{t-1} X_{s,i}\mathbbm{1}\{i\in\mathcal{J}_s\} \bigg)^2.
	\end{align}
    Next, we redefine the the confidence radii $\alpha(t,T)$ and $\beta(t,T)$ (originally defined in~\eqref{radius}) by
    \begin{align}
	     &\tilde{\alpha}(t) :=\left(1+\sqrt{\epsilon}\right)\sqrt{\frac{1+\epsilon}{2t}\ln\left(\frac{4N\ln\left((1+\epsilon)t\right)}{\delta}\right)},\quad\mbox{and}\\
	    &\tilde{\beta}(t):=3\tilde{\alpha}(t),\label{newbounds}
    \end{align}
    respectively, where   $\epsilon\in(0,1)$ is a fixed constant. 
    These choices of the confidence radii allow us to avoid using a union bound to bound the probability of the complement of the  ``good'' event $E$ in \eqref{E}. With the above modifications, we show in  App.~\ref{LIL} that VA-LUCB is $\delta$-PAC (for $\epsilon=0.9$ and $\delta<0.1$) and succeeds in 
    \begin{equation}
    O\left(H_{\mathrm{VA}}^{(1)}\ln\frac{N}{\delta}+H_{\mathrm{VA}}^{(3)}\right) \quad \mbox{time steps,}      \label{eqn:LIL_time_steps}
    \end{equation} 
    where
	\begin{align}
	    H_{\mathrm{VA}}^{(1)} &:=\frac{1}{\min\{\Delta_{i^\star},\frac{2}{3}\Delta_{i^\star}^\mathrm{v}\}^2}
	    +\sum_{i\in\mathcal{F}\cap\mathcal{S}}\frac{1}{\Delta_i^2}\\
	    &\;+\sum_{i\in\bar{\mathcal{F}}^c\cap\mathcal{R}}\frac{1}{(\frac{2}{3}\Delta_i^\mathrm{v})^2 }
	    +\sum_{i\in\bar{\mathcal{F}}^c\cap\mathcal{S}}\frac{1}{\max\{\Delta_i,\frac{2}{3}\Delta_i^{\mathrm{v}}\}^2}, \label{lilhardness1} \\
	    H_{\mathrm{VA}}^{(3)}&:=\Hard \Big(\frac{1}{\min\{\Delta_{i^\star},\frac{2}{3}\Delta_{i^\star}^\mathrm{v}\}^2} \Big)
	    +\sum_{i\in\mathcal{F}\cap\mathcal{S}}\Hard \Big(\frac{1}{\Delta_i^2}\Big)\\
	    &+\sum_{i\in\bar{\mathcal{F}}^c\cap\mathcal{R}}\Hard \Big(\frac{1}{(\frac{2}{3}\Delta_i^\mathrm{v})^2 }\Big)
	    +\sum_{i\in\bar{\mathcal{F}}^c\cap\mathcal{S}}\Hard\Big(\frac{1}{\max\{\Delta_i,\frac{2}{3}\Delta_i^{\mathrm{v}}\}^2}\Big)\label{lilhardness2}
	\end{align}
	and $\Hard: x\in\mathbb{R}_+ \mapsto x  \ln\ln_+(x)$ with $\ln\ln_+: x\in\mathbb{R}_+\mapsto\ln\ln(\max\{e,x\})$.
    Note that $H_{\mathrm{VA}}^{(1)}$ and $H_\mathrm{VA}$ are order-wise equal and $H_{\mathrm{VA}}^{(3)}$  is also of the same order as $H_\mathrm{VA}$ up to double logarithmic terms in the gaps $\{( \Delta_i,\Delta_i^\mathrm{v} )\}_{i\in [N]}$. 
\end{remark}

\begin{remark}
    While we only study  a generalization of the LUCB-based method \cite{Kalyanakrishnan2012} for this variance-constrained BAI problem, other confidence bound-based strategies, e.g., Successive Elimination \cite{Even2006action} and lil'UCB \cite{jamieson14lil}, also have the potential to be generalized to solve this problem. We provide some intuitions in the following.
    \begin{itemize}
        \item {\bf Successive Elimination}: Denote the set of active arms as $\mathcal{A}_t$ and initialize $\mathcal{A}_0=[N]$. At each time step $t$, the algorithm first pulls all active arms once and updates the sample means and sample variances. It then updates the confidence bounds $\alpha(t)$ and $\beta(t)$ for the means and variances (which are similar to the $\alpha_t$ in \cite[Alg.~3]{Even2006action}). Next, it  identifies the empirically best feasible arm $i_t^\star=\argmax\{\hat{\mu}_i(t):i\in\mathcal{A}_t,U_i^\mathrm{v}(t)\leq \bar{\sigma}^2\}$. Finally, it updates the active arm set to be $\mathcal{A}_{t+1}=\{i\in\mathcal{A}_t:\hat{\mu}_{i_t^\star}(t)-\hat{\mu}_i(t)<2\alpha(t),\hat{\sigma}_i^2(t)-\bar{\sigma}^2<\beta(t)\}$. The algorithm terminates when the active set is empty (which indicates  that the instance is infeasible) or the active set contains  only the empirically best feasible arm (which is then declared  to be the best feasible arm).
        \item {\bf lil'UCB}: The sampling strategy for this algorithm when there is a constraint on variance (or risk) of the arms is obvious. In particular,   the algorithm samples arm $i_t=\argmax\{U_i^\mu(t):L_i^\mathrm{v}(t)\leq \bar{\sigma}^2\}$ where $U_i^\mu(t)$ and  $L_i^\mathrm{v}(t)$ are constructed in view of the LIL. However, the stopping criterion is not straightforward, since the LIL-based stopping rule \cite{jamieson14lil} cannot be directly utilized. This is an interesting direction for future research.
    \end{itemize}
\end{remark}

\subsection{Lower Bound}
A natural question is whether the upper bound stated in Theorem \ref{thm:up_bd} (or the number of time steps of the LIL version of VA-LUCB in~\eqref{eqn:LIL_time_steps}) is tight and whether the quantity $H_{\mathrm{VA}}$ is {\em fundamental}. This is addressed in this section via an information-theoretic lower bound which indicates the expected sample complexity  of VA-LUCB is optimal up to $\ln  H_\mathrm{VA} $. 

Since the rewards  are bounded in $[0,1]$, the variance of each arm is at most ${1}/{4}$.
Therefore, when $\bar{\sigma}^2 \in[ 1/4,\infty)$, all arms are feasible and there exists a generic lower bound \cite{Garivier2016}.
When $\bar{\sigma}^2 \in (0,1/4)$, let 
\begin{align*}
	\bar{a} := \frac{1 + \sqrt{1 - \bar{\sigma}^2} }{2 }
	~\text{ and }~
	\underline{a} := \frac{1 - \sqrt{1 - \bar{\sigma}^2} }{2 } . 
\end{align*}
These quantities are the solutions to the quadratic equation $a(1-a)=\bar{\sigma}^2$.
\begin{restatable}[Lower bound]{thm}{thmLowBd}
	\label{thm:low_bd}
	Given any instance $ (\nu,\bar{\sigma}^2) $ with $\bar{\sigma}^2\in (0,1/4)$, define the constant $c(\nu,\bar{\sigma}^2) :=\min \big\{ \underline{a} (1/4 - \bar{\sigma}^2 ) ,\ \underline{a}/ 8 ,\  ( 1-\mu_{ i^{\star} })/  8 
	\big\}$,  
	\begin{align}
		&\tau_\delta^\star
		\ge  
		c(\nu,\bar{\sigma}^2)\, H_\mathrm{VA}\,
		\ln \bigg( \frac{1}{2.4\delta} \bigg)   \label{eqn:lb}
		.
	\end{align}   
\end{restatable}
The proof is in  App.~\ref{lowerbound}. Based on Theorems~\ref{thm:up_bd} and~\ref{thm:low_bd}, we have the following corollary whose proof is also provided in App.~\ref{lowerbound}.	This almost conclusive result says that we have  characterized $\tau_\delta^\star$ up to a  (small)  factor logarithmic in~$ H_{\mathrm{VA}}$. 

\begin{restatable}[Almost optimality of VA-LUCB]{cor}{tightness} \label{cor:almost_opt}
	Given any instance $ (\nu,\bar{\sigma}^2) $ and confidence parameter $\bar{\sigma}^2\in(0,1/4)$, the optimal expected sample complexity is 
	\begin{equation}\label{thm:tightness}
		\tau_\delta^\star=O\Big(H_\mathrm{VA}\ln\frac{H_\mathrm{VA}}{\delta}\Big)\bigcap \Omega\Big(H_\mathrm{VA}\ln\frac{1}{\delta}\Big).
	\end{equation}
	The bounds  can also be expressed as 
	\begin{equation}
	   \lim_{\delta\downarrow 0} \frac{\tau_\delta^\star}{\log \frac{1}{\delta}  \vphantom{\big)} }=\Theta( H_{\mathrm{VA }} ),
	\end{equation}
	and  VA-LUCB achieves the upper bounds.
\end{restatable}
Corollary~\ref{cor:almost_opt} says that $H_{\mathrm{VA}}$ is the fundamental limit for the problem of variance-constrained BAI. 

\subsection{Comparison to David {\em et al.}~\cite{David2018}} \label{sec:comp}
We adopt the variance as the risk measure and focus on the {\em (strict) best feasible} arm identification problem under the $\delta$-PAC framework, while David {\em et al.}~\cite{David2018} uses the $\alpha$-quantile as the risk metric and consider   {\em $\epsilon_\rho$-approximately feasible and $\epsilon_\mu$-approximately optimal} arms.
We consider a variant of their algorithm, named {\em RiskAverse-UCB-BAI} (See App.~\ref{sec:RiskAverse}) that is tailored to our variance-constrained problem in which the best feasible arm must be produced w.h.p. 

\begin{itemize}
    \item \textbf{Parameters}:
The most important  difference is that  VA-LUCB is {\em parameter free}. In contrast, RiskAverse-UCB-BAI heavily relies on knowledge of the hardness parameter $H$  (which appears in the confidence radii), and the accuracy parameters $\epsilon_\mu$ and $\epsilon_\mathrm{v}$ (of the mean and variance respectively), which determine when it terminates. To output the best feasible arm w.h.p., one needs to set the accuracy parameters to be some functions of the unknown mean gaps and variance gaps such that the only {\em $\epsilon_\mathrm{v}$-approximately feasible and $\epsilon_\mu$-approximately optimal} arm is exactly the (strict) best feasible arm. Thus,  if we want to output the best feasible arm, RiskAverse-UCB-BAI  is {\em not} parameter free.
\item \textbf{Upper Bounds}: 
The hardness parameters $\sum_{i\in[N]}C_i$ and $H$, defined in \eqref{equ:David_hardness1} and \eqref{equ:David_hardness2}  respectively, are used to characterize the upper bound (on the sample complexity of RiskAverse-UCB-BAI) in \cite[Theorem~3]{David2018} and   are  lower bounded by $H_\mathrm{VA}$  (see App.~\ref{sec:discussion_upper}). Intuitively, since $H$  is only a function of the accuracy parameters $(\epsilon_\mathrm{v}, \epsilon_\mu)$, but $H_\mathrm{VA}$ takes the means and variances of all arms into account,  the latter is smaller (hence better). We formalize this intuition in App.~\ref{sec:discussion_upper}. Even disregarding these constants, the additional $\ln$ term in $N$ and $\ln\ln$ term in $N/\delta$ in the upper bound of RiskAverse-UCB-BAI (see Eqn.~\eqref{equ:David_upperbd1}) indicates that its sample complexity  is strictly larger than that of VA-LUCB (see App.~\ref{sec:discussion_upper} for details).
\item \textbf{Lower Bounds}: By comparing   terms involving arm $i$ in both lower bounds, we deduce that our lower bound is strictly larger than that in   \cite[Theorem~2]{David2018}  for  most  ($\ge 99.9\%$ of)   $(\mu_{i^\star},\bar{\sigma}^2)$ pairs (see App.~\ref{sec:discussion_lower}). 
Corollary~\ref{cor:almost_opt} states that $H_\mathrm{VA}$ is fundamental in characterizing the   hardness of the instance. This also implies the lower bound of \cite{David2018} is, in general, not tight   in our variance-constrained BAI setting. 
Due to the choice of confidence radius in~\eqref{radius}, we also claim that VA-LUCB identifies  risky arms faster than RiskAverse-UCB-BAI (see App.~\ref{sec:discussion_riskyarms}). 
\end{itemize}

\section{Experiments}\label{experiments}
We design experiments to illustrate the empirical performance of VA-LUCB. We compare VA-LUCB to RiskAverse-UCB-BAI~\cite{David2018} and a na\"ive baseline algorithm  VA-Uniform (described in Section~\ref{sec:uniform}).
The code to reproduce all the figures is available at \href{https://github.com/Y-Hou/VA-BAI.git}{https://github.com/Y-Hou/VA-BAI.git}.

\subsection{Experimental Design}
By Theorem \ref{thm:up_bd}, the sample complexity of VA-LUCB is upper bounded by $ O\left(H_{\mathrm{VA}}\ln(H_\mathrm{VA}/\delta)\right) $ w.h.p. 
We design four sets of test cases 
 to empirically demonstrate  the impact of the mean gaps $\Delta_i$ and the variance gap $\Delta_i^\mathrm{v}$ in $H_\mathrm{VA}$ on the sample complexity, in particular the smaller one of $\Delta_{i^\star}/2,\Delta_{i^\star}^{\mathrm{v}}$ will dominate the best feasible arm term and the greater one of $\Delta_{i}/2,\Delta_{i}^{\mathrm{v}}$ will dominate the suboptimal and infeasible arm term. 
The parameters that are varied in each test case are described below.

\noindent
1.\ For the first term  $ \min\{\Delta_{i^\star}/2,\Delta_{i^\star}^{\mathrm{v}}\}^{-2}  $,\newline
(a). Under the condition that $ \Delta_{i^\star}/2\leq \Delta_{i^\star}^{\mathrm{v}} $, when $ \Delta_{i^\star} $ and $  \Delta_{i^{\star\star}}  $ increase with the rest of the arms kept the same, $ H_{\mathrm{VA}} $ and the sample complexity will decrease.\newline
(b). Under the condition that $ \Delta_{i^\star}/2\leq \Delta_{i^\star}^{\mathrm{v}} $, when $ \Delta_{i^\star}^{\mathrm{v}} $ increases, $ H_{\mathrm{VA}} $ and the sample complexity will be kept the same.\newline
(c). Under the condition that $ \Delta_{i^\star}/2\geq \Delta_{i^\star}^{\mathrm{v}} $, as $ \Delta_{i^\star}^{\mathrm{v}} $ increases, $ H_{\mathrm{VA}} $  and the sample complexity will decrease.\newline
(d). Under the condition that $ \Delta_{i^\star}/2\geq \Delta_{i^\star}^{\mathrm{v}} $, as $ \Delta_{i^\star} $ and $ \Delta_{i^{\star\star}}  $ increase, $ H_{\mathrm{VA}} $ and the sample complexity will decrease.\newline
2.\ For the second term $ \sum_{i\in\mathcal{F}\cap\mathcal{S}} 4  \Delta_i^{-2 }$, when $  \Delta_{i^\star} $ and $  \Delta_{i}$ for all  $i\in\mathcal{F}\cap\mathcal{S}   $ increase, $ H_{\mathrm{VA}} $ and the sample complexity will decrease.\newline
3.\ For the third term $ \sum_{i\in\bar{\mathcal{F}}^c\cap\mathcal{R}}( \Delta_{i}^{\mathrm{v}})^{-2}  $, when $  \Delta_{i}^{\mathrm{v}}$ for all  $i\in\bar{\mathcal{F}}^c\cap\mathcal{R}  $ increase, $ H_{\mathrm{VA}} $ and the sample complexity will decrease.\newline
4.\ For the fourth term $ \sum_{i\in\bar{\mathcal{F}}^c\cap\mathcal{S}} \max\{\Delta_i/2,\Delta_i^{\mathrm{v}}\}^{-2} $, the design is quite similar to Case 1, and thus the details are omitted here and presented in App.~\ref{design of fourth}.

The confidence parameter $\delta$ is set to be $0.05$. In each case, there are $ 11 $ instances with $N= 20 $ arms. The specific instances are described in detail in App.~\ref{specific para}. For each algorithm and instance, we run $20$ independent trials to estimate the average time complexities and their standard deviations.

Note that there are $4$ cases for the first term as we wish to elucidate that the smaller quantity between $ \Delta_{i^\star}/2$ and $\Delta_{i^\star}^{\mathrm{v}} $  dominates the sample complexity of $ i^\star $. The same experimental design applies to the study of the fourth term.
\begin{figure*}[htbp]
	\centering
	\subfigure[Case 1(a)]{
		\includegraphics[width=.1845\textwidth]{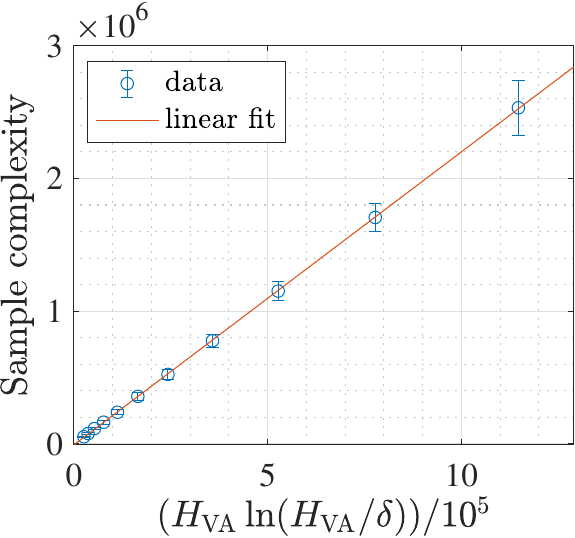}
	} \hspace{-.06in}
	\subfigure[Case 1(c)]{
		\includegraphics[width=.182\textwidth]{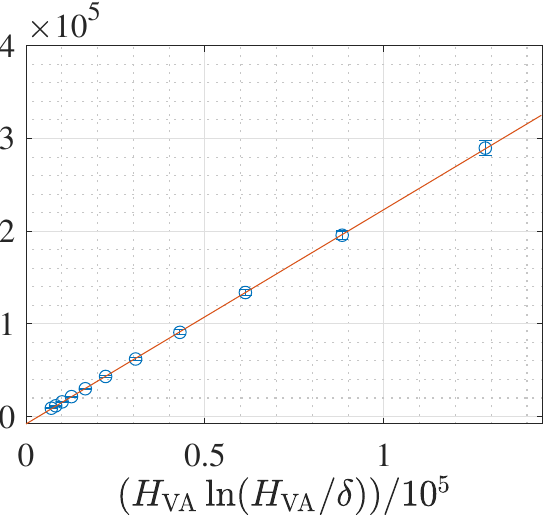}
	}  \hspace{-.06in}
	\subfigure[Case 1(d)]{
		\includegraphics[width=.183\textwidth]{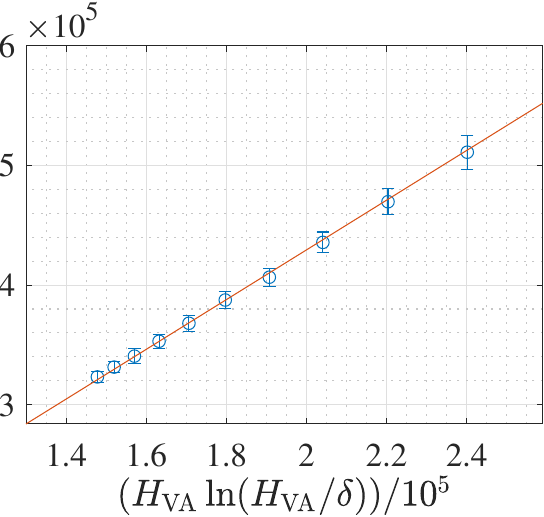}
	}  \hspace{-.06in}
	\subfigure[Case 2]{
		\includegraphics[width=.182\textwidth]{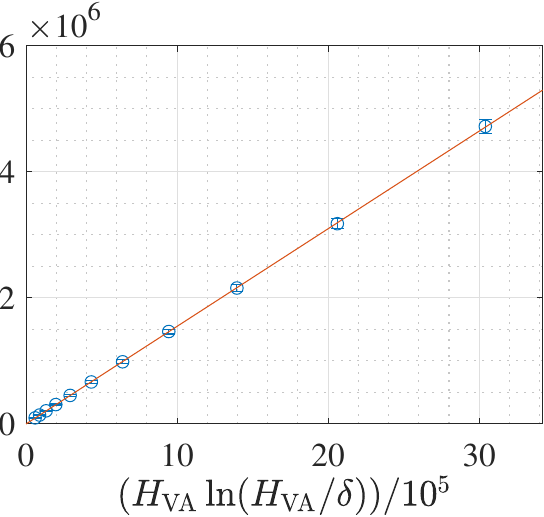}
	} \hspace{-.06in}
	\subfigure[Case 3]{
		\includegraphics[width=.182\textwidth]{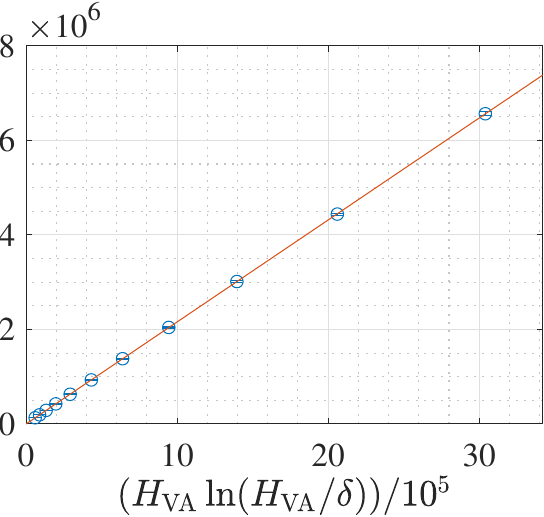}
	}
	\caption{The time complexities for various cases with respect to $ H_{\mathrm{VA}}\ln (H_{\mathrm{VA}}/\delta )$ with $\delta= 0.05$.}
	\label{example}
\end{figure*}

\subsection{Performance of VA-LUCB} \label{sec:perf_VA}
We plot the time complexities of Cases 1--3 with respect to     $H_{\mathrm{VA}}\ln(H_{\mathrm{VA}}/\delta) $ in Figure \ref{example};  the rest of the figures are relegated to App.~\ref{experiment figure}. 
Although we do not prove the sample complexity grows linearly with $ H_{\mathrm{VA}}\ln(H_{\mathrm{VA}}/\delta)  $, this phenomenon can indeed be observed in our experiments.
All the experimental results indicate the true sample complexity of VA-LUCB appears to be linear in  $ H_{\mathrm{VA}}\ln(H_{\mathrm{VA}}/\delta)  $ (showing the tightness of our analyses) and is also  bounded  by $ H_{\mathrm{VA}}\ln(H_{\mathrm{VA}}/\delta)  $  and $3\, H_{\mathrm{VA}}\ln(H_{\mathrm{VA}}/\delta)  $. The upper bound of $ 3\, H_{\mathrm{VA}}\ln(H_{\mathrm{VA}}/\delta) $  is usually sufficient for VA-LUCB to succeed. 

\begin{figure*}[t]
	\centering
	\begin{minipage}{0.45\linewidth}
		\centering
		\includegraphics[width=\linewidth]{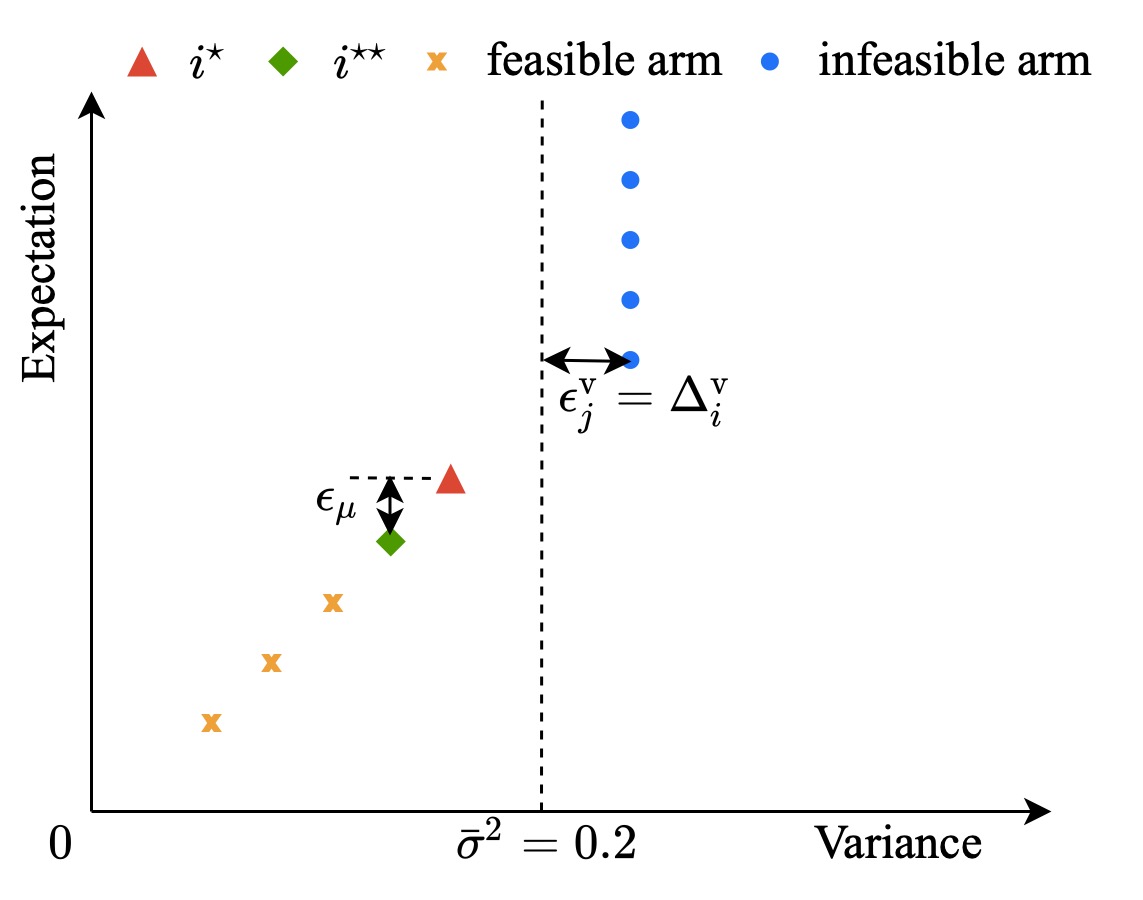}
		\caption{An illustration of the instances. }
		\label{Competition}
	\end{minipage}
	\quad
	\begin{minipage}{0.45\linewidth}
		\centering
		\begin{tabular}{|ccc|}
			\hline
			\multicolumn{3}{|c|}{$\bar{\sigma}^2=0.2,N=10$}                                                        \\ \hline
			\multicolumn{1}{|c|}{arm} & \multicolumn{1}{c|}{$\mu_i$} & $\sigma_i^2$ \\ \hline
			\multicolumn{1}{|c|}{1}   & \multicolumn{1}{c|}{$0.1$}   & $0.08$                     \\ \hline
			\multicolumn{1}{|c|}{2}   & \multicolumn{1}{c|}{$0.15$}  & $0.1$                      \\ \hline
			\multicolumn{1}{|c|}{3}   & \multicolumn{1}{c|}{$0.2$}   & $0.12$                     \\ \hline
			\multicolumn{1}{|c|}{4}   & \multicolumn{1}{c|}{$0.25$}  & $0.14$                     \\ \hline
			\multicolumn{1}{|c|}{5}   & \multicolumn{1}{c|}{$0.3$}   & $0.16$                     \\ \hline
			\multicolumn{1}{|c|}{6}   & \multicolumn{1}{c|}{$0.4$}   & $\epsilon_j^\mathrm{v}$                     \\ \hline
			\multicolumn{1}{|c|}{7}   & \multicolumn{1}{c|}{$0.45$}  & $\epsilon_j^\mathrm{v}$                     \\ \hline
			\multicolumn{1}{|c|}{8}   & \multicolumn{1}{c|}{$0.5$}   & $\epsilon_j^\mathrm{v}$                      \\ \hline
			\multicolumn{1}{|c|}{9}   & \multicolumn{1}{c|}{$0.55$}  & $\epsilon_j^\mathrm{v}$                      \\ \hline
			\multicolumn{1}{|c|}{10}  & \multicolumn{1}{c|}{$0.6$}   & $\epsilon_j^\mathrm{v}$                   \\ \hline
		\end{tabular}
		\captionof{table}{Parameter settings for instance $j\in[10]$. The variance gaps for the infeasible arms $\Delta_i^\mathrm{v}=\epsilon_j^\mathrm{v}=0.233-0.003\cdot j$ in instance $j\in[10]$.}
		\label{competition_para}
	\end{minipage}
\end{figure*}

\subsection{Comparison of VA-LUCB to  RiskAverse-UCB-BAI \cite{David2018} and VA-Uniform} \label{sec:uniform}
We  compare VA-LUCB to its closest competitor RiskAverse-UCB-BAI and VA-Uniform, which differs from VA-LUCB only in the sampling strategy.  VA-Uniform uniformly samples two out of $N=10$ arms  at each time step. For  comparison among the three algorithms, we  construct $10$ high-risk, high-reward instances with $N=10$ arms in each instance to demonstrate that VA-LUCB outperforms a variant of RiskAverse-UCB-m-best \cite{David2018} (named RiskAverse-UCB-BAI) and VA-Uniform in identifying the risky arms and the optimal feasible arm. We fix the feasible arms and the threshold $\bar{\sigma}^2$ and vary the variance gaps $\Delta_i^\mathrm{v}$ of the infeasible arms. The accuracy parameters $\epsilon_\mu=\Delta_{i^\star}$ and $\epsilon^{\mathrm{v}}_j : =\min_{i\in\mathcal{R}\setminus\{i^\star\}}\Delta_i^\mathrm{v}$ in instance $j$.\footnote{Our notation $\epsilon^{\mathrm{v}}_j$ corresponds to  $\epsilon_{\mathrm{v}} $ in \cite{David2018} under the variance-constrained setup (for the $j^{\text{th}}$ instance).}  An illustration of the parameter setting of the arms is in Figure \ref{Competition} and the specific parameters for the arms in instance $j\in[10]$ are presented in Table~\ref{competition_para}. Note that the larger the index $j$, the riskier the instance as the true variances of the infeasible arms is closer to $\bar{\sigma}^2$ but their means are higher than that of the optimal feasible arm. This instance is apt for modeling  real-world investment settings in which there may be several   high-reward but risky options such as mini-bonds  or cryptoassets, and several other low-reward but less risky options such as real estate (which  appreciates with  time with high probability). 

Ohe results are presented in Figure~\ref{Comparisonbar}. VA-LUCB outperforms RiskAverse-UCB-BAI and VA-Uniform in all instances. 
In the  riskiest instance considered (i.e., the one with the smallest~$\epsilon_j^{\mathrm{v}}$),  VA-LUCB requires   $\approx\!32\%$ fewer arm pulls compared to  RiskAverse-UCB-BAI. 
\begin{figure}[t]
	\centering
	\includegraphics[width=0.42\textwidth]{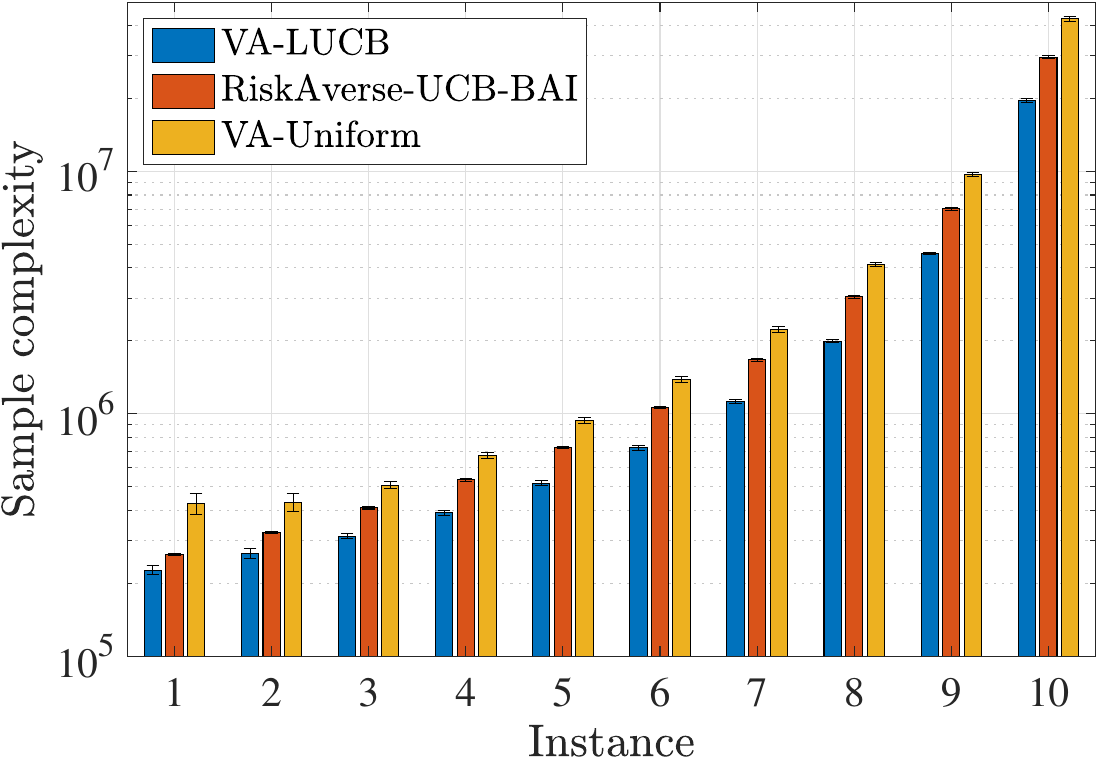}
	\caption{Comparison among the time complexities of VA-LUCB,  RiskAverse-UCB-BAI, and VA-Uniform  (error bars denote $1$ standard deviation across $20$ runs). As the index of the instance increases, the instance becomes riskier.}
	\label{Comparisonbar}
\end{figure}
\begin{figure}[t]
	\centering
	\includegraphics[width=.45\textwidth]{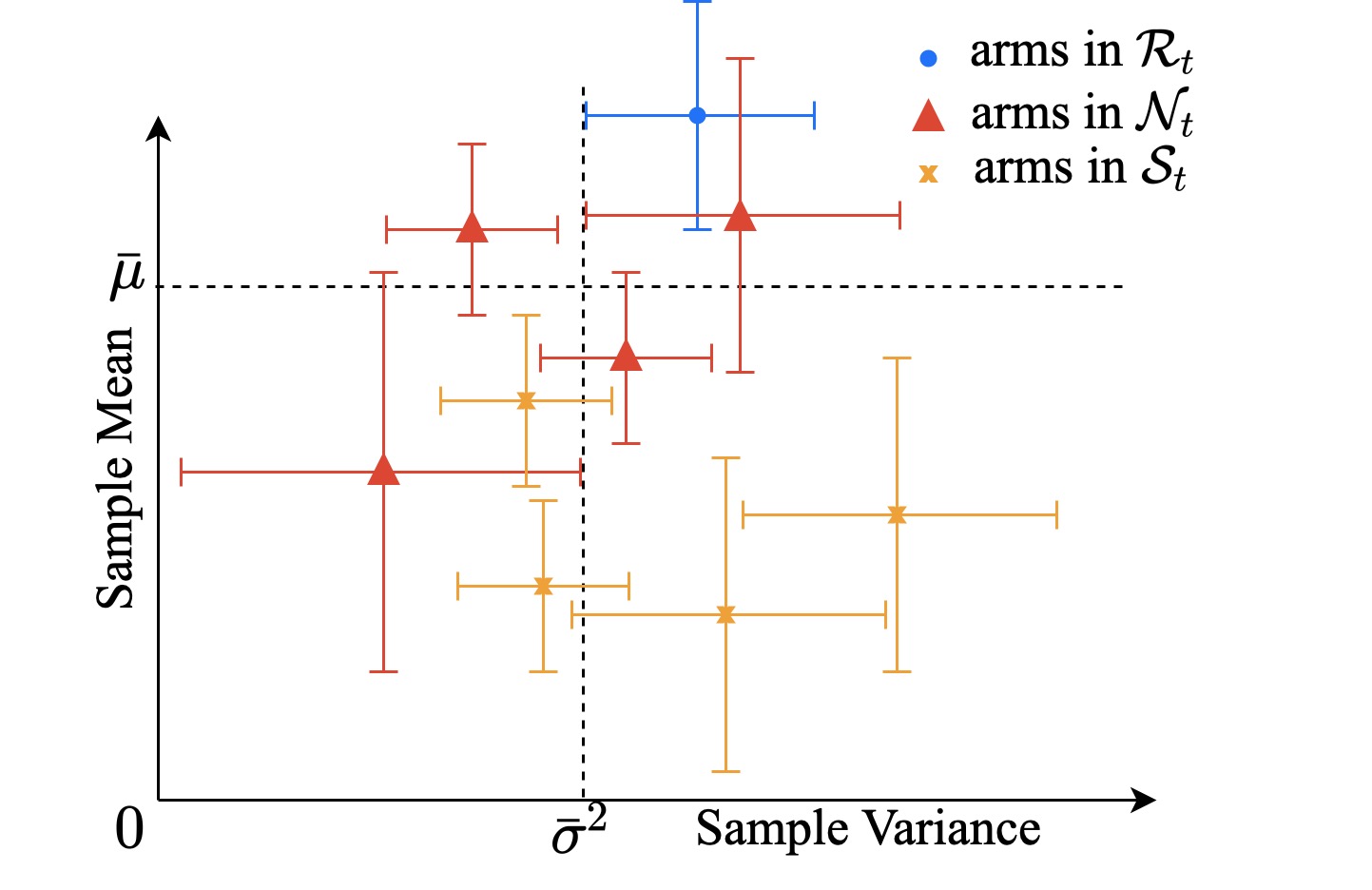}
	\caption{Illustration of the confidence intervals of the empirically suboptimal and the empirically risky sets}
	\label{fig:St}
\end{figure}

\section{Sketch of the Proof of Theorem \ref{thm:up_bd}}\label{proofsketch}
We extend the techniques used in the analysis of LUCB~\cite{Kalyanakrishnan2012}  to derive an upper bound on the sample complexity of VA-LUCB. To facilitate the analysis,   define the {\em empirically suboptimal set}, {\em empirically risky set} and the complement of their union respectively as
\begin{align}
	&\mathcal{S}_t :=\{i:U_i^\mu(t)< \bar{\mu} \},\quad 
	\mathcal{R}_t:=\{i:L_i^\mu(t)> \bar{\mu} \},\quad\mbox{and}
	\\
	&\mathcal{N}_t:=[N]\setminus(\mathcal{S}_t\cup\mathcal{R}_t)=\{i:L_i^\mu(t)\leq \bar{\mu} \leq U_i^\mu(t)\}.
\end{align}		
Note that $ \mathcal{S}_t $ and $ \mathcal{R}_t $ can be regarded as the empirical versions of  $ \mathcal{S} $ and $ \mathcal{R} $ respectively. Intuitively,   when $ t $ is  large enough,   $ \mathcal{S}_t=\mathcal{S}$ and $\mathcal{R}_t=\mathcal{R} $. We illustrate these sets in Figure~\ref{fig:St}. Define the events 
\begin{align}
	&E_i^\mu(t):=\{|\hat{\mu}_i(t)-\mu_i|\le\alpha(t,T_i(t))\}, \label{Ei}\\
	&E_i^{\mathrm{v}}(t):=\{ |\hat{\sigma}_i^2(t)-\sigma_i^2|\le\beta(t,T_i(t))\} ,\quad \mbox{and}
	\\
	&E_i(t):=E_i^\mu(t)\bigcap E_i^{\mathrm{v}}(t),\quad \forall\, i\in[N] . 
\end{align}
Finally, for $t\geq 2$, define
\begin{equation}\label{E}
	E(t):=\bigcap_{i\in[N]} E_i(t)  \quad \mbox{and}\quad  E:=\bigcap_{t\ge 2} E(t) .
\end{equation}
Conditioned on $ E $, we can show that  the empirical mean and   variance are accurate estimates of the true mean and variance respectively, in the sense that $ \mu_i\in[L_i^\mu(t),U_i^\mu(t)]$ and   $\ \sigma_i^2\in[L_i^{\mathrm{v}}(t),U_i^{\mathrm{v}}(t)]$ for all $i\in[N]$ and  $t\in\mathbb{N}$.
\begin{restatable}{lem}{Ehappens}\label{Ehappens}
	Define $E$ as in \eqref{E} with $\alpha(t,T)$ and $\beta(t,T)$ as in \eqref{radius}. Then $E$ occurs with probability at least $1- {\delta}/{2}$.
\end{restatable}
\begin{restatable}{lem}{goodtermination}\label{ioutisistar}
	Given an instance $ (\nu,\bar{\sigma}^2) $ with confidence parameter $ \delta $, on the event $ E(\tau) $,  and the termination of VA-LUCB, 
	\begin{itemize}
		\item if the instance is infeasible, $ \hat{\mathsf{f}}=\mathsf{f}=0 $.
		\item if the instance is feasible, $ i_\mathrm{out}=i_\tau=i_\tau^\star=i^\star,\hat{\mathsf{f}}=\mathsf{f}=1 $.
	\end{itemize}
\end{restatable} 
The proofs of the above lemmas are provided in  App.~\ref{upperbound}. Lemma \ref{ioutisistar} also justifies our stopping criterion.

What is left to do is to prove that VA-LUCB terminates at some finite time. We first state a useful  core lemma, which constitutes the main workhorse of the entire argument that VA-LUCB succeeds upon termination.
\begin{restatable}{lem}{usefullemma}\label{useful}
	On the event $ E(t) $, if VA-LUCB does not terminate, then at least one of the following statements holds:
	\begin{itemize}
		\item $ i_{t} \in  (\partial\mathcal{F}_t\backslash\mathcal{S}_t)\cup(\mathcal{F}_t\cap \mathcal{N}_t).$ 
		\item $ c_{t} \in  (\partial\mathcal{F}_t\backslash\mathcal{S}_t)\cup(\mathcal{F}_t\cap \mathcal{N}_t).$ 
	\end{itemize}
\end{restatable}
The   proof is presented in App.~\ref{upperbound}. When VA-LUCB has not terminated, there are three possible scenarios.
Firstly, the feasibility of the instance remains uncertain, i.e., $\mathcal{F}_t=\emptyset,\partial\mathcal{F}_t\neq\emptyset$. 
Secondly, the feasibility of $ i_{t} $ has not been confirmed, i.e., $ i_t\in\partial\mathcal{F}_t $ and $ i_t\neq i_t^\star $ (if $ i_t^\star $ exists).
Thirdly, the optimality of $ i_t $ has not been ascertained, i.e., $ U_{c_{t}}^\mu \geq L_{i_{t}}^\mu $.
Note that when only arm $ i_t $ is sampled, i.e., $ |\bar{\mathcal{F}}_t|=1 $ or $ U_{c_{t}}^\mu < L_{i_{t}}^\mu $, the optimality of $ i_t $ is guaranteed and we prove $  i_t \in  (\partial\mathcal{F}_t\backslash\mathcal{S}_t)\cup(\mathcal{F}_t\cap \mathcal{N}_t) $. Thus $ c_t $ does not need to be pulled at this time step. This strategy is essential in practice when the variances of the arms in $ \mathcal{R} $ are much closer to the threshold $\bar{\sigma}^2$ compared to the arms in $ \mathcal{S} $.

Lemma \ref{useful} indicates a sufficient condition for the termination of the algorithm. Namely, when neither of the arms $ i_t $ and $ c_t $ belongs to  $ (\partial\mathcal{F}_t\backslash\mathcal{S}_t)\cup(\mathcal{F}_t\cap \mathcal{N}_t)  $, the algorithm must have terminated.

Next, we show that after sufficiently many pulls of each arm, the set $ (\partial\mathcal{F}_t\backslash \mathcal{S}_t )\cup (\mathcal{F}_t\cap\mathcal{N}_t)$ remains nonempty with small probability.
For a sufficient large $ t $, let $ u_i(t) $ be the smallest number of pulls of a suboptimal arm $ i $ such that
$ \alpha(t,u_i(t)) $ is no greater than $ \Delta_i $, i.e.,
\begin{equation*}\label{ui}
	u_i(t):=\left\lceil{ \frac{1}{2 \Delta_i^2} \ln \left(\frac{2 N t^{4}}{\delta}\right)}\right\rceil
\end{equation*}
and $ v_i(t) $ be the smallest number of pulls of an arm $ i $ such that
$ \beta(t,u_i(t)) $ is no greater than $ \Delta_i^{\mathrm{v}}$, i.e.,
\begin{equation*}\label{vi}
	v_i(t):=\left\lceil{\frac{1}{2 (\Delta_i^{\mathrm{v}})^2} \ln \left(\frac{2 N t^{4}}{\delta}\right)}\right\rceil
\end{equation*}
Here we follow the convention: $ \frac{1}{0}=+\infty $ and $ \frac{1}{+\infty}=0 $, which may occur when $ \mathcal{F}=\emptyset $ or $ \mathcal{S}=\emptyset $.
\begin{restatable}{lem}{sufficientcondition}\label{sufcon}
	Using VA-LUCB, then 
	1) for $ i^\star $,
	\begin{equation*}\label{mean1}
		\mathbb{P}[T_{i^\star}(t)\!>\! 16u_{i^\star}(t), {i^\star}\!\notin\! \mathcal{R}_t]\!\leq\! \frac{\delta}{2(\frac{\Delta_{i^\star}}{2})^2 Nt^4}\! =:\!  A_1(i^\star)
	\end{equation*}
	2) for any suboptimal arm $ i \in\mathcal{S}$,
	\begin{equation*}\label{mean}
		\mathbb{P}[T_i(t)>16u_i(t), i\notin \mathcal{S}_t]\leq\frac{\delta}{2(\frac{\Delta_{i}}{2})^2 Nt^4}=:A_2(i)
	\end{equation*}
	3) for any feasible arm $ i\in\mathcal{F} $, 
	\begin{equation*}\label{var}
		\mathbb{P}[T_i(t)>4v_i(t),i\notin\mathcal{F}_t]\leq\frac{\delta}{2(\Delta_{i}^{\mathrm{v}})^2 Nt^4}=:A_3(i)			
	\end{equation*}
	4)	for any infeasible arm $ i\in\bar{\mathcal{F}}^c $, 
	\begin{equation*}\label{var1}
		\mathbb{P}[T_i(t)>4v_i(t),i\notin\bar{\mathcal{F}}_t^c]\leq\frac{\delta}{2(\Delta_{i}^{\mathrm{v}})^2 Nt^4}=:A_4(i)
	\end{equation*}
\end{restatable}
For a suboptimal arm $ i $, note that $ \Delta_i/2\leq\bar{\mu}-\mu_i \leq \Delta_i$. We compute 
$ \mathbb{P}[T_i(t)>16u_i(t), i\notin \mathcal{S}_t] $ in the same approach as \cite{Kalyanakrishnan2012}. This method is also utilized to analyze  the variances. 

Lemma \ref{sufcon} indicates the following:
\begin{itemize}
	\item For the best feasible arm $ i^\star $ (if it exists), after sampling it $ \max\{16u_{i^\star}(t),4v_{i^\star}(t)\} $ times, by using a union bound, $ i^\star \in\mathcal{F}_t\cap\mathcal{R}_t$ with failure probability at most $B_1(i^\star):=A_1(i^\star)+A_3(i^\star).$
	Therefore $i^\star\notin (\mathcal{F}_t\cap\mathcal{N}_t)\cup(\partial\mathcal{F}_t \backslash\mathcal{S}_t) $.
	\item For any feasible and suboptimal arm $ i\in\mathcal{F}\cap\mathcal{S} $, when $ T_i(t)>16u_i(t) $, $i\in \mathcal{S}_t $ with failure probability at most $B_2(i):=A_2(i). $
	\item For any arms $ i\in\bar{\mathcal{F}}^c\cap\mathcal{R}$, when $ T_i(t)>4v_i(t) $, $i\in \bar{\mathcal{F}}_t^c $ with failure probability at most $B_3(i):=A_4(i).$
	\item For arm $ i\in\bar{\mathcal{F}}^c\cap\mathcal{S}$, if it has been pulled more than $ \min\{16u_i(t),4v_i(t)\} $ times, $ i\in\bar{\mathcal{F}}_t^c\cup\mathcal{S}_t$ with failure probability at most $B_4(i):=A_2(i)\cdot\mathbbm{1}\{16u_i(t)<4v_i(t)\} 
	+A_4(i)\cdot\mathbbm{1}\{16u_i(t)\geq4v_i(t)\}$.
\end{itemize}
In conclusion, if all arms are pulled sufficiently many times, the probability that any of them stays in the set $ (\mathcal{F}_t\cap\mathcal{N}_t)\cup(\partial\mathcal{F}_t \backslash\mathcal{S}_t)$ is upper bounded by
\begin{equation}\label{failureprob}
	B_1(i^\star)\!+\!\sum_{i\in\mathcal{F}\cap\mathcal{S}}\!\! B_2(i)\!+\!\sum_{i\in\bar{\mathcal{F}}^c\cap\mathcal{R}}\!\!B_3(i)
	\!+\!\sum_{i\in\bar{\mathcal{F}}^c\cap\mathcal{S}}\!\! B_4(i).
\end{equation}

Finally, based on the above lemmas, we show that the algorithm dose not terminate with small probability after   time $t^\star=O(H_{\mathrm{VA}}\ln(H_{\mathrm{VA}}/\delta) ) $.
\begin{restatable}{lem}{stoppingtime}\label{stop}
	Let $ t^\star = 152H_{\mathrm{VA}}\ln(H_{\mathrm{VA}}/\delta)   $. 
    At any time step $ t>t^\star$, the probability that Algorithm~\ref{VALUCB} does not terminate is at most ${5\delta}/{t^2}$.
\end{restatable}
According to Lemma \ref{useful}, when neither arm $ i_t $ nor $ c_t $ belongs to $(\partial\mathcal{F}_t\backslash\mathcal{S}_t)\cup(\mathcal{F}_t\cap \mathcal{N}_t)$, the algorithm stops. In particular, if none of arms in $[N]$ is in $(\partial\mathcal{F}_t\backslash\mathcal{S}_t)\cup(\mathcal{F}_t\cap \mathcal{N}_t)$, the algorithm must terminate, which can be guaranteed by Lemma \ref{sufcon} with failure probability at most \eqref{failureprob}. The complete proof involves counting the numbers of pulls of the arms and estimating $ t^\star $. This is presented in App.~\ref{upperbound}.

\section{Conclusion and Future Direction}
We   proposed  framework for the risk-constrained Best Arm Identification problem and also developed an algorithm VA-LUCB whose sample complexity is almost optimal in the sense that its upper bound almost matches the information-theoretic lower bound. 
We highlight the VA-LUCB Algorithm constitutes a convenient framework to tackle any risk-aware BAI problem in the sense that it  is compatible with other concentration bounds, including LIL bounds.

However, we believe it is hard to derive an exact sample complexity using confidence bound-based algorithms, in the sense of nailing down the exact number
$$\liminf_{\delta\downarrow 0} \frac{\tau_\delta^\star}{\log \frac{1}{\delta}  \vphantom{\big)} },$$
where $\tau_\delta^\star$ is the minimum expectation of the stopping time for an algorithm to be $\delta$-PAC.

To characterize the exact asymptotic sample  complexity, we have explored adapting tracking-based algorithms such as Track and Stop (T\&S) from~\cite{Kaufmann2016} to the variance-constrained BAI problem. A lower bound similar to \cite{Kaufmann2016} can  be derived. For the corresponding algorithm, since the variances and the bound on the variance  complicate the alternative instances $ \operatorname{Alt}(\nu,\bar{\sigma}^2) $ for a given instance $ (\nu,\bar{\sigma}^2) $, the optimization  to obtain the optimal proportion of the arm pulls is difficult. In particular,  the (allocation vector) $w$ that attains the supremum in
$$\sup_{w : w_i> 0\,\forall\, i \in [K],\sum_{i=1}^Kw_i=1  } \inf_{(\nu^\prime,\bar{\sigma}^2) \in \operatorname{Alt}(\nu,\bar{\sigma}^2)}\sum_{i=1}^{N} w_{i} \mathrm{KL} (\nu_i, \nu_i^\prime )$$
 is difficult to characterize even for Gaussians because the variances (in addition to the means) are now  {\em variables} in the  inner optimization. This complicates the design and analysis of a {\em  constrained} T\&S-like algorithm, especially the sampling strategy. This is an promising direction for future research.
 

	\subsubsection*{Acknowledgements}
The authors would like to sincerely thank the two anonymous reviewers for their detailed and constructive reviews that have helped to improve the quality of the present paper.




\appendices
\allowdisplaybreaks[4]
\setcounter{table}{0}
\renewcommand{\thetable}{S.\arabic{table}}

\setcounter{figure}{0}
\renewcommand{\thefigure}{S.\arabic{figure}}
\setcounter{equation}{0}
\renewcommand{\theequation}{S.\arabic{equation}} 

\appendices
\begin{center}
	{\Large {\bf Appendices}}
\end{center}

In Appendix~\ref{sec:equality_case}, we discuss the necessity of the assumption  $\sigma_{i^{\star}}^2<\bar{\sigma}^2$.  
In Appendix~\ref{sec:RiskAverse}, a variant of RiskAverse-UCB-m-best~\cite{David2018} is presented. 
In Appendix~\ref{sec:discussion}, we  systematically compare the bounds on the sample complexity presented in this paper to those in \cite{David2018}. We also compare the assumptions needed to output the best feasible arm.   
In Appendix~\ref{upperbound}, we provide the detailed proofs of the lemmas used to prove Theorem~\ref{thm:up_bd}.
In Appendix~\ref{sec:subgaussian}, VA-LUCB is extended to VA-LUCB-sub-Gaussian, which deals with arms following $\sigma$-sub-Gaussian distributions.
In Appendix~\ref{LIL}, we discuss how to modify the analysis of VA-LUCB when the confidence radii are designed based on the non-asymptotic LIL (cf.\ Remark~\ref{rmk:LIL}).
In Appendix~\ref{lowerbound}, the complete proofs of Theorem~\ref{thm:low_bd} and Corollary~\ref{cor:almost_opt} are presented.
In Appendix~\ref{expdetail}, specific parameter settings and additional    numerical results are presented.

\section{Discussion of the case $\sigma_{i^\star}^2=\bar{\sigma}^2$} \label{sec:equality_case}

We assume $\sigma_{i^{\star}}^2<\bar{\sigma}^2$ in Section~\ref{sec:prob_setup} such that the problem is solvable by applying confidence-bound techniques without knowledge of any unknown parameter. We provide an explanation in this section.
Given a permissible bound on the variance $\bar{\sigma}^2$, it is natural to define the feasible set as
\begin{align}
	&\mathcal{F}:=\{ i\in [N]: \sigma_i^2 < \bar{\sigma}^2 \}
	\quad\text{or}\\
	&\mathcal{F}:=\{ i\in [N]: \sigma_i^2 \le \bar{\sigma}^2 \}. \label{eqn:defF}
\end{align}

First, with either choice of definition of $\mathcal{F}$,
inspired by Lemma~\ref{equ:concentration_inequ}, to ascertain there is no feasible arm and to terminate, an algorithm needs to check either
\begin{align}
	&\{ i\in [N]: U_i^{\mathrm{v}}(t) \le  \bar{\sigma}^2 \} = \emptyset
	\quad\mbox{or}\\
	&\{ i\in [N]: L_i^{\mathrm{v}}(t) \le  \bar{\sigma}^2 \} = \emptyset.
\end{align}
Since the feasible arms do not satisfy $U_i^{\mathrm{v}}(t) \le  \bar{\sigma}^2$ w.h.p. in the beginning, 
it is only reasonable to ascertain there is no feasible arm and terminate the algorithm when
$ \{ i\in [N]: L_i^{\mathrm{v}}(t) \le  \bar{\sigma}^2 \} = \emptyset$ as in our algorithm.

Note that with either choice of $\mathcal{F}$ in \eqref{eqn:defF}, we are confident  ascertaining that an arm is feasible if $U^\mathrm{v}_i(t) \le \bar{\sigma}^2$ and is infeasible if $L^\mathrm{v}_i(t) > \bar{\sigma}^2$.
We can only say an arm is possibly feasible with only $L^\mathrm{v}_i(t) \le \bar{\sigma}^2$.
Our termination rule is $\bar{\mathcal{F}}_t\cap\mathcal{P}_t = \emptyset$, where $\mathcal{F}_t$, $\bar{\mathcal{F}}_t$, and $\mathcal{P}_t$ are defined as in \eqref{f}--\eqref{p} and repeated here for easy reference: 
\begin{align*}
	\mathcal{F}_t &= 	\{ i: U^\mathrm{v}_i(t) \le \bar{\sigma}^2 \} ,\\
	\bar{\mathcal{F}}_t & = 	\{ i: L^\mathrm{v}_i(t) \le \bar{\sigma}^2 \} ,\\
	\mathcal{P}_t &=\left\{\begin{aligned}
		&	\{  i :   U^\mu_i(t) \ge L^\mu_{i^\star_t}(t), i \neq i^\star_t  \} ,&\mathcal{F}_t\neq\emptyset\\
		&	[N],&\mathcal{F}_t=\emptyset
	\end{aligned}\right. ,
\end{align*} 
where $i^\star_t = \arg\max_{ i\in \mathcal{F}_t } \hat{\mu}_i(t)$.

Next, we discuss each possible choice of $\mathcal{F} $ in \eqref{eqn:defF} individually.

\noindent
\underline{Choice 1}: $\mathcal{F}=\{ i\in [N]: \sigma_i^2 < \bar{\sigma}^2 \}$.  
Consider a case where there is an infeasible arm $j$ with $\sigma_{j}^2 =  \bar{\sigma}^2$ and $ \mu_j>\mu_{i^\star} $. After pulling  arms for a large number of times, w.h.p., we have
\begin{align*}
	&
	L^\mathrm{v}_{ j }(t) < \bar{\sigma}^2 < U^{\mathrm{v}}_{ j } (t) ~\Longrightarrow~ j\in\bar{\mathcal{F}}_t\setminus\mathcal{F}_t \qquad\mbox{and}
	\\
	&L^\mu_{j}(t) > U^\mu_{i^\star}(t)>L^\mu_{i^\star}(t), \quad 
	U^\mathrm{v}_{i}(t) \le \bar{\sigma}^2 \quad\forall\, i \in  \mathcal{F}
	\\
	&\Longrightarrow i^\star \in \mathcal{F}\subset \mathcal{F}_t,\quad  j \in \mathcal{P}_t,
\end{align*}
which implies that $j\in \bar{\mathcal{F}}_t\cap\mathcal{P}_t$ , and hence $\bar{\mathcal{F}}_t\cap\mathcal{P}_t \ne \emptyset$. In other words, the algorithm will never terminate w.h.p.

\noindent
\underline{Choice 2}: $\mathcal{F}=\{ i\in [N]: \sigma_i^2 \le \bar{\sigma}^2 \}$.  
Consider a case where  $\sigma_{i*}^2 =  \bar{\sigma}^2$.
Similar to the discussion above, we can see that $i^{\star}\notin \mathcal{F}_t,i^{\star}\in \bar{\mathcal{F}}_t\cap\mathcal{P}_t$ , and hence $\bar{\mathcal{F}}_t\cap\mathcal{P}_t \ne \emptyset$ w.h.p. Therefore, the algorithm will not terminate w.h.p.

Altogether, 
under either choice of the definition of the feasible set, any algorithm using UCB- and LCB-based termination rule will not terminate w.h.p.\ when there exists an arm $ i $ with high expectation and $\sigma_i^2= \bar{\sigma}^2 $. Thus, we define $\mathcal{F}=\{ i\in [N]: \sigma_i^2 \le \bar{\sigma}^2 \}$ and assume $ \sigma_{i^\star}^2< \bar{{\sigma}}^2 $ so that the algorithm will terminate in a finite number of time steps w.h.p.\ when a confidence bound-based algorithm is employed.

\noindent
\textbf{Additional prior knowledge.}
We note that the variant of RiskAverse-UCB-m-best algorithm proposed by David {\em et al.}~\cite{David2018}, RiskAverse-UCB-BAI, can also be applied to identify the best feasible arm (without any suboptimality or subfeasibility) under the $\delta$-PAC framework only when $\min_{i\in \mathcal{R}\setminus \{ i^\star \}} \Delta_i^\mathrm{v} $ is known (see Appendix~\ref{sec:discussion} for detailed discussion).
Though it can be applied when $ \sigma_{i^\star}^2= \bar{\sigma}^2 $, we remark that our algorithm can also handle this case ($ \sigma_{i^\star}^2= \bar{\sigma}^2$)  with such additional prior knowledge on the parameters.
In detail, we regard $\epsilon_\mathrm{v}$ be an optional parameter of our algorithm (set as it to be  $0$ if $\min_{i\in \mathcal{R}\setminus \{ i^\star \}} \Delta_i^\mathrm{v} $ is  unknown) and define 
\begin{equation*}
	\bar{\mathcal{F}}_t  := 	\{ i: L^\mathrm{v}_i(t) \le \bar{\sigma}^2 \},\quad\mathcal{F}_t := \{ i\in \bar{\mathcal{F}}_t: U^\mathrm{v}_{i}(t) \le  \bar{\sigma}^2 +  \epsilon_\mathrm{v} \}.
\end{equation*}
We set $\epsilon_\mathrm{v}:=\min_{i\in \mathcal{R}\setminus \{ i^{\star} \}} \Delta_i^\mathrm{v} $ when the quantity is known and we are not sure if $ \sigma_{i^\star}^2< \bar{\sigma}^2 $ (i.e., it is possible that   $ \sigma_{i^\star}^2= \bar{\sigma}^2$).
With the prior knowledge of $\epsilon_\mathrm{v}=\min_{i\in \mathcal{R}\setminus \{ i^{\star} \}} \Delta_i^\mathrm{v} $, the upper bound of the sample complexity of VA-LUCB can be improved to 
$ O\big( \widetilde{H}_{\mathrm{VA}} \ln\frac{\widetilde{H}_{\mathrm{VA}}}{\delta}\big)$
w.h.p., where
\begin{align}
	\widetilde{H}_{\mathrm{VA}}&:= \frac{1}{\min\{\frac{\Delta_{i^\star}}{2},\Delta_{i^\star}^{\mathrm{v}}+\epsilon_v\}^2 }
	+\sum_{i\in\mathcal{F}\cap\mathcal{S}}\frac{1}{(\frac{\Delta_i}{2})^2}\\
	&\;+\sum_{i\in\bar{\mathcal{F}}^c\cap\mathcal{R}}\frac{1}{(\Delta_{i}^{\mathrm{v}})^2 }
	+\sum_{i\in\bar{\mathcal{F}}^c\cap\mathcal{S}}\frac{1}{\max\{\frac{\Delta_i}{2},\Delta_i^{\mathrm{v}}\}^2} < H_{\mathrm{VA}}.
\end{align}
This enables us not only to deal with the case $\sigma_{i^\star}=\bar{\sigma}^2,$ but also facilitates in ascertaining the feasibility of the best feasible arm in the usual instance in which  $\sigma_{i^\star}^2<\bar{\sigma}^2$. Therefore, the sample complexity is better than the current VA-LUCB algorithm, as well as RiskAverse-UCB-BAI algorithm to be discussed extensively in Appendix \ref{sec:discussion}.
The proof just follows the same procedure as in Section \ref{proofsketch}.

\section{RiskAverse-UCB-BAI}\label{sec:RiskAverse}
We present a variant of RiskAverse-UCB-m-best algorithm from \cite{David2018}, named RiskAverse-UCB-BAI, which is adapted to our variance-constrained BAI setup. To avoid any confusion, we redefine the sample mean, sample variance and confidence bounds, which are consistent with the notations in \cite{David2018}. For arm $i\in[N]$, define 
\begin{itemize}
	\item the counter:  
	\begin{equation}\label{T_prime}
		T_i^\prime(t):=\sum_{s=1}^{t}\mathbbm{1}\{i=i_s^\dag\};
	\end{equation}
	\item the sample mean and the sample variance respectively as: 
	\begin{align}
		&\hat{\mu}_i^\prime(t):=\frac{1}{T_i^\prime(t)}\sum_{s=1}^{t} X_{s,	i_s^\dag}\mathbbm{1}\{i =i_s^\dag \},\quad\mbox{and}
		\\
		&(\hat{\sigma}_i^\prime)^2(t) :=
		\frac{\sum_{s=1}^{t}\big(X_{s,i_s^\dag}-\hat{\mu}_i(t) \big)^2\mathbbm{1}\{i=i_s^\dag\}}{T_i^\prime(t)-1};
	\end{align}
	\item the confidence radii for the mean and variance respectively as
	\begin{align}\label{equ:David_radius}
		&f_\mu(T):=\sqrt{\frac{1}{2T}\ln\left(\frac{6HN}{\delta}\right)},\quad\mbox{and}\\
		&f_\mathrm{v}(T):=\sqrt{\frac{2}{T}\ln\left(\frac{6HN}{\delta}\right)};
	\end{align}
	\item the confidence bounds for the mean:
	\begin{align} 
		&L_i^{\mu^\prime}(t):= \hat{\mu}_i^\prime(t)-f_\mu(T_i^\prime(t)),\quad\mbox{and}
		\\
		&U_i^{\mu^\prime}(t):= \hat{\mu}_i^\prime(t)+f_\mu(T_i^\prime(t));
	\end{align}
	\item the confidence bounds for the variance:
	\begin{align} 
		&L_i^{\mathrm{v}^\prime}(t):=( \hat{\sigma}_i^\prime)^2(t)-f_\mathrm{v}(T_i^\prime(t)), \quad\mbox{and}
		\\
		&U_i^{\mathrm{v}^\prime}(t):=( \hat{\sigma}_i^\prime)^2(t)+f_\mathrm{v}(T_i^\prime(t)).
	\end{align}
\end{itemize}
The algorithm, RiskAverse-UCB-BAI, an adaptation of RiskAverse-UCB-m-best~\cite{David2018} to our variance-constrained setting, is presented in Algorithm~\ref{RiskAverse}.
\begin{algorithm}
	\caption{RiskAverse-UCB-BAI (A variant of RiskAverse-UCB-m-best~\cite{David2018})}
	\begin{algorithmic}[1]
		\STATE \textbf{Input}: threshold $\bar{\sigma}^2 \!>\!0$, confidence parameter $ \delta\! \in\! (0,1)$ and accuracy parameters $\epsilon_\mu,\epsilon_\mathrm{v}\in(0,+\infty)$.
		\FOR{each $i\in [N]$ }
		\STATE Sample arm $i$ twice, update the counter $T_i^\prime(N)$, the sample mean $\hat{\mu}_i^\prime(N)$ and sample variance $(\hat{\sigma}_i^\prime)^2(N)$.
		\ENDFOR
		\STATE Set 
		 $H:=3N\left(\frac{1}{2\epsilon_\mu^2}+\frac{4}{\epsilon_\mathrm{v}^2}\right)\ln\left(\frac{6N}{\delta}N\left(\frac{1}{2\epsilon_\mu^2}+\frac{4}{\epsilon_\mathrm{v}^2}\right)\right)$ and $t:=2N$.
		\REPEAT
		\STATE Set $\bar{\mathcal{F}}_t^\prime:=\{i:L_i^{\mathrm{v}^\prime}(t)\leq \bar{\sigma}^2\}$.
		\STATE Select an optimistic arm
		$i_{t+1}^\dag:=\mathop{\rm argmax}_{i\in\bar{\mathcal{F}}_t^\prime}U_i^{\mu^\prime}(t)$.
		\STATE Draw a sample from the selected arm $i_t^\dag$.
		\STATE Set $t=t+1$.
		\STATE Update the counter $T_i^\prime(t)$ and estimates $\hat{\mu}_i^\prime(t)$ and $(\hat{\sigma}_i^\prime)^2(t)$ for all arm $i\in[N]$ accordingly.
		\UNTIL{$\left(\left[\right.\right.f_\mu(T_{i_{t}^\dag}^\prime(t))\leq\epsilon_\mu/2$ and $(\hat{\sigma}_{i_{t}^\dag}^\prime)^2(T_i^\prime(t))-\epsilon_\mathrm{v}\leq \bar{\sigma}^2  \left.\right]$ or $t\geq H\left.\right)$.}
		\STATE Return $i_t^\dag$
	\end{algorithmic}\label{RiskAverse}
\end{algorithm}
Since Algorithm~\ref{RiskAverse} only guarantees to output an {\em $\epsilon_\mathrm{v}$-approximately feasible and $\epsilon_\mu$-approximately optimal} arm, 
in order to output the best feasible arm, the accuracy parameters have to be sufficiently small such that the only {\em $\epsilon_\mathrm{v}$-approximately feasible and $\epsilon_\mu$-approximately optimal} arm is the best feasible arm. See Appendix~\ref{sec:discussion} for details.
\begin{remark}
    We remark that Algorithm~\ref{VALUCB} can also be adapted to the BAI problem with an $\alpha$-quantile constraint by replacing the sample variance and its associated confidence bound by the sample $\alpha$-quantile and the corresponding concentration bound (see \cite[Lemma~6]{David2018}). The modified Algorithm~\ref{VALUCB} is completely parameter-free, whereas RiskAverse-UCB-m-best\cite{David2018} is not. The sample complexity of the modified Algorithm~\ref{VALUCB} can be derived in a similar procedure as in this paper.
\end{remark}

\section{Discussion of the bounds in  David {\em et al.}~\cite{David2018}} \label{sec:discussion}
For RiskAverse-UCB-BAI to identify the best feasible arm (without any suboptimality or subfeasibility)  under the $\delta$-PAC framework, it needs to ensure that parameters $\epsilon_{\mathrm{v}}$ and $\epsilon_\mu$ are set sufficiently small so that  the $\epsilon_\mathrm{v}$-approximately feasible and $\epsilon_\mu$-approximately optimal arm is exactly the best feasible arm. A sufficient condition is $\epsilon_\mu<\Delta_{i^\star}$ and $\epsilon_\mathrm{v}<\min_{i\in\mathcal{R}\setminus\{i^\star\}}\Delta_i^\mathrm{v}$. Without the former/latter condition, a suboptimal/risky arm maybe produced by the RiskAverse-UCB-BAI. 
However, even we relax the accuracy parameters by allowing them to assume  equality, i.e., $\epsilon_\mu=\Delta_{i^\star}$ and $\epsilon_\mathrm{v}=\min_{i\in\mathcal{R}}\Delta_i^\mathrm{v}$, as well as that $H$ is given, we can still assert that   VA-LUCB is superior in terms of the sample complexity; this is what we do in Section~\ref{sec:discussion_upper}.  In addition, the confidence radii of the mean and variance \eqref{equ:David_radius} contain $H$, the hardness parameter that depends on the instance which is not known in practice, further underscoring that RiskAverse-UCB-BAI is not parameter free. 
In the following discussion, we recall that  random variables bounded in $[0,1]$ are $1/2$-subgaussian.
\subsection{Discussion of the Upper Bounds}\label{sec:discussion_upper}
The upper bound of the sample complexity of a variant of RiskAverse-UCB-m-best presented in \cite[Theorem~3]{David2018}, which we call RiskAverse-UCB-BAI, and analyze using techniques along the same lines   is
\begin{align}\label{equ:David_upperbd}
	\sum_{i\in[N]}C_i\ln\left(\frac{6NH}{\delta}\right)
\end{align}
where 
\begin{align}
	\!\! H:=3N\left(\frac{1}{2\epsilon_\mu^2}+\frac{4}{\epsilon_\mathrm{v}^2}\right)\ln\left(\frac{6N}{\delta}N\left(\frac{1}{2\epsilon_\mu^2}+\frac{4}{\epsilon_\mathrm{v}^2}\right)\right)\label{equ:David_hardness1}
\end{align} 
and
\begin{align}
	\!\! C_i :=&\min\left\{ \frac{1}{\max\left\{0,\mu_{i^\star}-\mu_i\right\}^2},\frac{4}{\max\left\{0,\sigma_i^2-\bar{\sigma}^2\right\}^2},\right.
	\\*
	& \;\left.\max\bigg\{ \frac{1}{\epsilon_\mu^2}, \frac{4}{\max\left\{0,\epsilon_\mathrm{v}-(\sigma_i^2-\bar{\sigma}^2)\right\}^2}\bigg\}
	\right\}\label{equ:David_hardness2}
\end{align} 
for all $i\in[N]$.
For the sake of brevity, define $H^\prime:=3N\left(\frac{1}{2\epsilon_\mu^2}+\frac{4}{\epsilon_\mathrm{v}^2}\right)$. Then the upper bound in~\eqref{equ:David_upperbd} can be rewritten as $\sum_{i\in[N]}C_i\ln\Big(\frac{6NH^\prime\ln(\frac{2NH^\prime}{\delta})}{\delta}\Big)$. 
Given the similar roles of $H_\mathrm{VA}$ and $H^\prime$ in the upper bounds, we  can also regard $H^\prime$ as another hardness parameter in \cite{David2018} (in addition to $\{C_i\}_{i\in [N]}$). 
Since both $\sum_{i\in[N]}C_i$ and $H$ appear in the upper bound \eqref{equ:David_upperbd}, we carefully compare both of them to  $H_\mathrm{VA}$.
We firstly compare the terms in $H_\mathrm{VA}$ with $C_i$:
\begin{itemize}
	\item For arm $i^\star$, 
	$$C_{i^\star}= \max\left\{\frac{1}{\epsilon_\mu^2}, \frac{4}{(\epsilon_\mathrm{v}-(\sigma_{i^\star}^2-\bar{\sigma}^2))^2}\right\}
	\geq \frac{1}{\min\left\{\Delta_{i^\star},\Delta_{i^\star}^\mathrm{v}\right\}^2 }
	$$
	where equality holds if $\Delta_{i^\star}^\mathrm{v}=\epsilon_\mathrm{v}$.
	\item For any feasible and suboptimal arm $i$, $C_i=\frac{1}{\Delta_i^2}.$
	\item For any risky arm $i\neq i^\star$, $C_i=\frac{4}{(\Delta_i^\mathrm{v})^2}.$
	\item For any infeasible and suboptimal arm $i$, 
	$$C_i=\min\left\{ \frac{1}{\Delta_i^2},\frac{4}{(\Delta_i^\mathrm{v})^2}\right\}=
	\frac{1}{\max\left\{\Delta_i,\Delta_i^\mathrm{v}/2\right\}^2}.
	$$
\end{itemize}
This trivially leads to $ \sum_{i\in[N]}C_i\ge H_\mathrm{VA}/4$. 
In terms of $H$,
note that $\epsilon_\mu=\Delta_{i^\star},\epsilon_\mathrm{v}=\min_{i\in\mathcal{R}}\Delta_i^\mathrm{v}$, so $H^\prime> 3H_\mathrm{VA}/8$ trivially holds. However, in a practical instance where the means and variances of the arms are diverse, $H^\prime> H_\mathrm{VA}$, e.g., when $\Delta_{i^\star}\ge\epsilon_\mathrm{v}/2$ (this can be interpreted as  the scenario in which  identifying the risky arms is more difficult than ascertaining the optimality of the best feasible arm) or there are at most $\lceil N/6\rceil $ suboptimal arms with $\Delta_i\leq 2\Delta_{i^\star}$,
$H^\prime> H_\mathrm{VA}$ holds.
Therefore, the upper bound in \cite{David2018} is 
\begin{align}\label{equ:David_upperbd1}
	&\underbrace{\sum_{i\in[N]}C_i\ln\bigg(\frac{6NH^\prime\ln(\frac{2NH^\prime}{\delta})}{\delta}\bigg)}_{ \text{Upper bound (UB) in \cite{David2018}} } \\
	&=\underbrace{\Omega\bigg(H_\mathrm{VA}\ln\bigg(\frac{NH_\mathrm{VA} \ln ( \frac{NH_{\mathrm{VA}}}{\delta})}{\delta}\bigg)\Bigg)}_{\text{Order-wise result of UB in \cite{David2018}}} \\
	&= \omega\Bigg( \underbrace{H_\mathrm{VA}\ln\left(\frac{H_\mathrm{VA}}{\delta}\right)}_{\text{Our upper bound up to constants}}\Bigg).
\end{align}
Even disregarding constants and the fact that RiskAverse-UCB-BAI is not parameter free if we demand that the (strictly) best feasible arm is output by the algorithm, we note  the  presence of the additional $\ln$ term in $N$ and the $\ln\ln$ term in $NH_{\mathrm{VA}}/\delta$ in the order-wise result of the upper bound in \cite{David2018}. 
We conclude that the upper bound of the sample complexity of  RiskAverse-UCB-BAI in  \eqref{equ:David_upperbd} \cite{David2018} is strictly larger in order than ours.

\subsection{Discussion of the Lower Bounds}\label{sec:discussion_lower}
While the lower bound of \cite[Theorem~2]{David2018} holds under a set of assumptions, we assume that these assumptions are generally not needed and the only assumption made here is that $\delta\in(0,0.01)$. 
The lower bound in \cite{David2018} is
\begin{align}\label{equ:David_lowerbd}
	\min _{i^{\prime} \in [N]} \sum_{i \in[N] \setminus \left\{i^{\prime}\right\}} 
	&\frac{\ln \left(\frac{1}{9 \delta}\right)}{900}\max \left\{\left(16 \Delta_i^\mathrm{v}\right)^{2},\right.\\
	&\left.\left(5 \Delta_{i^\star}+4 \max\{0,\mu_{i^\star}-\mu_i\}\right)^{2}\right\}^{-1}-N.
\end{align}
We also compare the denominator term-by-term:
\begin{itemize}
	\item For arm $i^\star$, 
	\begin{align}
		&\max \left\{\left(5 \Delta_{i^\star}+4 \max\{0,\mu_{i^\star}-\mu_{i^\star}\}\right)^{2},\left(16 \Delta_{i^\star}^\mathrm{v}\right)^{2}\right\}
		\\
		&=
		\max \left\{5 \Delta_{i^\star},16 \Delta_{i^\star}^\mathrm{v}\right\}^2\ge
		\min \left\{5 \Delta_{i^\star},16 \Delta_{i^\star}^\mathrm{v}\right\}^2
	\end{align}
	where   equality holds if and only if $5 \Delta_{i^\star}=16 \Delta_{i^\star}^\mathrm{v}$.
	\item For any feasible and suboptimal arm $i$, 
	\begin{align}
		&\max \left\{\left(5 \Delta_{i^\star}+4 \max\{0,\mu_{i^\star}-\mu_{i}\}\right)^{2},\left(16 \Delta_{i}^\mathrm{v}\right)^{2}\right\}
		\\
		&= 
		\max \left\{5 \Delta_{i^\star}+4\Delta_i,16 \Delta_{i}^\mathrm{v}\right\}^2>
		16\Delta_i^2.
	\end{align}
	\item For  any risky arm $i\neq i^\star$, 
	\begin{align}
		&\max \left\{\left(5 \Delta_{i^\star}+4 \max\{0,\mu_{i^\star}-\mu_{i}\}\right)^{2},\left(16 \Delta_{i}^\mathrm{v}\right)^{2}\right\}
		\\
		&= 
		\max \left\{5 \Delta_{i^\star},16 \Delta_{i}^\mathrm{v}\right\}^2\ge
		\left(16 \Delta_{i}^\mathrm{v}\right)^2.
	\end{align}
	\item For any infeasible and suboptimal arm $i$, 
	\begin{align}
		&\max \left\{\left(5 \Delta_{i^\star}+4 \max\{0,\mu_{i^\star}-\mu_{i}\}\right)^{2},\left(16 \Delta_{i}^\mathrm{v}\right)^{2}\right\}
		\\
		&= 
		\max \left\{5 \Delta_{i^\star},16 \Delta_{i}^\mathrm{v}\right\}^2.
		\end{align}
\end{itemize}
Therefore, the lower bound \eqref{equ:David_lowerbd} is strictly smaller than
\begin{align} 
	&\left(\frac{1}{\min\left\{5 \Delta_{i^\star},16 \Delta_{i^\star}^\mathrm{v}\right\}^2}
	+   \sum_{i\in\mathcal{F}\cap\mathcal{S}}\frac{1}{16\Delta_i^2}
	+   \sum_{i\in\bar{\mathcal{F}}^c\cap\mathcal{R}}\frac{1}{(16{\Delta_i^\mathrm{v}})^2}\right.
	\\
	&\quad\left.
	+   \sum_{i\in\bar{\mathcal{F}}^c\cap\mathcal{S}}\frac{1}{\max \left\{5 \Delta_{i^\star},16 \Delta_{i}^\mathrm{v}\right\}^2}
	\right)
\frac{\ln \left( \frac{1}{9\delta} \right)}{900}.
\end{align}
This is strictly smaller than our lower bound in \eqref{eqn:lb} (see Theorem \ref{thm:low_bd}) when $\bar{\sigma}^2\ge 9\cdot 10^{-5}$ and $\mu_{i^\star}\le1-9\cdot 10^{-5}$, which is a large subset  of practical instances (recalling that the rewards are bounded in $[0,1]$). In fact,  the space of $(\mu,\bar{\sigma}^2)$ for which our lower bound is strictly better than that in \cite{David2018} is $> 1-7.2\times 10^{-4}>99.9\%$ times of the total area of the  permissible parameter space of $(\mu,\bar{\sigma}^2) \in [0,1]\times[0,1/4)$.

Considering both the upper and lower bounds, as well as Corollary \ref{cor:almost_opt}, even though we have relaxed several assumptions in~\cite{David2018}, the bounds in \cite{David2018} (performance upper bound on the sample complexity of RiskAverse-UCB-BAI and lower bound) are  looser than ours.
Furthermore, the term of arm $i$ in the lower bound~\eqref{equ:David_lowerbd} does not match the corresponding term in the upper bound~\eqref{equ:David_upperbd} or $H^\prime$, showing that the terms $\sum_{i\in[N]} C_i$ and $H^\prime$ do not   characterize the inherent difficulty of identifying the best feasible arm. In contrast, we have shown that the optimal sample complexity of identifying the best feasible arm is  characterized exactly by~$H_{\mathrm{VA}}$.

\subsection{Discussion of the Complexity of Identifying  Risky arms}\label{sec:discussion_riskyarms}
Since we are considering risk-constrained bandits, an important task is to identify the risky arms as quickly as possible.
Based on the discussion of the   accuracy parameters in the previous sections, we investigate the convergence speed of the confidence radius of the variance, which is essential in eliminating the risky arms. The confidence radius of risky arm $i\neq i^\star$ at round $t$ in \cite{David2018} is
\begin{align}\label{equ:David_cr}
	f_\mathrm{v}(T_i^\prime(t))=\sqrt{\frac{2}{T_i^\prime(t)}\ln\frac{6HN}{\delta}}
\end{align}
where $T_i^\prime(t)$ is defined in \eqref{T_prime}.
We claim that \eqref{equ:David_cr} is strictly greater than $\beta$ in~\eqref{radius} in  a generic case  as follows:\footnote{Due to the difference in algorithms (VA-LUCB vs.\  RiskAverse-UCB-BAI)  and the definitions of $T_i(t)$ (in this paper and in \cite{David2018}), one  should use $T_i^\prime(2t-2)$ in $f_\mathrm{v}$  in~\eqref{equ:radius_compare} to be consistent with  Algorithm~\ref{RiskAverse}. However, in order to be fair when comparing the confidence radii, we  assume there are two identical risky arms for the two algorithms.  Thus both algorithm will pull the two arms approximately the same number of times. Hence, the denominator in the definition of  $f_\mathrm{v}$ is roughly $T_i(t)$.}
\begin{align}\label{equ:radius_compare}
	 f_\mathrm{v}(T_i(t))&=\sqrt{\frac{2}{T_i(t)}\ln\frac{6HN}{\delta}}\\
	&\geq \beta(t,T_i(t))=\sqrt{\frac{1}{2 T_i(t)} \ln \left(\frac{2 N t^{4}}{\delta}\right)}
\end{align} 
This is equivalent to 
\begin{align}
	&4\ln\frac{6HN}{\delta}\geq \ln \frac{2 N t^{4}}{\delta}\quad
	\Longleftrightarrow \quad
	\left(\frac{6HN}{\delta}\right)^{4}\geq \frac{2 N t^{4}}{\delta}\\
	&\Longleftrightarrow\quad
	\frac{6}{ 2^{1/4}} \left(\frac{N}{\delta}\right)^{3/4}H^\prime \ln\frac{2NH^\prime}{\delta}\geq 152H_\mathrm{VA}\ln\frac{H_\mathrm{VA}}{\delta},
\end{align}
if $N\ge10$ and $\delta\le0.05$, $\frac{6}{ 2^{1/4}} \left(\frac{N}{\delta}\right)^{3/4}>152$.
According to the comparison between $H^\prime$ and $H_\mathrm{VA}$ in Section~\ref{sec:discussion_upper}, $H^\prime>H_\mathrm{VA}$ in a general instance, thus the last inequality holds even we ignore the logarithmic term in $N$. In particular, the greater $N$ is and the more risky arms, the larger $f_\mathrm{v}(T_i(t))$ is. 
Hence,  the convergence speed of the confidence radius in $\beta$ will be faster than that of \eqref{equ:David_cr}, resulting fewer arm pulls of the risky arms of VA-LUCB. This indicates VA-LUCB is more efficient  in the risk-constrained setup. This is also corroborated by our experiments in Section~\ref{sec:uniform}.

From the experimental/practical point of view, the constant $152$ on the right-hand side of \eqref{equ:radius_compare} can essentially be replaced by~$3$ as our experiments indicate; see Section~\ref{sec:perf_VA}. Furthermore, both algorithms will identify the risky arms first, thus $t$ in $\beta$ is small at the beginning. We refer to Section~\ref{sec:uniform} for further experimental validations.

\section{Proof of Upper Bound}\label{upperbound}

\begin{lem}[Implication of Hoeffding's and McDiarmid's Inequalities]\label{equ:concentration_inequ}
	Given an instance $ (\nu,\bar{\sigma}^2) $, for any arm $ i $ with $T_i(t)\geq2 $ and $\varepsilon>0 $ we have
	\begin{align}\label{hoe}
		\mathbb{P}[\hat{\mu}_i(t)-\mu_i\geq\varepsilon]&\leq \exp(-2T_i(t)\varepsilon^2),\\
		\mathbb{P}[\hat{\mu}_i(t)-\mu_i\leq-\varepsilon]&\leq \exp(-2T_i(t)\varepsilon^2),
	\end{align}
	and 
	\begin{align}\label{mc}
		\mathbb{P}[\hat{\sigma}_i^2(t)-\sigma_i^2\geq\varepsilon]&\leq \exp(-2T_i(t)\varepsilon^2),\\
		\mathbb{P}[\hat{\sigma}_i^2(t)-\sigma_i^2\leq-\varepsilon]&\leq \exp(-2T_i(t)\varepsilon^2).
	\end{align}
\end{lem}
\begin{proof}
	Note that since the reward distribution is bounded in $ [0,1] $, \eqref{hoe} can be derived by a straightforward application of Hoeffding's inequality.
	As for the sample variance, note  that for i.i.d.\ random variables $ X_1,\ldots,X_n $, supported on $ [0,1] $ with sample mean $ \hat{\mu} $, the unbiased sample variance can be written as
	\begin{align}\label{sv1}
		\hat{\sigma}^2(n)&=\frac{1}{n-1}\sum_{i=1}^n(X_i-\hat{\mu})^2\\
		&=\frac{1}{n(n-1)}\sum_{i<j}(X_i-X_j)^2
	=:f(X_1,X_2,\ldots,X_n)
	\end{align}
	Note  by the unbiasedness of the sample variance that $ \mathbb{E}[\hat{\sigma}^2(n)]=\sigma^2 =\textrm{Var}( X_i )$ and 
	\begin{align}
		&\big|f(x_1,\ldots,x_{i-1},x_i,x_{i+1},\ldots,x_n)\\
		&\;-f(x_1,\ldots,x_{i-1},x_i^\prime,x_{i+1},\ldots,x_n)\big|\leq\frac{1}{n}
	\end{align}
	for any $x_1,\ldots, x_{i-1}, x_i,x_i^\prime, x_{i+1},\ldots, x_n\in[0,1] $. Applying McDiarmid's inequality \cite{mcdiarmid1989} to $f$, we get 
	\begin{align}\label{mc1}
		&\mathbb{P}[\hat{\sigma}^2(n)-\sigma^2\geq\varepsilon]\leq \exp(-2n\varepsilon^2),\\
		&\mathbb{P}[\hat{\sigma}^2(n)-\sigma^2\leq-\varepsilon]\leq \exp(-2n\varepsilon^2).
	\end{align}
	Apply \eqref{mc1} to the arms, we obtain \eqref{mc}.
\end{proof}

\begin{lem}\label{ctinpt}
	Conditioned on $ E $, if $ |\bar{\mathcal{F}}_t|>1 $ and Algorithm \ref{VALUCB} has not terminated, then $ c_{t}\in\mathcal{P}_t\cup\{i_t^\star\} $.
\end{lem} 
\begin{proof}
	There are two trivial scenarios. Firstly, when $\mathcal{F}_t=\emptyset$, $\mathcal{P}_t=[N]$. Secondly, $\mathcal{F}_t\neq\emptyset$ and $c_t= i_t^\star$. The result is straightforward.
	
	Consider the case where $\mathcal{F}_t\neq\emptyset$ and $c_t\neq i_t^\star$,
	when $ |\bar{\mathcal{F}}_t|>1$ and the algorithm has not terminated, we must have $ c_t\in\mathcal{P}_t $, i.e.,
	\begin{equation}\label{case2}
		U_{c_t}^\mu(t)\geq L_{i_t^\star}^\mu(t)
	\end{equation}
	Otherwise, conditioned on the event $ E $,
	\begin{align}\label{in2}
		\hat{\mu}_i(t)&< U_{i}^\mu(t)\leq U_{c_t}^\mu(t)< L_{i_t^\star}^\mu(t)\\
		&<\hat{\mu}_{i_t^\star} (t)\leq\hat{\mu}_{i_t} (t)<U_{i_t}^\mu(t),\quad\forall i\in\bar{\mathcal{F}}_t\setminus\{i_t\}.
	\end{align}
	If $ i_t\neq i_t^\star $, we have $ U_{c_t}^\mu(t)< L_{i_t^\star}^\mu(t)<\hat{\mu}_{i_t^\star} (t)\leq U_{i_t^\star}^\mu(t) $ which contradicts the definition of $ c_t $. 
	Thus $ i_t=i_t^\star $ must hold. In this case, $ U_{i}^\mu(t)\leq U_{c_t}^\mu(t)< L_{i_t^\star}^\mu(t)$ for all  $i\in\bar{\mathcal{F}}_t\backslash\{i_t^\star\}$, i.e., $ \bar{\mathcal{F}}_t\cap\mathcal{P}_t=\emptyset $. This contradicts the assumption that the algorithm does not terminate. Therefore, $ c_t\in\mathcal{P}_t $.
\end{proof}

\begin{lem}\label{integral}
	The function $ h:\mathbb{R}\rightarrow\mathbb{R} $ defined by $$ h(x)=\exp\left(-2(\Delta_i^{\mathrm{v}})^2 \big(\sqrt{x}-\sqrt{v_i(t)} \big)^2\right) $$ is convex and decreasing on $ (4v_i(t),\infty) $ for all $ i\in[N] $. Furthermore, $$ \int_{4v_i(t)}^\infty h(x) \,\mathrm{d}x\leq \frac{\delta}{2(\Delta_{i}^{\mathrm{v}})^2 Nt^4}.$$
\end{lem}
\begin{proof}
	For simplicity, fix any $ i\in[N] $ and let $ a=(\Delta_i^{\mathrm{v}})^2 ,b=\sqrt{v_i(t)}$, then $ h(x)=\exp(-2a(\sqrt{x}-b)^2)  $.
	By simple algebra, when $x>4b^2$
	\begin{align}
		h^\prime(x)&=\frac{-2a(\sqrt{x}-b)}{\sqrt{x}}\exp(-2a(\sqrt{x}-b)^2)<0,\\
		h^{\prime\prime}(x)&=\frac{4a^2\sqrt{x}(\sqrt{x}-b)^2-ab}{x^\frac{3}{2}}\exp(-2a(\sqrt{x}-b)^2).
	\end{align}
	and
	\begin{align} 
		&4a^2\sqrt{x}(\sqrt{x}-b)^2-ab 
		 \geq8a^2b^3-ab
		\\
		&\qquad=ab(8ab^2-1)=ab(8(\Delta_i^{\mathrm{v}})^2 v_i(t)-1)
		\\
		&\qquad\geq ab \left(4\ln  \left(\frac{2N t^{4}}{\delta}\right)-1 \right)>0
	\end{align}
	Thus $h^{\prime\prime}(x) >0 $ on $ (4v_i(t),\infty)  $.
	The integral can be estimated as
	\begin{align*}
		&\int_{4v_i(t)}^\infty h(x)\,\mathrm{d}x\\
		&= \int_{4b^2}^\infty \exp(-2a(\sqrt{x}-b)^2)\,\mathrm{d}x\\
		&=\int_{b}^\infty 2y\exp(-2ay^2)\,\mathrm{d}y+\int_{b}^\infty 2b\exp(-2ay^2)\,\mathrm{d}y \\ &=\frac{1}{2a}\exp(-2ab^2)+2b\int_{b^2}^\infty\frac{1}{2\sqrt{z}}\exp(-2az)\,\mathrm{d}z\\
		&\leq \frac{1}{2a}\exp(-2ab^2)+\int_{b^2}^\infty\exp(-2az)\,\mathrm{d}z\\
		&= \frac{1}{a}\exp(-2ab^2) 
		\leq\frac{\delta}{2(\Delta_{i}^{\mathrm{v}})^2 Nt^4},
	\end{align*}
	as desired.
\end{proof}

\begin{proof}[Proof of Lemma~\ref{Ehappens}]
	Note that  our choice of $ \alpha $ and  $ \beta $ in \eqref{radius} satisfies
	\begin{equation}\label{req}
		\sum_{t=2}^{\infty} \sum_{T=1}^{t-1} \exp \left(-2 T \theta(t,T)^{2}\right) \leq C_\theta\frac{\delta}{N}, \quad \theta=\alpha,\beta
	\end{equation}
	where $ C_\alpha=C_\beta={1}/{8} $.
	Lemma \ref{equ:concentration_inequ}, the definition of $ E $ in \eqref{E}, and the properties of $ \alpha, \beta $ in \eqref{req} imply a lower bound on the probability of $ E $  
	$$ \mathbb{P}[E]\geq 1-2(C_\alpha+C_\beta )\delta=1-\frac{\delta}{2}. $$
	This completes the proof.
\end{proof}

\begin{proof}[Proof of Lemma~\ref{ioutisistar}]
	Conditioned on $ E(\tau) $,
	where $ \tau $ denotes the stopping time step, i.e., $ \bar{\mathcal{F}}_\tau\cap\mathcal{P}_\tau  =\emptyset$. 
	
	When the input instance is infeasible but $ i_\tau\in\mathcal{F}_\tau $ is returned, then $ U_{i_\tau}^{\mathrm{v}}(\tau) \leq \bar{\sigma}^2<\sigma_{i_\tau}^2$ must hold at time step $ \tau $, which contradicts the event $ E_{i_\tau}(\tau) $. Therefore, $ \hat{\mathsf{f}}=\mathsf{f}=0 $. 
	
	When the input instance is feasible, 
	\begin{itemize}
		\item[(i)] If the instance is deemed to be infeasible, there exists arm $ i\in\mathcal{F}$ such that $i\in\bar{\mathcal{F}}_\tau^c$, which violates $ E_i(\tau) $.
	\end{itemize}
	So $\bar{\mathcal{F}}_\tau\neq\emptyset$.
	To ease the proof of the rest cases, we prove $i_\tau=i_\tau^\star$. If $\mathcal{F}_\tau=\emptyset $ and $i_\tau\in\partial\mathcal{F}_\tau$, according to the definition of $\mathcal{P}_t$ in \eqref{p}, $\mathcal{P}_\tau=[N]$. Thus $\bar{\mathcal{F}}_\tau\cap\mathcal{P}_\tau=\bar{\mathcal{F}}_\tau\neq\emptyset$, which contradicts the stopping criterion. Therefore, $\mathcal{F}_\tau\neq\emptyset $ and $i_\tau^\star$ exists. By the definition of $ i_\tau $, 
	we have $U_{i_\tau}^\mu(\tau)>\hat{\mu}_{i_\tau}(\tau)\geq\hat{\mu}_{i_\tau^\star}(\tau)>L_{i_\tau^\star}^\mu(\tau)  $. This indicates $ i_\tau=i_\tau^\star $ or $ i_\tau\in\mathcal{P}_\tau $. If $ i_\tau \in\mathcal{P}_\tau $, we have $\bar{\mathcal{F}}_\tau\cap\mathcal{P}_\tau\supset\{i_\tau\}\neq\emptyset$, which contradicts the stopping criterion.
	Therefore, $ i_\tau = i_\tau^\star\in\mathcal{F}_\tau$.
	\begin{itemize}
		\item[(ii)] If the instance is evaluated as feasible but the returned arm $ i_\tau\in\mathcal{F}_\tau $ is a truly infeasible arm, then it must violate $ E_ {i_\tau}(\tau)$.
		\item[(iii)] If the instance is evaluated as feasible but the returned arm $ i_\tau$ is a truly feasible arm but not $ i^\star $. Conditioned on $ E_{i^\star}(\tau) $, the arm $ i^\star$ belongs to  $\bar{\mathcal{F}}_\tau $. Thus the stopping criterion yields $ L_{i_\tau}^\mu(\tau)>U_{i^\star}^\mu(\tau) $.
		Together with $ \mu_{i_\tau}<\mu_{i^\star} $, we have  $ L_{i_\tau}^\mu(\tau)>\mu_{i_\tau}$ or $ U_{i^\star}^\mu(\tau)<\mu_{i^\star} $. This violates either $ E_{i_\tau}(\tau) $ or $E_{i^\star}(\tau) $.
	\end{itemize}
	Hence, $ i_\mathrm{out}=i_\tau=i_\tau^\star=i^\star,\hat{\mathsf{f}}=\mathsf{f}=1 $. 
\end{proof}

\begin{proof}[Proof of Lemma~\ref{useful}]
	According to the termination condition,
	if the algorithm does not terminate, then $ \bar{\mathcal{F}}_t\cap\mathcal{P}_t\neq\emptyset$.  
	
	To commence our  discussion, we state an obvious case  before proceeding. Note that if  $i_t, c_t$ and $i^\star$ exist with $ i_t,c_t\in\mathcal{S}_t $, conditioned on $ E(t) $, we have $i^\star\in\bar{\mathcal{F}}_t$ and
	\begin{equation}\label{case}
		\bar{\mu}>\max\{U_{i_t}^\mu(t),U_{c_t}^\mu(t)\}\geq U_{i^\star}^\mu(t)>\mu_{i^\star}>\bar{\mu}
	\end{equation}
	which constitutes a contradiction. So we cannot have $ i_t $ and $ c_t $ belong to $\mathcal{S}_t$ at the same time.
	The following discussions will heavily depend on $ E (t)$, which guarantees $ \mathcal{F}_t\subset \mathcal{F} $ and $ \bar{\mathcal{F}}_t^c\subset \bar{\mathcal{F}}^c $.
	
	\noindent
	\underline{\textbf{Case One}}: Only one arm is sampled.
	\begin{enumerate}
		\item $ |\bar{\mathcal{F} }_t|=1$: in this case only arm $i_t=\mathop{\rm argmax}\{\hat{\mu}_i(t)):i\in \bar{\mathcal{F}}_t\} $ is sampled while the rest of the arms are in $ \bar{\mathcal{F}}_t^c \subset 
		\bar{\mathcal{F}}^c$.
		\begin{enumerate}
			\item If $ i_t\in\mathcal{F}_t$, by the definition of $i_t^\star$, $ i_t=i_t^\star $. Therefore we have $ i_t\notin\mathcal{P}_t $ and $\bar{\mathcal{F}}_t =\{i_t\} $, leading to $ \bar{\mathcal{F}}_t\cap\mathcal{P}_t =\emptyset$. This contradicts the assumption that the algorithm does not terminate.
			\item If $ i_t\in\partial\mathcal{F}_t$, since $ |\bar{\mathcal{F} }_t|=1$ and $ i_t^\star $ does not exist, there are two cases:
			\begin{enumerate}
				\item $ i_t\in\mathcal{F}$, i.e., $ i_t $ is a truly feasible arm and $ i_t=i^\star $. According to \eqref{Ei},
				$$\bar{\mu}<\mu_{i_t}< U_{i_t}^\mu(t)$$
				which indicates $ i_t\notin\mathcal{S}_t $. Thus $ i_t\in\partial\mathcal{F}_t\backslash\mathcal{S}_t $.
				\item $ i_t\in\bar{\mathcal{F}}^c$, i.e., $ i_t $ is a truly infeasible arm and the instance infeasible. By the definition of $ \bar{\mu} $, $ \bar{\mu}=-\infty $ which makes $ \mathcal{S}_t =\emptyset$. So  $ i_t\in\partial\mathcal{F}_t=\partial\mathcal{F}_t\backslash\mathcal{S}_t $.
			\end{enumerate}
		\end{enumerate}
		\item $ |\bar{\mathcal{F} }_t|>1$: in this case $ U_{c_t}^\mu < L_{i_t}^\mu  $ and only $ i_t $ is sampled.
		\begin{enumerate}
			\item When $ i_t\in\mathcal{F}_t$, we have $ \mathcal{F}_t\neq\emptyset $ and $ i_t^\star$ exists. We assert that $ i_t=i_t^\star $. Assume that $ i_t $ and $ i_t^\star $ are two different arms, note $ i_t=\mathop{\rm argmax}\{\hat{\mu}_i(t):i\in \bar{\mathcal{F}}_t\} $, so $ \hat{\mu}_{i_t}(t)\geq\hat{\mu}_{i_t^\star}(t) $. And $  i_t\in\mathcal{F}_t  $, $ i_t^\star=\mathop{\rm argmax}\{\hat{\mu}_i(t):i\in\mathcal{F}_t\} $, so $ \hat{\mu}_{i_t}(t)\leq\hat{\mu}_{i_t^\star}(t) $. We have 
			\begin{equation}\label{itisitstar}
				\hat{\mu}_{i_t}(t)=\hat{\mu}_{i_t^\star}(t)
			\end{equation}
			By the definition of $ c_t $, 
			\begin{align} 		
				&U_{c_t}^\mu(t)\geq U_{i_t^\star}^\mu(t) >\hat{\mu}_{i_t^\star}(t)= \hat{\mu}_{i_t}(t)>L_{i_t}^\mu(t)\\
				&\Rightarrow U_{c_t}^\mu(t)>L_{i_t}^\mu(t)
			\end{align}
			which contradicts $ U_{c_t}^\mu < L_{i_t}^\mu  $. Hence, we must have $ i_t=i_t^\star $. In this case, $ U_{c_t}^\mu < L_{i_t}^\mu  $  would indicate 
			$$  U_{i}^\mu\leq U_{c_t}^\mu<L_{i_t}^\mu(t)=L_{i_t^\star}^\mu(t),\quad \forall\, i\in\bar{\mathcal{F}}_t\backslash\{i_t\}$$
			Thus $ \mathcal{P}_t \cap \bar{\mathcal{F}}_t=\emptyset$. This contradicts the assumption that the algorithm does not terminate.
			\item When $ i_t\in\partial\mathcal{F}_t$, we assert that $ i_t\notin\mathcal{S}_t $. If $ i_t\in\mathcal{S}_t $, then 
			$$U_{c_t}^\mu<L_{i_t}^\mu <U_{i_t}^\mu < \bar{\mu} $$
			which indicates $ c_t\in\mathcal{S}_t $. This contradicts \eqref{case}. Thus $ i_t\in\partial\mathcal{F}_t\backslash\mathcal{S}_t $.
		\end{enumerate}
		We conclude that when only arm $ i_t $ is pulled, we have $ i_t\in\partial\mathcal{F}_t\backslash\mathcal{S}_t $.
	\end{enumerate}
	\underline{\textbf{Case Two}}: Both $ i_t $ and $ c_t $ are sampled, i.e., $ |\bar{\mathcal{F} }_t|>1$ and $ U_{c_t}^\mu \geq L_{i_t}^\mu  $.
	\begin{enumerate}
		\item If $ \mathcal{F}_t=\emptyset $, $ \partial\mathcal{F}_t$ cannot be empty, otherwise the algorithm terminates. According to \eqref{case}, at least one of $ i_t$ or $c_t$ locates in $\partial\mathcal{F}_t\backslash\mathcal{S}_t $. 
		\item If $ \mathcal{F}_t\neq\emptyset $, by Lemma \ref{ctinpt}, we have $ c_t\in\mathcal{P}_t \cup\{i_t^\star\}$.
		\begin{enumerate}
			\item $ i_t\in\partial\mathcal{F}_t $:  when $ i_t\notin\mathcal{S}_t $, thus $ i_t\in \partial\mathcal{F}_t \backslash\mathcal{S}_t$. When $ i_t\in\mathcal{S}_t$, according to \eqref{case}, $ c_t\notin\mathcal{S}_t $:
			\begin{enumerate}
				\item if $ c_t\in\partial\mathcal{F}_t $, we have $ c_t \in \partial\mathcal{F}_t\backslash\mathcal{S}_t$.
				\item if $ c_t  \in \mathcal{F}_t$, note that
				\begin{align}
					&U_{c_t}^\mu(t)\geq\bar{\mu}>U_{i_t}^\mu(t)>\hat{\mu}_{i_t}(t)\geq\hat{\mu}_{c_t}(t)>L_{c_t}^\mu(t)\\
					&\Rightarrow U_{c_t}^\mu(t)\geq\bar{\mu}>L_{c_t}^\mu(t)
				\end{align}
				So $ c_t\in\mathcal{N}_t $. This gives $ c_t \in  \mathcal{F}_t\cap \mathcal{N}_t$.
			\end{enumerate}
			\item $ i_t\in\mathcal{F}_t $: by the same reasoning as \eqref{itisitstar}, we obtain
			$ \hat{\mu}_{i_t}(t)=\hat{\mu}_{i_t^\star}(t) $.
			Firstly, if $ i_t $ and $ i_t^\star $ are two different arms, we have
			$ 		U_{c_t}^\mu(t)\geq U_{i_t^\star}^\mu(t) >\hat{\mu}_{i_t^\star}(t)= \hat{\mu}_{i_t}(t)>L_{i_t}^\mu(t)$.
			Secondly, if $ i_t $ and $ i_t^\star $ are the same arm,
			since $c_t\in\mathcal{P}_t\cup\{i_t^\star\}$, we have
			$U_{c_t}^\mu(t)>L_{i_t}^\mu(t)$.
			In either case,
			\begin{equation}\label{case3}
				U_{c_t}^\mu(t)>L_{i_t}^\mu(t)
			\end{equation}
			holds.
			Note $ i_t\in\mathcal{F} $, conditioned on $ E(t) $.
			\begin{enumerate}
				\item When $ c_t \in \mathcal{F}_t\subset\mathcal{F} $, if $ i_t,c_t\notin \mathcal{N}_t$, we have the following:
				\begin{itemize}
					\item $ i_t\in  \mathcal{R}_t,c_t\in \mathcal{R}_t$: since we assumed the optimal arm is unique, then at least one of the two arms has expectation smaller than $ \bar{\mu} $. Denote this arm by $ j $. We have 
					$ \bar{\mu}< L_j^\mu(t)< \mu_j< \bar{\mu}$,
					which is a contradiction.
					\item $  i_t\in  \mathcal{S}_t,c_t\in \mathcal{S}_t$: this contradicts \eqref{case}.
					\item $  i_t\in  \mathcal{R}_t,c_t\in \mathcal{S}_t$: we have
					$U_{c_t}^\mu(t)< \bar{\mu}<L_{i_t}^\mu(t)$.
					This contradicts \eqref{case3}.
					\item $  i_t\in  \mathcal{S}_t,c_t\in \mathcal{R}_t$: we have
					$ \hat{\mu}_{i_t}(t)< U_{i_t}^\mu(t)<\bar{\mu}< L_{c_t}^\mu(t)<\hat{\mu}_{c_t}(t)$.
					This contradicts the definition of~$ i_t $.
				\end{itemize}
				So at least one of $ i_t $ and $ c_t $ lies in $\mathcal{N}_t  $. This gives $ i_t \in  \mathcal{F}_t\cap \mathcal{N}_t$ or $ c_t \in  \mathcal{F}_t\cap \mathcal{N}_t $.
				\item   When $  c_t \in \partial\mathcal{F}_t$, if $ c_t\notin\mathcal{S}_t $, this gives $c_t \in \partial\mathcal{F}_t\backslash\mathcal{S}_t $. If $ c_t\in\mathcal{S}_t $, according to \eqref{case}, $ i_t\notin\mathcal{S}_t $.
				\begin{itemize}
					\item if $ i_t\in \mathcal{R}_t $, we have 
					$U_{c_t}^\mu(t)< \bar{\mu}<L_{i_t}^\mu(t)$.
					This contradicts \eqref{case3}.
					\item if $ i_t\in \mathcal{N}_t $, then $ i_t \in  \mathcal{F}_t\cap \mathcal{N}_t$ .
				\end{itemize}
			\end{enumerate}
		\end{enumerate}
	\end{enumerate}
	In conclusion, we have $$ i_t \in  (\partial\mathcal{F}_t\backslash\mathcal{S}_t)\cup(\mathcal{F}_t\cap \mathcal{N}_t)\, \quad\mbox{or}\,\quad c_t \in  (\partial\mathcal{F}_t\backslash\mathcal{S}_t)\cup(\mathcal{F}_t\cap \mathcal{N}_t)$$
	as desired.
\end{proof}

\begin{proof}[Proof of Lemma~\ref{sufcon}]
	The techniques used in the analysis of LUCB \cite{Kalyanakrishnan2012} are adapted in this proof.
	For a suboptimal arm $ i $, note that $ \frac{\Delta_i}{2}\leq\bar{\mu}-\mu_i \leq \Delta_i$.
	\begin{align*}\label{equ1}
		&\mathbb{P}[T_i(t)>16u_i(t),i\in \mathcal{N}_t]\\
		&\leq\mathbb{P}[T_i(t)>16u_i(t), i\notin \mathcal{S}_t]\\
		&=\mathbb{P}[T_i(t)>16u_i(t), U_i^\mu(t)\geq\bar{\mu}]\\
		&\leq\sum_{T=16u_i(t)+1}^\infty \mathbb{P}[T_i(t)=T,\hat{\mu}_i(t)-\mu_i\geq\bar{\mu}-\alpha(t,T)-\mu_i]\\
		&\leq\sum_{T=16u_i(t)+1}^\infty \exp\left(-2T\left(\bar{\mu}-\alpha(t,T)-\mu_i\right)^2\right)\\
		&\leq \sum_{T=16u_i(t)+1}^\infty \exp\left(-2T\left(\frac{\Delta_i}{2}-\sqrt{\frac{1}{2 T} \ln \left(\frac{2N t^{4}}{\delta}\right)}\right)^2\right)\\
		&\leq\sum_{T=16u_i(t)+1}^\infty \exp\left(-2\Delta_{i}^2\left(\frac{\sqrt{T}}{2}-\sqrt{u_i(t)}\right)^2\right)\\
		&\leq\int_{16u_i(t)}^\infty \exp\left(-2\Delta_{i}^2\left(\frac{\sqrt{x}}{2}-\sqrt{u_i(t)}\right)^2\right)\,\mathrm{d}x\\
		&=\int_{4u_i(t)}^\infty 4\exp\left(-2\Delta_{i}^2\left(\sqrt{x}-\sqrt{u_i(t)}\right)^2\right)\,\mathrm{d}x\\
		&\leq\frac{\delta}{2(\frac{\Delta_{i}}{2})^2 Nt^4}
	\end{align*}
	The third inequality results from Hoeffding's inequality. The summation can be upper bounded by the integral is due to the fact that the integrand is convex and decreasing within the range of integration, which can be derived by using similar techniques as in Lemma \ref{integral}.
	Similarly, for $ i=i^\star $, we have
	\begin{align}
		\mathbb{P}[T_{i^\star}(t)>16u_{i^\star}(t),{i^\star}\in \mathcal{N}_t]&\leq\mathbb{P}[T_{i^\star}(t)>16u_{i^\star}(t), {i^\star}\notin \mathcal{R}_t]\nonumber\\*
		&\leq\frac{\delta}{2(\frac{\Delta_{{i^\star}}}{2})^2Nt^4}
	\end{align}
	For any arm $ i  \in\mathcal{F}$
	\begin{align*}
		&\mathbb{P}[T_i(t)>4v_i(t),i\in \partial\mathcal{F}_t]\\
		&\leq\mathbb{P}[T_i(t)>4v_i(t),i\notin\mathcal{F}_t]\\
		&=\mathbb{P}[T_i(t)>4v_i(t),U_i^{\mathrm{v}}(t)>\bar{\sigma}^2]\\
		&\leq\sum_{T=4v_i(t)+1}^\infty \mathbb{P}[T_i(t)=T,\hat{\sigma}_i^2(t)-\sigma_i^2>\bar{\sigma}^2-\sigma_i^2-\beta(t,T)]\\
		&\leq\sum_{T=4v_i(t)+1}^\infty \exp\left(-2T\left(\bar{\sigma}^2-\sigma_i^2-\beta(t,T)\right)^2\right)\\
		&\leq\sum_{T=4v_i(t)+1}^\infty \exp\left(-2T\left(\Delta_{i}^{\mathrm{v}}-\sqrt{\frac{1}{2 T} \ln \left(\frac{2 N 	t^{4}}{\delta}\right)}\right)^2\right)\\
		&\leq\sum_{T=4v_i(t)+1}^\infty \exp\left(-2(\Delta_{i}^{\mathrm{v}})^2\left(\sqrt{T}-\sqrt{v_i(t)}\right)^2\right)\\
		&\leq\int_{T=4v_i(t)}^\infty \exp\left(-2(\Delta_{i}^{\mathrm{v}})^2\left(\sqrt{x}-\sqrt{v_i(t)}\right)^2\right)\,\mathrm{d}x\\
		&\leq\frac{\delta}{2(\Delta_{i}^{\mathrm{v}})^2 Nt^4}.
	\end{align*}
	The third inequality utilizes McDiarmid's inequality. The last two steps are due to Lemma \ref{integral}.
	The same holds for $ i\in\bar{\mathcal{F}}^c $:
	\begin{equation*}
		\begin{aligned}
			\mathbb{P}[T_i(t)>4v_i(t),i\in \partial\mathcal{F}_t]&\leq\mathbb{P}[T_i(t)>4v_i(t),i\notin\bar{\mathcal{F}}_t^c]\\
			&\leq\frac{\delta}{2(\Delta_{i}^{\mathrm{v}})^2 Nt^4}.
		\end{aligned}
	\end{equation*}
	This completes the proof.
\end{proof}

\begin{proof}[Proof of Lemma~\ref{stop}]
	The proof improves the techniques used in the analysis of the original LUCB algorithm \cite{Kalyanakrishnan2012} in order to analyze the effect of the empirical variances on the sample complexity.
	
	Let $t$ be a sufficiently large integer and $ \mathcal{T}:=\{\lceil t/2\rceil,\ldots, t-1\} $. Define events
	\begin{itemize}
		\item $\mathcal{I}_1$: $  \exists\, s\in\mathcal{T}$ such that $i^\star\notin\mathcal{F}_s\cap\mathcal{R}_s$ and $ T_{i^\star}(s)>\max\{16u_{i^\star}(s),4v_{i^\star}(s)\}$.
		\item $\mathcal{I}_2$: $ \exists\, i\in\mathcal{F}\cap\mathcal{S}, s\in\mathcal{T}$ such that $ i\notin\mathcal{S}_s$ and $ T_i(s)>16u_i(s)$.
		\item $\mathcal{I}_3$: $ \exists\, i\in\bar{\mathcal{F}}^c\cap\mathcal{R}, s\in\mathcal{T}$ such that $i\notin \bar{\mathcal{F}}_s^c$ and $T_i(s)>4v_i(s)$
		\item $\mathcal{I}_4$: $ \exists\, i\in\bar{\mathcal{F}}^c\cap\mathcal{S}, s\in\mathcal{T}$ such that $i\notin \bar{\mathcal{F}}_s^c\cup\mathcal{S}_s$ and $ T_i(s)>\min\{16u_i(s),4v_i(s)\}$.
	    \item $\mathcal{I}_5$: $  \exists\, s\in\mathcal{T}$ such that $E(s)$ does not occur.
	\end{itemize} 
	If VA-LUCB terminates before $  \lceil{t}/{2}\rceil $, the statement is definitely right. If not, we assume the above 
    five
	events do not occur. Based on Lemma \ref{useful}, the additional time steps after $ \lceil{t}/{2}\rceil -1$ for VA-LUCB can be upper bounded as \eqref{estimation1} in \underline{Derivation $1$} on the next page.
	\begin{figure*}[htbp]
		\underline{Derivation $1$}:
		\begin{align}
		&\sum_{s\in\mathcal{T}}\mathbbm{1}\{ i_s\in(\mathcal{F}_s\cap\mathcal{N}_s)\cup (\partial\mathcal{F}_s\backslash\mathcal{S}_s)\,\mbox{or}\, c_s\in(\mathcal{F}_s\cap\mathcal{N}_s)\cup (\partial\mathcal{F}_s\backslash\mathcal{S}_s)\}
		\\
		&\leq\sum_{s\in\mathcal{T}}\sum_{i\in[N]}\mathbbm{1}\{i=i_s\,\mbox{or}\, c_s,\, i\in(\mathcal{F}_s\cap\mathcal{N}_s)\cup (\partial\mathcal{F}_s\backslash\mathcal{S}_s)\}\\
		&=\sum_{s\in\mathcal{T}}\left(\mathbbm{1}\{i^\star=i_s\,\mbox{or}\, c_s,\, i^\star\in(\mathcal{F}_s\cap\mathcal{N}_s)\cup (\partial\mathcal{F}_s\backslash\mathcal{S}_s)\}
		+\sum_{i\in\mathcal{F}\cap\mathcal{S}}\mathbbm{1}\{i=i_s\,\mbox{or}\, c_s,\, i\in(\mathcal{F}_s\cap\mathcal{N}_s)\cup (\partial\mathcal{F}_s\backslash\mathcal{S}_s)\}\right.\\*
		&\quad \left.+\sum_{i\in\bar{\mathcal{F}}^c\cap\mathcal{R}}\mathbbm{1}\{i=i_s\,\mbox{or}\, c_s,\, i\in(\mathcal{F}_s\cap\mathcal{N}_s)\cup (\partial\mathcal{F}_s\backslash\mathcal{S}_s)\}
		+\sum_{i\in\bar{\mathcal{F}}^c\cap\mathcal{S}}\mathbbm{1}\{i=i_s\,\mbox{or}\, c_s,\, i\in(\mathcal{F}_s\cap\mathcal{N}_s)\cup (\partial\mathcal{F}_s\backslash\mathcal{S}_s)\}\right)\\
		&\leq\sum_{s\in\mathcal{T}}\left(\mathbbm{1}\{i^\star=i_s\,\mbox{or}\, c_s,\, i^\star\notin\mathcal{F}_s\cap\mathcal{R}_s\}
		+\sum_{i\in\mathcal{F}\cap\mathcal{S}}\mathbbm{1}\{i=i_s\,\mbox{or}\, c_s,\, i\notin\mathcal{S}_s\}\right.\\*
		&\quad \left.+\sum_{i\in\bar{\mathcal{F}}^c\cap\mathcal{R}}\mathbbm{1}\{i=i_s\,\mbox{or}\, c_s,\, i\notin\bar{\mathcal{F}}_s^c\}
		+\sum_{i\in\bar{\mathcal{F}}^c\cap\mathcal{S}}\mathbbm{1}\{i=i_s\,\mbox{or}\, c_s,\, i\notin\bar{\mathcal{F}}_s^c\cup\mathcal{S}_s\}\right)\\
		&\leq\sum_{s\in\mathcal{T}}\left(\mathbbm{1}\{i^\star=i_s\,\mbox{or}\, c_s,\, T_{i^\star}(s)\leq\max\{16u_{i^\star}(s),4v_{i^\star}(s)\}\}
		+\sum_{i\in\mathcal{F}\cap\mathcal{S}}\mathbbm{1}\{i=i_s\,\mbox{or}\, c_s,\, T_i(s)\leq16u_i(s)\}\right.\\
		&\quad \left.+\sum_{i\in\bar{\mathcal{F}}^c\cap\mathcal{R}}\mathbbm{1}\{i=i_s\,\mbox{or}\, c_s,\, T_i(s)\leq4v_i(s)\}
		+\sum_{i\in\bar{\mathcal{F}}^c\cap\mathcal{S}}\mathbbm{1}\{i=i_s\,\mbox{or}\, c_s,\, T_i(s)\leq\min\{16u_i(s),4v_i(s)\}\}\right)\\
		&\leq\sum_{s\in\mathcal{T}}\mathbbm{1}\{i^\star=i_s\,\mbox{or}\, c_s,\, T_{i^\star}(s)\leq\max\{16u_{i^\star}(s),4v_{i^\star}(s)\}\}
		+\sum_{i\in\mathcal{F}\cap\mathcal{S}}\sum_{s\in\mathcal{T}}\mathbbm{1}\{i=i_s\,\mbox{or}\, c_s,\, T_i(s)\leq16u_i(s)\}\\
		&\quad +\sum_{i\in\bar{\mathcal{F}}^c\cap\mathcal{R}}\sum_{s\in\mathcal{T}}\mathbbm{1}\{i=i_s\,\mbox{or}\, c_s,\, T_i(s)\leq4v_i(s)\}
		+\sum_{i\in\bar{\mathcal{F}}^c\cap\mathcal{S}}\sum_{s\in\mathcal{T}}\mathbbm{1}\{i=i_s\,\mbox{or}\, c_s,\, T_i(s)\leq\min\{16u_i(s),4v_i(s)\}\}\\
		&\leq\max\{16u_{i^\star}(t),4v_{i^\star}(t)\}
		+\sum_{i\in\mathcal{F}\cap\mathcal{S}}16u_i(t)+\sum_{i\in\bar{\mathcal{F}}^c\cap\mathcal{R}} 4v_i(t)
		+\sum_{i\in\bar{\mathcal{F}}^c\cap\mathcal{S}}\min\{16u_i(t),4v_i(t)\}.\label{estimation1}
		\end{align}
		\hrulefill
	\end{figure*}
	On the other hand, we observe that if the time step $ t=CH_\mathrm{VA}\ln({H_\mathrm{VA}}/{\delta})$ with $C\ge 152$, we have \eqref{estimation2} as in \underline{Derivation $2$} on the next page.
	\begin{figure*}[htbp]		
		\underline{Derivation $2$}:	
		\begin{align}
		&\max\{16u_{i^\star}(t),4v_{i^\star}(t)\}
		+\sum_{i\in\mathcal{F}\cap\mathcal{S}}16u_i(t)+\sum_{i\in\bar{\mathcal{F}}^c\cap\mathcal{R}} 4v_i(t)
		+\sum_{i\in\bar{\mathcal{F}}^c\cap\mathcal{S}}\min\{16u_i(t),4v_i(t)\}\\
		&=\max\left\{16\left\lceil{\frac{1}{2 \Delta_{i^\star}^2} \ln \left(\frac{2N t^{4}}{\delta}\right)}\right\rceil,4\left\lceil{\frac{1}{2 (\Delta_{i^\star}^{\mathrm{v}})^2} \ln \left(\frac{2 N t^{4}}{\delta}\right)}\right\rceil\right\}+
		16\sum_{i\in\mathcal{F}\cap\mathcal{S}}\left\lceil{\frac{1}{2 \Delta_i^2} \ln \left(\frac{2 N t^{4}}{\delta}\right)}\right\rceil\\
		&\quad +4\sum_{i\in\bar{\mathcal{F}}^c\cap\mathcal{R}}\left\lceil{\frac{1}{2 (\Delta_i^{\mathrm{v}})^2} \ln \left(\frac{2 N t^{4}}{\delta}\right)}\right\rceil
		+\sum_{i\in\bar{\mathcal{F}}^c\cap\mathcal{S}}\min\left\{16\left\lceil{\frac{1}{2 \Delta_{i}^2} \ln \left(\frac{2 N t^{4}}{\delta}\right)}\right\rceil,4\left\lceil{\frac{1}{2 (\Delta_{i}^{\mathrm{v}})^2} \ln \left(\frac{2 N t^{4}}{\delta}\right)}\right\rceil\right\}\\
		&\leq 16N+\max\left\{16{\frac{1}{2 \Delta_{i^\star}^2} \ln \left(\frac{2 N t^{4}}{\delta}\right)},4{\frac{1}{2 (\Delta_{i^\star}^{\mathrm{v}})^2} \ln \left(\frac{2 N t^{4}}{\delta}\right)}\right\}+
		16\sum_{i\in\mathcal{F}\cap\mathcal{S}}{\frac{1}{2 \Delta_i^2} \ln \left(\frac{2 N t^{4}}{\delta}\right)}\\*
		&\quad +4\sum_{i\in\bar{\mathcal{F}}^c\cap\mathcal{R}}{\frac{1}{2 (\Delta_i^{\mathrm{v}})^2} \ln \left(\frac{2 N t^{4}}{\delta}\right)}
		+\sum_{i\in\bar{\mathcal{F}}^c\cap\mathcal{S}}\min\left\{16{\frac{1}{2 \Delta_{i}^2} \ln \left(\frac{2 N t^{4}}{\delta}\right)},4{\frac{1}{2 (\Delta_{i}^{\mathrm{v}})^2} \ln \left(\frac{2 N t^{4}}{\delta}\right)}\right\}\\
		&= 16N+2H_\mathrm{VA}\ln \left(\frac{2 N t^{4}}{\delta}\right)\\
		&= (16N+2H_\mathrm{VA}\ln 2)+2H_\mathrm{VA}\ln \frac{N}{\delta}+8H_\mathrm{VA}\ln \left(CH_\mathrm{VA}\ln\frac{H_\mathrm{VA}}{\delta}\right)\\
		&\leq (16N+2H_\mathrm{VA}\ln 2+8H_\mathrm{VA}\ln C)+2H_\mathrm{VA}\ln \frac{N}{\delta}+16H_\mathrm{VA}\ln\frac{H_\mathrm{VA}}{\delta}\\
		&\leq  \frac{1}{2}CH_\mathrm{VA}\ln\frac{H_\mathrm{VA}}{\delta}.\label{estimation2}
	\end{align}
	\hrulefill
\end{figure*}
	So the total number of time steps is bounded by 
	\begin{align}
		&\left\lceil\frac{t}{2}\right\rceil-1+\max\{16u_{i^\star}(t),4v_{i^\star}(t)\}
		+\sum_{i\in\mathcal{F}\cap\mathcal{S}}16u_i(t)\\
		&\;+\sum_{i\in\bar{\mathcal{F}}^c\cap\mathcal{R}} 4v_i(t)
		+\sum_{i\in\bar{\mathcal{F}}^c\cap\mathcal{S}}\min\{16u_i(t),4v_i(t)\}\leq t.
	\end{align}
	We then compute the probability of the event $\underset{{i\in[5]}}{\cup}\mathcal{I}_i$
	.	To simplify  notations used in the following, we define the hardness quantity
	\begin{align}
		\widetilde{H}:=&\frac{1}{(\Delta_{i^\star}^{\mathrm{v}})^2 } +\sum_{i\in\mathcal{F}}\frac{1}{(\frac{\Delta_i}{2})^2 }\\
		&\;
		+\sum_{i\in\bar{\mathcal{F}}^c\cap\mathcal{R}}\frac{1}{(\Delta_{i}^{\mathrm{v}})^2} +\sum_{i\in\bar{\mathcal{F}}^c\cap\mathcal{S}}\frac{1}{f(u_i(t),v_i(t))},
	\end{align}
	where $f(u_i(t),v_i(t))=(\frac{\Delta_i}{2})^2\cdot\mathbbm{1}\{16u_i(t)<4v_i(t)\}+(\Delta_i^{\mathrm{v}})^2\cdot\mathbbm{1}\{16u_i(t)\geq4v_i(t)\}
	$ which is exactly $\max\{\frac{\Delta_i}{2},\Delta_i^{\mathrm{v}}\}^2$ if we ignore the ceiling operators in $u_i(t)$ and $v_i(t)$. 
	By Lemma \ref{equ:concentration_inequ} and Lemma \ref{sufcon},
	\begin{equation*}
		\begin{aligned}
			&\mathbb{P}[\mathcal{I}_1]\leq \sum_{s\in\mathcal{T}} \bigg(\frac{\delta}{2(\frac{\Delta_{i^\star}}{2})^2Ns^4}+\frac{\delta}{2(\Delta_{i^\star}^{\mathrm{v}})^2 Ns^4}\bigg),\\
			&\mathbb{P}[\mathcal{I}_2]\leq \sum_{i\in\mathcal{F}\cap\mathcal{S}}\sum_{s\in\mathcal{T}} \frac{\delta}{2(\frac{\Delta_{i}}{2})^2 Ns^4},\\
			&\mathbb{P}[\mathcal{I}_3]\leq \sum_{i\in\bar{\mathcal{F}}^c\cap\mathcal{R}}\sum_{s\in\mathcal{T}} \frac{\delta}{2(\Delta_{i}^{\mathrm{v}})^2 Ns^4},\\
			&\mathbb{P}[\mathcal{I}_4]\leq \sum_{i\in\bar{\mathcal{F}}^c\cap\mathcal{S}}\sum_{s\in\mathcal{T}} \frac{\delta}{2f(u_i(t),v_i(t))Ns^4},\\
			&\mathbb{P}[\mathcal{I}_5]\leq \sum_{s\in\mathcal{T}} \frac{2\delta}{s^3}\leq \frac{3\delta}{t^2},\\
		\end{aligned}
	\end{equation*}
	which implies that 
    \begin{equation*}
		\mathbb{P} \Bigg[\underset{{i\in[5]}}{\bigcup}\mathcal{I}_i \Bigg]\leq \widetilde{H} \sum_{s\in\mathcal{T}}\frac{\delta}{2Ns^4}+\frac{3\delta}{t^2}
		\leq \frac{5\delta}{t^2},
	\end{equation*}
	where the last inequality utilizes the fact that $\widetilde{H}<2H_\mathrm{VA}<t$ and $N\geq 2$.
	This yields the upper bound of the probability that the algorithm does not terminate at time step $ t $. 
\end{proof}

\begin{proof}[Proof of Theorem~\ref{thm:up_bd}]
	By Lemmas~\ref{Ehappens} and \ref{ioutisistar}, if it terminates, Algorithm~\ref{VALUCB} succeeds on event $E$, which occurs with probability at least $1-{\delta}/{2}$. According to Lemma~\ref{stop}, Algorithm~\ref{VALUCB} terminates at time 
    $t>t^\star>152>5$ 
	with probability at least ${5\delta}/{t^2}$. So Algorithm~\ref{VALUCB} succeeds after $O(H_\mathrm{VA}\ln({H_\mathrm{VA}}/{\delta}) )$ time steps with probability at least $1-({\delta}/{2}+{5\delta}/{5^2})\geq1-\delta$.\footnote{The reader may  notice that our  estimate of $t$ is rather coarse  here. When $t$ is large enough, the probability that the algorithm does not stop is negligible.}
	Note that the sample complexity is at most twice of number of time steps. This completes the proof of Theorem \ref{thm:up_bd}.
\end{proof}

\section{The Sub-Gaussian Case}\label{sec:subgaussian}
In this section, we extend the utility and analysis of VA-LUCB to the case in which the rewards are sub-Gaussian. We see that the main difficulty lies in the fact that the empirical variance is sub-Exponential and its concentration bound (see Lemma~\ref{lem:cv_sub}) is not as convenient as that for the bounded rewards case (in Lemma~\ref{equ:concentration_inequ}). Thus the main change of VA-LUCB is the inclusion of a warm-up phase in which we pull each arm a fixed number of times and a forced-sampling procedure in the following time steps. We specify precisely in the following how many we need to pull each arm in the initial forced exploration phase and in the forced-sampling procedure.

Recall (see, for example, Duchi~\cite[Chapter~3]{duchi}) that a random variable $X$ is  {\em $\sigma$-sub-Gaussian}  (or {\em sub-Gaussian} with variance proxy $\sigma^2$) if for all $s\in\mathbb{R}$, $$\ln\mathbb{E}[\exp( s (X-\mathbb{E}X ))]\le \frac{s^2\sigma^2}{2}.$$ Additionally, $Y$ is {\em sub-Exponential} with parameters $(\tau^2,b)$ (also written as $Y\sim\mathrm{SE}(\tau^2,b)$) if for all $s \in\mathbb{R}$ such that $|s|\le 1/b$, $$\ln\mathbb{E}[\exp(s( Y-\mathbb{E}Y ))]\le \frac{s^2\tau^2 }{2}.$$

For brevity, let $c$ denote the absolute constant $64$ from now on. Given an instance $ (\nu,\bar{\sigma}^2) $, where $\nu_i,i\in[N]$ are independent $\sigma$-sub-Gaussian distributions, we define the {\em hardness parameter} for this $\sigma$-sub-Gaussian instance as
\begin{align}
	&H_{\mathrm{VA},N}^{(\sigma)}:=\max\{H_\mathrm{VA}^{(\sigma)},N\},\quad\mbox{where}\\
	&H_\mathrm{VA}^{(\sigma)}:=\max\left\{\frac{2\sigma^2}{(\frac{\Delta_{i^\star}}{2})^2} ,\frac{2c\sigma^4}{(\Delta_{i^\star}^{\mathrm{v}})^2} \right\}+
	\sum_{i\in\mathcal{F}\cap\mathcal{S}}{\frac{2\sigma^2}{(\frac{\Delta_{i}}{2})^2}}
	\\
	&\qquad+\sum_{i\in\bar{\mathcal{F}}^c\cap\mathcal{R}}{\frac{2c\sigma^4}{(\Delta_{i}^{\mathrm{v}})^2}}
	+\sum_{i\in\bar{\mathcal{F}}^c\cap\mathcal{S}}\min\left\{{\frac{2\sigma^2}{(\frac{\Delta_{i}}{2})^2} },{\frac{2c\sigma^4}{(\Delta_{i}^{\mathrm{v}})^2} }\right\}.
\end{align}
Algorithm~\ref{VALUCB_subgaussian}, designed for sub-Gaussian random rewards, is a slight extension of VA-LUCB (Algorithm~\ref{VALUCB}) and has the following guarantee.
\begin{restatable}[Upper Bound for $\sigma$-sub-Gaussian Case]{thm}{thmuppboundsub}\label{thm:up_bd_sub}
	Given an instance $ (\nu,\bar{\sigma}^2) $ with $\bar{\sigma}^2<\sigma^2$ and confidence parameter $ \delta \leq 0.05$, 
	with probability at least $1-\delta$, 
	Algorithm~\ref{VALUCB_subgaussian} succeeds and terminates in
	\begin{equation}
		O\left( H_{\mathrm{VA},N}^{(\sigma)}\ln\frac{H_{\mathrm{VA},N}^{(\sigma)}}{\delta}\right) \label{eqn:valucb_subG}
	\end{equation} 
	time steps. Furthermore, the expected sample complexity is as in \eqref{eqn:valucb_subG}. 
\end{restatable}
\begin{remark}
	When the threshold $\bar{\sigma}^2\geq\sigma^2$, all the arms are feasible and the problem reduce to vanilla BAI problem. The more interesting case is the case where $\bar{\sigma}^2<\sigma^2$ and the expectations of the arms are close, e.g., $\Delta_i\leq 2\sigma$ for all $i\in [N]$. In this case, $\Delta_i^\mathrm{v}<\sigma^2,\Delta_i\leq 2\sqrt{2}\sigma $ for all $i\in[N]$, leading to $H_{\mathrm{VA},N}^{(\sigma)}=H_\mathrm{VA}^{(\sigma)}$. Furthermore,
	\begin{equation}
		\min\big\{2\sigma^2,2c\sigma^4\big\}H_\mathrm{VA}\leq H_\mathrm{VA}^{(\sigma)}\leq\max\big\{2\sigma^2,2c\sigma^4\big\}H_\mathrm{VA},
	\end{equation}
	where $H_\mathrm{VA}$ is defined in \eqref{eqn:hardness}. These bounds  imply that $H_\mathrm{VA}^{(\sigma)}$ essentially captures the intrinsic hardness of the instance and is related linearly to the hardness parameter for the bounded rewards case $H_{\mathrm{VA}}$.
\end{remark}

\subsection{VA-LUCB for  the $\sigma$-sub-Gaussian Case}
We extend VA-LUCB to $\sigma$-sub-Gaussian distributions. 
The modified algorithm based on Algorithm~\ref{VALUCB} is stated in Algorithm~\ref{VALUCB_subgaussian}.
\begin{algorithm}[t]
	\caption{Variance-Aware LUCB for $\sigma$-sub-Gaussian Distributions (VA-LUCB-$\sigma$-sub-Gaussian)}
	\begin{algorithmic}[1]
		\STATE \textbf{Input}: threshold $\bar{\sigma}^2  > 0$, sub-Gaussian parameter $\sigma$, and confidence parameter $ \delta\in  (0,1)$. 
		\STATE Sample each of the $ N $ arms $T_0$ (see \eqref{equ:T0}) times and set $\bar{\mathcal{F}}_{T_0}=[N]$.
		\FOR{ time step $ t=T_0+1, T_0+2\ldots $}
		\STATE Compute the sample mean and sample variance using \eqref{sm} and \eqref{sv} for $i\in\bar{\mathcal{F}}_{t-1}$.
		\STATE Update the confidence bounds for the mean and variance by \eqref{cb} and \eqref{cbv} for $i\in\bar{\mathcal{F}}_{t-1}$.
		\STATE Update $\mathcal{F}_t:=\{i: U_i^{\mathrm{v}}(t)\leq \bar{\sigma}^2\}.$ and $ \bar{\mathcal{F}}_t:=\{i: L_i^{\mathrm{v}}(t)\leq \bar{\sigma}^2\}.$ 
		\STATE Find $ 	i_t^\star:=\mathop{\rm argmax}\{\hat{\mu}_i(t):i\in\mathcal{F}_t\}$ if $ \mathcal{F}_t\neq\emptyset $.
		\STATE Update $ \mathcal{P}_t:=\{i:L_{i_t^\star}^\mu(t)\leq U_i^\mu(t),i\neq i_t^\star\}$ if   $\{\mathcal{F}_t\neq \emptyset\}$, otherwise $\mathcal{P}_t:=[N]$.
		\IF{$\bar{\mathcal{F}}_t\cap\mathcal{P}_t=\emptyset$}
		\IF {$ \mathcal{F}_t\neq\emptyset $} 
		\STATE Set $ i_\mathrm{out}=i_{t}=\mathop{\rm argmax}\big\{\hat{\mu}_i(t):i\in \bar{\mathcal{F}}_t\big\} $ and $ \hat{\mathsf{f}}=1 $. 
		\ELSE 
		\STATE Set $ \hat{\mathsf{f}}=0 $. 
		\ENDIF
		\BREAK
		\ENDIF
		\IF{$ |\bar{\mathcal{F} }_t|=1$}
		\STATE Sample arm  $i_{t}=\mathop{\rm argmax}\big\{\hat{\mu}_i(t):i\in \bar{\mathcal{F}}_t\big\}.$ (in one round).
		\ELSE
		\STATE Find $i_{t}=\mathop{\rm argmax}\big\{\hat{\mu}_i(t):i\in \bar{\mathcal{F}}_t\big\}$ and competitor arm $c_{t}=\mathop{\rm argmax}\{U_i^\mu(t):i\in \bar{\mathcal{F}}_t,i\neq i_{t}\}.$
		\IF{$ U_{c_{t}}^\mu(t) \geq L_{i_{t}}^\mu(t) $}
		\STATE Sample arms $ i_{t} $ and $ c_{t} $ (in two rounds)
		\ELSE 
		\STATE Sample arm $ i_{t} $ (in one round). 
		\ENDIF
		\ENDIF 
		\STATE Find $\mathcal{M}_t=\{i\in\bar{\mathcal{F}_{t}}:\beta(t+1,T_i(t))>\sigma^2,i$ has not been sampled at this time step$\}$. 
		\STATE Sample each arm in $ \mathcal{M}_t$ once (in $ | \mathcal{M}_t|$ rounds).
		\ENDFOR
	\end{algorithmic}\label{VALUCB_subgaussian}
\end{algorithm}
The notations from VA-LUCB (Algorithm~\ref{VALUCB}) can be directly adapted to  the $\sigma$-sub-Gaussian case, except that two notations need to be modified slightly.
\begin{itemize}
	\item Let $\mathcal{J}_t$ denote the set of arms pulled in time step $t$. Note that there can be {\em more than $2$} arms being sampled in one time step.
	\item The confidence radii for the mean and variance are re-defined to be
	\begin{align}\label{equ:confidence_radius_subgaussian}
		&\alpha(t,T)=\sqrt{\frac{2\sigma^2}{T}\ln\frac{kNt^4}{\delta}},\quad		\mbox{and}\\
		&\beta(t,T)=\sqrt{\frac{2c\sigma^4}{T}\ln\frac{kNt^4}{\delta}};
	\end{align}
	respectively, where $k > 0$ is an absolute constant to be determined.
\end{itemize}
Before the analysis of Algorithm~\ref{VALUCB_subgaussian}, we present a convenient concentration bound for the sample variance.
\begin{lem}\label{lem:cv_sub}
	For an i.i.d.\ $\sigma$-sub-Gaussian random variables $ X_1,\ldots,X_n $ with expectation  $ \mu $ and variance $\var(X)$, let $ \bar{X}=\frac{1}{n}\sum_{i=1}^n X_i$ denote the sample mean and $ S_n^2=\frac{1}{n-1}\sum_{i=1}^n(X_i-\bar{X})^2 $ denote the (unbiased) sample variance. For any integer $ n\geq 12 $ and $\epsilon>0$, we have 
	\begin{equation}
		\begin{aligned}
		 \mathbb{P}[S_n^2-\var(X)\geq\epsilon]&\leq \exp\left(-\frac{1}{16}\min\left\{\frac{n\epsilon^2}{8\sigma^4},\frac{n\epsilon}{\sigma^2}\right\}\right), \\  
			\mathbb{P}[S_n^2-\var(X)\leq-\epsilon]&\leq \exp\left(-\frac{1}{16}\min\left\{\frac{n\epsilon^2}{8\sigma^4},\frac{n\epsilon}{\sigma^2}\right\}\right). \label{equ:subgaussian_variance}
		\end{aligned}
	\end{equation}
\end{lem}
\begin{proof}
	We prove the former inequality here; the latter can be derived analogously.
	According to Honorio and Jaakkola~\cite[Appendix~B]{honorio2014}, for any $\sigma$-sub-Gaussian random variable $X$,
	\begin{equation}\label{Honorio}
		\mathbb{E}\left[\exp\left(t(X^2-\mathbb{E}[X^2])\right)\right]\leq\exp\left(16t^2\sigma^4\right),\quad\forall\, |t|\leq \frac{1}{4\sigma^2},
	\end{equation}
	which indicates that $X^2-\mathbb{E}[X^2] $ is sub-Exponential. More precisely,  $X^2-\mathbb{E}[X^2]\sim \mathrm{SE}(32\sigma^4,4\sigma^2) $.
	
	The sample variance can be reorganized as 
	\begin{align} 
		S_n^2&=\frac{1}{n-1}\sum_{i=1}^n(X_i-\bar{X})^2 =\frac{1}{n-1}\sum_{i=1}^n X_i^2-\frac{n}{n-1}\bar{X}^2
		\\
		&=\frac{1}{n-1}\sum_{i=1}^n (X_i-\mu)^2-\frac{n}{n-1}(\bar{X}-\mu)^2.
	\end{align} 
	By the properties of sub-Gaussian and sub-Exponential random variables (see Duchi~\cite[Chapter~3]{duchi}), 
	\begin{align}
		&X_i-\mu\sim\sigma\mbox{-sub-Gaussian}\\
		\Longrightarrow\qquad &\frac{X_i-\mu}{\sqrt{n-1}}\sim\frac{\sigma}{\sqrt{n-1}}\mbox{-sub-Gaussian}\\
		\Longrightarrow \qquad&\frac{(X_i-\mu)^2}{n-1}\sim \mathrm{SE} \left(32\frac{\sigma^4}{(n-1)^2},4\frac{\sigma^2}{n-1}\right)\\
		\Longrightarrow \qquad &S_n^2\sim \mathrm{SE}\left(32\frac{n\sigma^4}{(n-1)^2},4\frac{\sigma^2}{n-1}\right)
	\end{align} 
	where the last implication utilizes the independence of $ (X_i-\mu)$ across $i\in[n] $. Likewise, 
	\begin{align}
		&\sum_{i=1}^n X_i-n\mu\sim\sqrt{n}\sigma\mbox{-sub-Gaussian}\\
		\Longrightarrow \qquad &\bar{X}-\mu\sim\frac{\sigma}{\sqrt{n}}\mbox{-sub-Gaussian}\\
		\Longrightarrow \qquad &\sqrt{\frac{n}{n-1}}(\bar{X}-\mu)\sim\frac{\sigma}{\sqrt{n-1}}\mbox{-sub-Gaussian}\\
		\Longrightarrow \qquad &\frac{n}{n-1}(\bar{X}-\mu)^2\sim \mathrm{SE}\left(32\frac{\sigma^4}{(n-1)^2},4\frac{\sigma^2}{n-1}\right)
	\end{align}
	Therefore, $S_n^2\sim \mathrm{SE}\left(32\sigma^4(\frac{\sqrt{n}+1}{n-1})^2,\frac{4\sigma^2}{n-1}\right)$. According to the concentration property of the sub-Exponential random variables presented in~\cite[Corollary 3.17]{duchi}, we have 
	\begin{align} 
		&\mathbb{P}[S_n^2-\var(X)\geq\epsilon]\\
		&\leq \exp\left(-\frac{1}{2}\min\left\{\frac{\epsilon^2}{32\sigma^4(\frac{\sqrt{n}+1}{n-1})^2},\frac{\epsilon}{\frac{4\sigma^2}{n-1}}\right\}\right).
	\end{align} 
	When $ n\geq 12 $, we have $(\frac{\sqrt{n}+1}{n-1})^2\leq \frac{2}{n}  $ and $ \frac{4\sigma^2}{n-1}\leq \frac{8\sigma^2}{n} $. Hence.
	$$\mathbb{P}[S_n^2-\var(X)\geq\epsilon]\leq \exp\left(-\frac{1}{2}\min\left\{\frac{n\epsilon^2}{64\sigma^4},\frac{n\epsilon}{8\sigma^2}\right\}\right)$$ as desired. 
\end{proof}

For $\sigma$-sub-Gaussian distributions, we have the following concentration inequalities for the mean in corresponds to Lemma~\ref{equ:concentration_inequ}:
\begin{align}\label{equ:subgaussian}
	&\mathbb{P}[\hat{\mu}_i(t)-\mu_i\geq\epsilon]\leq \exp\left(-\frac{T_i(t)\epsilon^2}{2\sigma^2}\right),\\
	&\mathbb{P}[\hat{\mu}_i(t)-\mu_i\leq-\epsilon]\leq \exp\left(-\frac{T_i(t)\epsilon^2}{2\sigma^2} \right).
\end{align}
In order to get a tightness result, we force the confidence radius for the variance $\beta$ to be no greater than $c\sigma^2/8$ through out the algorithm (Line $2$ and Lines $27$ and 28 of Algorithm~\ref{VALUCB_subgaussian}), i.e., $ \beta(t,T_i(t))\leq {c\sigma^2}/{8} $ which is equivalent to $ T_i(t) \geq\frac{128}{c}\ln\frac{kNt^4}{\delta}$. Thus, \eqref{equ:subgaussian_variance} simplifies to
\begin{align}\label{equ:subgaussian_variance1}
	&\mathbb{P}[\hat{\sigma}_i^2(t)-\sigma_i^2\geq\epsilon]\leq \exp\left(-\frac{T_i(t)\epsilon^2}{2c\sigma^4}\right),\\
	&\mathbb{P}[\hat{\sigma}_i^2(t)-\sigma_i^2\leq-\epsilon]\leq \exp\left(-\frac{T_i(t)\epsilon^2}{2c\sigma^4}\right).
\end{align}

We are now ready to present the intuitions for Algorithm~\ref{VALUCB_subgaussian}. In the warm-up procedure (Line $2$),
\begin{equation}\label{equ:T0}
	T_0:=\min\left\{t\in\mathbb{N}:t\geq\frac{128}{c}\ln\frac{kNt^4}{\delta}\right\}
\end{equation} and all arms are sampled at each of the $T_0$ time steps. After the warm-up, $ t=T_0+1 $ and $ \beta(t,T_i(t))\leq \frac{c\sigma^2}{8} $ so that \eqref{equ:subgaussian_variance1} holds for all arms. The intuitions for Line $ 3 $ to Line $ 26 $ are the same as Algorithm~\ref{VALUCB}. The only difference here is Lines $ 27 $ and $28$. The definition of $ \mathcal{M}_t $ guarantees each of the arms will be pulled at most once at each time step.
\begin{lem}
	When Algorithm~\ref{VALUCB_subgaussian} has not terminated, for any time step $ t>T_0 $ and arms $ i\in\bar{\mathcal{F}}_{t-1} $, $$ \beta(t,T_i(t)) \leq\frac{c\sigma^2}{8}.$$ 
\end{lem}
\begin{proof}
	We prove this lemma by induction. When $ t=T_0+1 $, by the choice of $ T_0 $, the lemma holds.
	
	Assume that for some $ t>T_0 $, the lemma holds, i.e. for arm $ i\in\bar{\mathcal{F}}_{t-1} $, $ \beta(t,T_i(t)) \leq\frac{c\sigma^2}{8}.$ Conditioned on event $ E $, $ \bar{\mathcal{F}}_t\subset\bar{\mathcal{F}}_{t-1} $.
	If arm $ i \in \bar{\mathcal{F}}_t$ is pulled at time step $ t $, we have $ T_i(t+1)=T_i(t)+1 $. Thus
	\begin{align} 
		&\beta(t+1,T_i(t+1))\leq \frac{c\sigma^2}{8}\\
		\Longleftrightarrow \;& T_i(t+1)\geq \frac{128}{c}\ln\frac{kN(t+1)^4}{\delta}\\
		\Longleftrightarrow \;&
		T_i(t)+1\geq  \frac{128}{c}\ln\frac{kNt^4}{\delta}+4\frac{128}{c}\ln\frac{t+1}{t}. \label{eqn:Leftrightarrows}
	\end{align}
	We now see that if $1\geq\frac{512}{ct}$ holds, then \eqref{eqn:Leftrightarrows} also holds trivially. 
	Consequently, if arm $ i \in \bar{\mathcal{F}}_t$ is not pulled at time step $ t $,  we must have $ T_i(t+1)=T_i(t) $ and $ \beta(t+1,T_i(t+1))\leq\frac{c\sigma^2}{8} $.
	Therefore, the lemma holds for $ t+1 $.
	
	By induction, the lemma holds.
\end{proof}
The above lemma guarantees we can always adopt \eqref{equ:subgaussian_variance1} for all arms in $ \bar{\mathcal{F}}_{t} $ after the warm-up procedure.
\subsection{Analysis}
Define the events 
\begin{align}
	E_i^\mu(t) &:=\{|\hat{\mu}_i(t)-\mu_i|\le\alpha(t,T_i(t))\}, \\
	E_i^{\mathrm{v}}(t)&:=\{ |\hat{\sigma}_i^2(t)-\sigma_i^2|\le\beta(t,T_i(t))\} ,\\
	E_i(t) &:=E_i^\mu(t)\bigcap E_i^{\mathrm{v}}(t),\quad \forall\, i\in[N] . 
\end{align}
For $t> T_0$, define
\begin{equation}\label{E_sub}
	E(t):=\bigcap_{i\in[N]} E_i(t)  \quad \mbox{and}\quad  E:=\bigcap_{t> T_0} E(t) .
\end{equation}

\begin{lem}[Analogue of  Lemma~\ref{Ehappens}]\label{Ehappens_sub}
	Define $E$ as in \eqref{E_sub} and $\alpha(t,T)$ and $\beta(t,T)$ as in \eqref{equ:confidence_radius_subgaussian}, then $E$ occurs with probability at least $1- {\delta}/{k}$.
\end{lem}
\begin{proof}
	Note that our choice of of confidence radii satisfies
	\begin{align}
		&\sum_{t=T_0+1}^{\infty} \sum_{T=1}^{t-1} \exp \left(- \frac{T\alpha(t,T)^{2}}{2\sigma^2}\right)\\
		&\quad\leq\sum_{t=2}^{\infty} \sum_{T=1}^{t-1} \exp \left(- \frac{T\alpha(t,T)^{2}}{2\sigma^2}\right) \leq \frac{\delta}{4kN}, \quad\mbox{and}\\ 
		&\sum_{t=T_0+1}^{\infty} \sum_{T=1}^{t-1} \exp \left(- \frac{T \beta(t,T)^{2}}{2c\sigma^4}\right)\\
		&\quad\leq \sum_{t=2}^{\infty} \sum_{T=1}^{t-1} \exp \left(- \frac{T \beta(t,T)^{2}}{2c\sigma^4}\right) \leq \frac{\delta}{4kN}.
	\end{align} 
	By \eqref{equ:subgaussian}, \eqref{equ:subgaussian_variance1}, and Lemma~\ref{Ehappens_sub}, we conclude that event  $ E $ occurs with probability at least $ 1-\frac{\delta}{k}. $
\end{proof}

Lemma~\ref{ioutisistar} and \ref{useful} still hold for the sub-Gaussian case, since both of them are only established on the confidence bounds. Lemma~\ref{sufcon} and \ref{stop} need to be modified.  

Given a number $ t $ large enough and arm $ i $, define $ u_i(t) $ and $ v_i(t) $ as the smallest numbers of arm pulls such that $\alpha(t,u_i(t))\leq \Delta_i/2 $ and $ \beta(t,v_i(t))\leq \Delta_i^\mathrm{v}$, i.e. 
$$ u_i(t)=\left\lceil \frac{2\sigma^2}{(\frac{\Delta_i}{2})^2}\ln\frac{kNt^4}{\delta}\right\rceil\;\;\mbox{and}\;\;
v_i(t)=\left\lceil \frac{2c\sigma^4}{(\Delta_i^v)^2}\ln\frac{kNt^4}{\delta}\right\rceil.$$
\begin{lem}[Analogue of Lemma~\ref{sufcon}]\label{sufcon_sub}
	Using Algorithm~\ref{VALUCB_subgaussian}, then 
	1) for $ i^\star $,
	\begin{equation*}
		\mathbb{P}[T_{i^\star}(t) > 4u_{i^\star}(t), {i^\star}\!\notin\! \mathcal{R}_t]\le \frac{16\sigma^2\delta}{kN\Delta_{i^\star}^2 t^4}  =:  A_1(i^\star)
	\end{equation*}
	2) for any suboptimal arm $ i \in\mathcal{S}$,
	\begin{equation*}
		\mathbb{P}[T_i(t)>4u_i(t), i\notin \mathcal{S}_t]\leq\frac{16\sigma^2\delta}{kN\Delta_{i}^2 t^4}=:A_2(i)
	\end{equation*}
	3) for any feasible arm $ i\in\mathcal{F} $, 
	\begin{equation*}
		\mathbb{P}[T_i(t)>4v_i(t),i\notin\mathcal{F}_t]\leq\frac{4c\sigma^4\delta}{kN(\Delta_{i}^{\mathrm{v}})^2 t^4}=:A_3(i)			
	\end{equation*}
	4)	for any infeasible arm $ i\in\bar{\mathcal{F}}^c $, 
	\begin{equation*}
		\mathbb{P}[T_i(t)>4v_i(t),i\notin\bar{\mathcal{F}}_t^c]\leq\frac{4c\sigma^4\delta}{kN(\Delta_{i}^{\mathrm{v}})^2 t^4}=:A_4(i)
	\end{equation*}
\end{lem}
\begin{proof}
	For a suboptimal arm $ i $, note that $ \Delta_i/2\leq\bar{\mu}-\mu_i \leq \Delta_i$.
	\begin{align*}\label{equ1_sub}
		&\mathbb{P}[T_i(t)>4u_i(t),i\in \mathcal{N}_t]\\
		&\leq\mathbb{P}[T_i(t)>4u_i(t), i\notin \mathcal{S}_t]\\
		&=\mathbb{P}[T_i(t)>4u_i(t), U_i^\mu(t)\geq\bar{\mu}]\\
		&\leq\sum_{T=4u_i(t)+1}^\infty \mathbb{P}[T_i(t)=T,\hat{\mu}_i(t)-\mu_i\geq\bar{\mu}-\alpha(t,T)-\mu_i]\\
		&\leq\sum_{T=4u_i(t)+1}^\infty \exp\left(-\frac{T}{2\sigma^2}\left(\bar{\mu}-\alpha(t,T)-\mu_i\right)^2\right)\\
		&\leq \sum_{T=4u_i(t)+1}^\infty \exp\left(-\frac{T}{2\sigma^2}\left(\frac{\Delta_i}{2}-\sqrt{\frac{2\sigma^2}{ T} \ln \left(\frac{kN t^{4}}{\delta}\right)}\right)^2\right)\\
		&\leq\sum_{T=4u_i(t)+1}^\infty \exp\left(-\frac{1}{2\sigma^2}\Big(\frac{\Delta_i}{2}\Big)^2\left(\sqrt{T}-\sqrt{u_i(t)}\right)^2\right)\\
		&\leq\int_{4u_i(t)}^\infty\exp\left(-\frac{1}{2\sigma^2}\Big(\frac{\Delta_i}{2}\Big)^2\left(\sqrt{x}-\sqrt{u_i(t)}\right)^2\right)\,\mathrm{d}x\\
		&\leq\frac{4\sigma^2\delta}{(\frac{\Delta_{i}}{2})^2 kNt^4}
	\end{align*}
	The third inequality results from \eqref{equ:subgaussian} and the last two inequalities can be derived using similar techniques in Lemma~\ref{integral}. Similarly, for arm $ i^\star $, 
	\begin{align*}
		\mathbb{P}[T_{i^\star}(t)>4u_{i^\star}(t),i^\star\in \mathcal{N}_t]&\leq\mathbb{P}[T_{i^\star}(t)>4u_{i^\star}(t), {i^\star}\notin \mathcal{R}_t]\\
		&\leq\frac{4\sigma^2\delta}{(\frac{\Delta_{i^\star}}{2})^2 kNt^4}
	\end{align*}
	Note that when $T>4v_i(t)>v_i(t)> 2\ln(kNt^4/\delta)$, \eqref{equ:subgaussian_variance1} can be utilized. For arms $ i\in\mathcal{F} $,
	\begin{align*}
		&\mathbb{P}[T_i(t)>4v_i(t),i\in \partial\mathcal{F}_t]\\
		&\leq\mathbb{P}[T_i(t)>4v_i(t), i\notin \mathcal{F}_t]\\
		&=\mathbb{P}[T_i(t)>4v_i(t), U_i^\mathrm{v}(t)\geq\bar{\sigma}^2]\\
		&\leq\sum_{T=4v_i(t)+1}^\infty \mathbb{P}[T_i(t)=T,\hat{\sigma}_i^2(t)-\sigma_i^2>\bar{\sigma}^2-\sigma_i^2-\beta(t,T)]\\
		&\leq\sum_{T=4v_i(t)+1}^\infty \exp\left(-\frac{T}{2c\sigma^4}\left(\Delta_i^\mathrm{v}-\beta(t,T_i(t))\right)^2\right)\\
		&\leq \sum_{T=4v_i(t)+1}^\infty \exp\left(-\frac{T}{2c\sigma^4}\left(\Delta_i^\mathrm{v}-\sqrt{\frac{\sigma^4}{ T} \ln \left(\frac{kN t^{4}}{\delta}\right)}\right)^2\right)\\
		&\leq \sum_{T=4v_i(t)+1}^\infty \exp\left(-\frac{(\Delta_i^\mathrm{v})^2}{2c\sigma^4}\left(\sqrt{T}-\sqrt{v_i(t )}\right)^2\right)\\
		&\leq\int_{4v_i(t)}^\infty\exp\left(-2\frac{(\Delta_i^\mathrm{v})^2}{4c\sigma^4}\left(\sqrt{x}-\sqrt{v_i(t)}\right)^2\right)\,\mathrm{d}x\\
		&\leq\frac{4c\sigma^4\delta}{(\Delta_{i}^\mathrm{v})^2 kNt^4}
	\end{align*}
	The third inequality results from \eqref{equ:subgaussian_variance1} and the last two inequalities can again be derived using similar techniques in Lemma~\ref{integral}. Similarly, for arm $ i\in\bar{\mathcal{F}}^c $, 
	\begin{align*}
		\mathbb{P}[T_i(t)>4v_i(t),i\in\partial\mathcal{F}_t]&\leq\mathbb{P}[T_i(t)>4v_i(t), i\notin\bar{\mathcal{F}}_t^c]\\
		&\leq\frac{4c\sigma^4\delta}{(\Delta_{i}^\mathrm{v})^2 kNt^4}
	\end{align*}
	This completes the proof.
\end{proof}

\begin{lem}[Analogue of Lemma~\ref{stop}]\label{stop_sub}
	Given an instance $(\nu,\bar{\sigma}^2)$, there exists a constant $ C_{k,N} $ and $$ t^\star = \bigg\lceil C_{k,N}H_{\mathrm{VA},N}^{(\sigma)}\ln\frac{H_{\mathrm{VA},N}^{(\sigma)} }{\delta} \bigg\rceil,$$ such that
	at any time step $ t>t^\star$, the probability that Algorithm~\ref{VALUCB_subgaussian} does not terminate is at most $\frac{22\delta}{kt^2}$.
\end{lem}
\begin{proof}
	Let $t$ be a sufficiently large integer; in particular $t>2T_0$ (which  will be justified after this lemma), and $ \mathcal{T}:=\{\lceil{t}/{2}\rceil,\ldots, t-1\} $. Define events
	\begin{itemize}
		\item $\mathcal{I}_1$: $  \exists\, s\in\mathcal{T}$ such that $i^\star\notin\mathcal{F}_s\cap\mathcal{R}_s$ and $ T_{i^\star}(s)>\max\{4u_{i^\star}(s),4v_{i^\star}(s)\}$.
		\item $\mathcal{I}_2$: $ \exists\, i\in\mathcal{F}\cap\mathcal{S}, s\in\mathcal{T}$ such that $ i\notin\mathcal{S}_s$ and $ T_i(s)>4u_i(s)$.
		\item $\mathcal{I}_3$: $ \exists\, i\in\bar{\mathcal{F}}^c\cap\mathcal{R}, s\in\mathcal{T}$ such that $i\notin \bar{\mathcal{F}}_s^c$ and $T_i(s)>4v_i(s)$
		\item $\mathcal{I}_4$: $ \exists\, i\in\bar{\mathcal{F}}^c\cap\mathcal{S}, s\in\mathcal{T}$ such that $i\notin \bar{\mathcal{F}}_s^c\cup\mathcal{S}_s$ and $ T_i(s)>\min\{4u_i(s),4v_i(s)\}$.
		\item $\mathcal{I}_5$: $  \exists\, s\in\mathcal{T}$ such that $E(s)$ does not occur.
	\end{itemize} 
	If VA-LUCB terminates before $  \lceil{t}/{2}\rceil $, the statement is definitely right. If not, we assume the above five events do not occur. Based on Lemma \ref{useful}, the additional time steps after $ \lceil{t}/{2}\rceil -1$ for VA-LUCB can be upper bounded as \eqref{estimation3} in \underline{Derivation $3$} on the next page.
	\begin{figure*}
		\underline{Derivation $3$}:
		\begin{align}
			&\sum_{s\in\mathcal{T}}\mathbbm{1}\{ i_s\in(\mathcal{F}_s\cap\mathcal{N}_s)\cup (\partial\mathcal{F}_s\backslash\mathcal{S}_s)\,\mbox{or}\, c_s\in(\mathcal{F}_s\cap\mathcal{N}_s)\cup (\partial\mathcal{F}_s\backslash\mathcal{S}_s)\}\\
			&\leq\sum_{s\in\mathcal{T}}\sum_{i\in[N]}\mathbbm{1}\{i=i_s\,\mbox{or}\, c_s,\, i\in(\mathcal{F}_s\cap\mathcal{N}_s)\cup (\partial\mathcal{F}_s\backslash\mathcal{S}_s)\}\\
			&=\sum_{s\in\mathcal{T}}\left(\mathbbm{1}\{i^\star=i_s\,\mbox{or}\, c_s,\, i^\star\in(\mathcal{F}_s\cap\mathcal{N}_s)\cup (\partial\mathcal{F}_s\backslash\mathcal{S}_s)\}
			+\sum_{i\in\mathcal{F}\cap\mathcal{S}}\mathbbm{1}\{i=i_s\,\mbox{or}\, c_s,\, i\in(\mathcal{F}_s\cap\mathcal{N}_s)\cup (\partial\mathcal{F}_s\backslash\mathcal{S}_s)\}\right.\\*
			&\quad \left.+\sum_{i\in\bar{\mathcal{F}}^c\cap\mathcal{R}}\mathbbm{1}\{i=i_s\,\mbox{or}\, c_s,\, i\in(\mathcal{F}_s\cap\mathcal{N}_s)\cup (\partial\mathcal{F}_s\backslash\mathcal{S}_s)\}
			+\sum_{i\in\bar{\mathcal{F}}^c\cap\mathcal{S}}\mathbbm{1}\{i=i_s\,\mbox{or}\, c_s,\, i\in(\mathcal{F}_s\cap\mathcal{N}_s)\cup (\partial\mathcal{F}_s\backslash\mathcal{S}_s)\}\right)\\
			&\leq\sum_{s\in\mathcal{T}}\left(\mathbbm{1}\{i^\star=i_s\,\mbox{or}\, c_s,\, i^\star\notin\mathcal{F}_s\cap\mathcal{R}_s\}
			+\sum_{i\in\mathcal{F}\cap\mathcal{S}}\mathbbm{1}\{i=i_s\,\mbox{or}\, c_s,\, i\notin\mathcal{S}_s\}\right.\\*
			&\quad \left.+\sum_{i\in\bar{\mathcal{F}}^c\cap\mathcal{R}}\mathbbm{1}\{i=i_s\,\mbox{or}\, c_s,\, i\notin\bar{\mathcal{F}}_s^c\}
			+\sum_{i\in\bar{\mathcal{F}}^c\cap\mathcal{S}}\mathbbm{1}\{i=i_s\,\mbox{or}\, c_s,\, i\notin\bar{\mathcal{F}}_s^c\cup\mathcal{S}_s\}\right)\\
			&\leq\sum_{s\in\mathcal{T}}\left(\mathbbm{1}\{i^\star=i_s\,\mbox{or}\, c_s,\, T_{i^\star}(s)\leq\max\{4u_{i^\star}(s),4v_{i^\star}(s)\}\}
			+\sum_{i\in\mathcal{F}\cap\mathcal{S}}\mathbbm{1}\{i=i_s\,\mbox{or}\, c_s,\, T_i(s)\leq4u_i(s)\}\right.\\
			&\quad \left.+\sum_{i\in\bar{\mathcal{F}}^c\cap\mathcal{R}}\mathbbm{1}\{i=i_s\,\mbox{or}\, c_s,\, T_i(s)\leq4v_i(s)\}
			+\sum_{i\in\bar{\mathcal{F}}^c\cap\mathcal{S}}\mathbbm{1}\{i=i_s\,\mbox{or}\, c_s,\, T_i(s)\leq\min\{4u_i(s),4v_i(s)\}\}\right)\\
			&\leq\sum_{s\in\mathcal{T}}\mathbbm{1}\{i^\star=i_s\,\mbox{or}\, c_s,\, T_{i^\star}(s)\leq\max\{4u_{i^\star}(s),4v_{i^\star}(s)\}\}
			+\sum_{i\in\mathcal{F}\cap\mathcal{S}}\sum_{s\in\mathcal{T}}\mathbbm{1}\{i=i_s\,\mbox{or}\, c_s,\, T_i(s)\leq4u_i(s)\}\\
			&\quad +\sum_{i\in\bar{\mathcal{F}}^c\cap\mathcal{R}}\sum_{s\in\mathcal{T}}\mathbbm{1}\{i=i_s\,\mbox{or}\, c_s,\, T_i(s)\leq4v_i(s)\}
			+\sum_{i\in\bar{\mathcal{F}}^c\cap\mathcal{S}}\sum_{s\in\mathcal{T}}\mathbbm{1}\{i=i_s\,\mbox{or}\, c_s,\, T_i(s)\leq\min\{4u_i(s),4v_i(s)\}\}\\
			&\leq\max\{4u_{i^\star}(t),4v_{i^\star}(t)\}
			+\sum_{i\in\mathcal{F}\cap\mathcal{S}}4u_i(t)+\sum_{i\in\bar{\mathcal{F}}^c\cap\mathcal{R}} 4v_i(t)
			+\sum_{i\in\bar{\mathcal{F}}^c\cap\mathcal{S}}\min\{4u_i(t),4v_i(t)\}.\label{estimation3}
		\end{align}
			\hrulefill
	\end{figure*}
	On the other hand, we observe that if the time step $ t>t^\star:=\lceil C_{k,N}H_{\mathrm{VA},N}^{(\sigma)}\ln\frac{H_{\mathrm{VA},N}^{(\sigma)}}{\delta}\rceil$ where $C_{k,N}$ is a constant that depends on the constant $ k $ and the number of arms $ N $, we have \eqref{estimation4} in \underline{Derivation $4$} on Page~\pageref{page:derivation4}.
	\begin{figure*} 
		\underline{Derivation $4$}: \label{page:derivation4}
		\begin{align}
			&\max\{4u_{i^\star}(t),4v_{i^\star}(t)\}
			+\sum_{i\in\mathcal{F}\cap\mathcal{S}}4u_i(t)+\sum_{i\in\bar{\mathcal{F}}^c\cap\mathcal{R}} 4v_i(t)
			+\sum_{i\in\bar{\mathcal{F}}^c\cap\mathcal{S}}\min\{4u_i(t),4v_i(t)\}\}\\
			&=\max\left\{4\left\lceil{\frac{2\sigma^2}{(\frac{\Delta_{i^\star}}{2})^2} \ln \left(\frac{kN t^{4}}{\delta}\right)}\right\rceil,4\left\lceil{\frac{2c\sigma^4}{(\Delta_{i^\star}^{\mathrm{v}})^2} \ln \left(\frac{k N t^{4}}{\delta}\right)}\right\rceil\right\}+
			4\sum_{i\in\mathcal{F}\cap\mathcal{S}}\left\lceil{\frac{2\sigma^2}{(\frac{\Delta_{i}}{2})^2}\ln \left(\frac{2 N t^{4}}{\delta}\right)}\right\rceil\\
			&\quad +\sum_{i\in\bar{\mathcal{F}}^c\cap\mathcal{R}}4\left\lceil{\frac{2c\sigma^4}{(\Delta_{i}^{\mathrm{v}})^2} \ln \left(\frac{k N t^{4}}{\delta}\right)}\right\rceil
			+\sum_{i\in\bar{\mathcal{F}}^c\cap\mathcal{S}}\min\left\{4\left\lceil{\frac{2\sigma^2}{(\frac{\Delta_{i}}{2})^2} \ln \left(\frac{kN t^{4}}{\delta}\right)}\right\rceil,4\left\lceil{\frac{2c\sigma^4}{(\Delta_{i}^{\mathrm{v}})^2} \ln \left(\frac{k N t^{4}}{\delta}\right)}\right\rceil\right\}\\
			&\leq 4N+4\max\left\{{\frac{2\sigma^2}{(\frac{\Delta_{i^\star}}{2})^2} \ln \left(\frac{kN t^{4}}{\delta}\right)},{\frac{2c\sigma^4}{(\Delta_{i^\star}^{\mathrm{v}})^2} \ln \left(\frac{k N t^{4}}{\delta}\right)}\right\}+
			\sum_{i\in\mathcal{F}\cap\mathcal{S}}{\frac{2\sigma^2}{(\frac{\Delta_{i}}{2})^2}\ln \left(\frac{2 N t^{4}}{\delta}\right)}\\
			&\quad +4\sum_{i\in\bar{\mathcal{F}}^c\cap\mathcal{R}}{\frac{2c\sigma^4}{(\Delta_{i}^{\mathrm{v}})^2} \ln \left(\frac{k N t^{4}}{\delta}\right)}
			+4\sum_{i\in\bar{\mathcal{F}}^c\cap\mathcal{S}}\min\left\{{\frac{2\sigma^2}{(\frac{\Delta_{i}}{2})^2} \ln \left(\frac{kN t^{4}}{\delta}\right)},{\frac{2c\sigma^4}{(\Delta_{i}^{\mathrm{v}})^2} \ln \left(\frac{k N t^{4}}{\delta}\right)}\right\}\\
			&= 4N+4H_\mathrm{VA}^{(\sigma)}\ln \left(\frac{k N t^{4}}{\delta}\right)\\
			&\leq 4N+4H_{\mathrm{VA},N}^{(\sigma)}\ln \left(\frac{k N t^{4}}{\delta}\right)\\
			&= \left(4N+4H_{\mathrm{VA},N}^{(\sigma)}\ln k\right)+4H_{\mathrm{VA},N}^{(\sigma)}\ln \frac{N}{\delta}+16H_{\mathrm{VA},N}^{(\sigma)}\ln \left(C_{k,N}H_{\mathrm{VA},N}^{(\sigma)}\ln\frac{H_{\mathrm{VA},N}^{(\sigma)}}{\delta}\right)\\
			&\leq \left(4N+4H_{\mathrm{VA},N}^{(\sigma)}\ln k+16H_{\mathrm{VA},N}^{(\sigma)}\ln C_{k,N} \right)+4H_{\mathrm{VA},N}^{(\sigma)}\ln \frac{N}{\delta}+32H_{\mathrm{VA},N}^{(\sigma)}\ln \frac{H_{\mathrm{VA},N}^{(\sigma)}}{\delta}\\
			&\leq  \frac{1}{2}C_{k,N}H_{\mathrm{VA},N}^{(\sigma)}\ln\frac{H_{\mathrm{VA},N}^{(\sigma)}}{\delta}\\
			&\leq  \frac{t}{2}.\label{estimation4}
		\end{align}
		\hrulefill
	\end{figure*}
	So the total number of time steps is bounded by 
	\begin{align}
		&\left\lceil\frac{t}{2}\right\rceil-1+\max\{4u_{i^\star}(t),4v_{i^\star}(t)\}
		+\sum_{i\in\mathcal{F}\cap\mathcal{S}}4u_i(t)\\
		&\;+\sum_{i\in\bar{\mathcal{F}}^c\cap\mathcal{R}} 4v_i(t)
		+\sum_{i\in\bar{\mathcal{F}}^c\cap\mathcal{S}}\min\{4u_i(t),4v_i(t)\}\leq t
	\end{align}
	We then compute the probability of the event $\cup_{i\in[5]}\mathcal{I}_i$. To simplify  notations used in the following, we define the hardness quantity
	\begin{align}
		\widetilde{H}^{(\sigma)}&:=\frac{8\sigma^2}{(\Delta_{i^\star})^2} +\frac{2c\sigma^4}{(\Delta_{i^\star}^{\mathrm{v}})^2} +
		\sum_{i\in\mathcal{F}\cap\mathcal{S}}{\frac{8\sigma^2}{(\Delta_{i})^2}}\\
		&\quad
		+\sum_{i\in\bar{\mathcal{F}}^c\cap\mathcal{R}}{\frac{2c\sigma^4}{(\Delta_{i}^{\mathrm{v}})^2}}
		+\sum_{i\in\bar{\mathcal{F}}^c\cap\mathcal{S}}g(u_i(t),v_i(t))
	\end{align}
	where $g(u_i(t),v_i(t))=\frac{8\sigma^2}{(\Delta_{i})^2}\cdot\mathbbm{1}\{u_i(t)<v_i(t)\}+\frac{2c\sigma^4}{(\Delta_{i}^{\mathrm{v}})^2} \cdot\mathbbm{1}\{u_i(t)\geq v_i(t)\}
	$ which is exactly $\min\{\frac{2\sigma^2}{(\frac{\Delta_{i}}{2})^2} ,\frac{2c\sigma^4}{(\Delta_{i}^{\mathrm{v}})^2}\}$ if we ignore the ceiling operators in $u_i(t)$ and $v_i(t)$. 
	By Lemma \ref{equ:concentration_inequ} and Lemma \ref{sufcon},
	\begin{equation*}
		\begin{aligned}
			&\mathbb{P}[\mathcal{I}_1]\leq \sum_{s\in\mathcal{T}} \bigg( \frac{16\sigma^2\delta}{kN\Delta_{i^\star}^2 s^4}+\frac{4c\sigma^4\delta}{kN(\Delta_{i^\star}^{\mathrm{v}})^2 s^4}\bigg)\\
			&\mathbb{P}[\mathcal{I}_2]\leq \sum_{i\in\mathcal{F}\cap\mathcal{S}}\sum_{s\in\mathcal{T}} \frac{16\sigma^2\delta}{kN\Delta_{i}^2 s^4}\\
			&\mathbb{P}[\mathcal{I}_3]\leq \sum_{i\in\bar{\mathcal{F}}^c\cap\mathcal{R}}\sum_{s\in\mathcal{T}} \frac{4c\sigma^4\delta}{kN(\Delta_{i}^{\mathrm{v}})^2 s^4}\\
			&\mathbb{P}[\mathcal{I}_4]\leq \sum_{i\in\bar{\mathcal{F}}^c\cap\mathcal{S}}\sum_{s\in\mathcal{T}} \frac{2\delta}{kNs^4}g(u_i(s),v_i(s))\\
			&\mathbb{P}[\mathcal{I}_5]\leq \sum_{s\in\mathcal{T}} \frac{4\delta}{ks^3}\leq \frac{6\delta}{kt^2}\\
		\end{aligned}
	\end{equation*}
	which implies that 
	\begin{equation*}
		\mathbb{P} \Bigg[\underset{{i\in[5]}}{\bigcup}\mathcal{I}_i \Bigg]
		\leq 2\widetilde{H}^{(\sigma)} \sum_{s\in\mathcal{T}}\frac{\delta}{kNs^4}
		+\frac{6\delta}{kt^2}
		\leq \frac{22\delta}{kt^2}
	\end{equation*}
	where the last inequality utilizes the fact that $\widetilde{H}^{(\sigma)}<2H_\mathrm{VA}^{(\sigma)}\leq2H_{\mathrm{VA},N}^{(\sigma)}<t$ and $N\geq 2$.
	This yields the upper bound of the probability that the algorithm does not terminate at time step~$ t $. 
\end{proof}

At this point, we give the constants in the above analysis. We set $ k=2 $. 
\begin{itemize}
	\item According to the definition of $T_0$ in \eqref{equ:T0}, $c=64$, $\delta\leq 0.05$ and the number of arms $N\geq 2$, 
	\begin{align}
		T_0&=\min\left\{t\in\mathbb{N}:t\geq2\ln\frac{2Nt^4}{\delta}\right\}\\
		&\geq\min\left\{t\in\mathbb{N}:t\geq2\ln(80t^4)\right\}>12.
	\end{align} 
	In other words, after the warm-up procedure in Line $2$ of Algorithm~\ref{VALUCB_subgaussian}, Lemma~\ref{lem:cv_sub} and the concentration inequality for the variance \eqref{equ:subgaussian_variance1} can be applied.
	\item According to Lemma~\ref{Ehappens_sub}, event $ E $ occurs with probability at least $ 1-\delta/2 $. 
	\item From the definition of $T_0$ and the computations in Lemma~\ref{stop_sub}, we can easily see that
	\begin{align}\label{equ:compareT_0}
		2T_0&\leq 4\left\lceil \ln\frac{kNT_0^4}{\delta}\right\rceil< 4N+4H_{\mathrm{VA},N}^{(\sigma)}\ln \left(\frac{k N t^{4}}{\delta}\right)\\
		&<\frac{1}{2}C_{k,N}H_{\mathrm{VA},N}^{(\sigma)}\ln\frac{H_{\mathrm{VA},N}^{(\sigma)}}{\delta}. 
	\end{align}
	This justifies the assumption that $t^\star=C_{k,N}H_{\mathrm{VA},N}^{(\sigma)}\ln({H_{\mathrm{VA},N}^{(\sigma)}}/{\delta})$ in the analysis of Lemma~\ref{sufcon_sub}.
\end{itemize}

The total number of arm pulls can be estimated as follows. Since $ T_i(t) $ grows with $ t $, when $ T_i(t)\geq \frac{128}{c}\ln\frac{2Nt^4}{\delta} =2\ln\frac{2Nt^4}{\delta}$, arm $ i $ will not appear in $ \mathcal{M}_t $. 
For any $t>t^\star$,
since $4N+4H_{\mathrm{VA},N}^{(\sigma)}\ln \left(\frac{2 N t^{4}}{\delta}\right)\leq \frac{t}{2}$, 
$$ \sum_{i\in[N]}2\ln\frac{2Nt^4}{\delta}= 2N\ln\frac{2Nt^4}{\delta}\leq \frac{t}{4}.$$
So the total number of pulls is upper bounded by $(2+\frac{1}{4})t$ if the algorithm terminates after $t$ time steps.

Equipped with the above preparatory results, we are now ready to present the proof of Theorem~\ref{thm:up_bd_sub}.
\begin{proof}[Proof of Theorem~\ref{thm:up_bd_sub}]
	According to Lemma~\ref{ioutisistar} and Lemma~\ref{Ehappens_sub}, 
	on the event $E$, which occurs with probability at least $1-{\delta}/{2}$, and the termination of Algorithm~\ref{VALUCB_subgaussian}, it succeeds. Lemma~\ref{stop_sub} indicates that Algorithm~\ref{VALUCB_subgaussian} terminates at time $t>t^\star>265>5$ with probability at least ${11\delta}/{t^2}$. So Algorithm~\ref{VALUCB_subgaussian} succeeds after $O\left( H_{\mathrm{VA},N}^{(\sigma)}\ln\frac{H_{\mathrm{VA},N}^{(\sigma)}}{\delta}\right)$ time steps with probability at least $1-({\delta}/{2}+{11\delta}/{5^2})\geq1-\delta$.
	
	Note that the sample complexity is at most $2.25$ times of number of time steps. To get the expected sample complexity, we first compute the expected number of time steps. According to Lemma~\ref{stop_sub}, Algorithm~\ref{VALUCB_subgaussian} does not terminate at time step $t>t^\star$ with probability at most $\frac{11\delta}{t^2}$. The expected number of time steps is upper bounded by
	\begin{align}
		t^\star+\sum_{t=t^{\star}+1}^\infty \frac{11\delta}{t^2}\leq t^\star+\frac{11\delta}{t^\star}\leq 265\, H_{\mathrm{VA},N}^{(\sigma)}\ln\frac{H_{\mathrm{VA},N}^{(\sigma)}}{\delta}+1. 
	\end{align}
	Thus, the expected sample complexity is upper bounded by $600H_{\mathrm{VA},N}^{(\sigma)}\ln\frac{H_{\mathrm{VA},N}^{(\sigma)}}{\delta}+3=O\left(H_{\mathrm{VA},N}^{(\sigma)}\ln\frac{H_{\mathrm{VA},N}^{(\sigma)}}{\delta}\right) $.
	This completes the proof of Theorem \ref{thm:up_bd_sub}. 
\end{proof}
Hence, we have generalized the analysis from the bounded rewards case to the sub-Gaussian rewards case, and the conclusion is that the hardness parameter $H_{\mathrm{VA}, N}^{(\sigma) }$ is merely a constant factor off  from its bounded rewards counterpart $H_\mathrm{VA}$.

\section{Results Using LIL-Based Confidence Bounds}\label{LIL}
While there exist various approaches to apply the LIL techniques to VA-LUCB algorithm \cite{jamieson14lil,howard2021time,simchowitz17the,tanczos2017a}, we adopt a simple non-asymptotic LIL concentration bound from Jamieson {\em et al.}~\cite{jamieson14lil} to show that 
different confidence bounds utilized in  VA-LUCB  can lead to slightly different upper bounds on the expected stopping time.

\begin{lem}[Lemma 3 in \cite{jamieson14lil}]\label{goodbound}
	Let $\{X_i\}_{i=1}^\infty$ be a sequence of i.i.d.\ centered sub-Gaussian random variables with scale parameter $\sigma$. Fix any $\epsilon\in(0,1)$ and $\delta\in(0,\ln(1+\epsilon)/e)$. Then one has 
	\begin{align}
		&\mathbb{P}\bigg[\forall\, t\!\in\!\mathbb{N}:\!\sum_{s=1}^t X_s\!\leq\! (1\!+\!\sqrt{\epsilon})\sqrt{2\sigma^2\left(1\!+\!\epsilon\right)t\ln\left(\frac{\ln\left((1\!+\!\epsilon)t\right)}{\delta}\right)}
		\bigg]
	    \\
		&\qquad\geq
		1-\xi(\delta),
	\end{align}
	where $\xi(\delta):=\frac{2+\epsilon}{\epsilon}\big(\frac{\delta}{\log (1+\epsilon)}\big)^{1+\epsilon}$.
\end{lem}
Define the ``good events''
    \begin{align}
    	\tilde{E}_i(t):=\big\{|\hat{\mu}_i(t)-\mu_i|\le\tilde{\alpha}(T_i(t)),|\tilde{\sigma}_i^2(t)-\sigma_i^2|\le\tilde{\beta}(T_i(t))
    	\big\}
    \end{align} 
    for all $t\in\mathbb{N}$ and $i\in[N]$, as well as their intersections
    \begin{align}
    \tilde{E}_i:=\bigcap_{t\in\mathbb{N}} 
    \tilde{E}_i(t)  \quad \mbox{and}\quad  \tilde{E}:=\bigcap_{i\in[N]} \tilde{E}_i .
    \end{align}
\begin{lem}\label{lem:lil_var_bd}
	With the choice of the confidence radii in \eqref{newbounds}, the event $\tilde{E}$ occurs  with probability at least  $1-\xi(\delta)$.  In particular, if we set $\epsilon=0.9$ and $\delta<0.1$, then $\xi(\delta)\geq\delta$ which implies that the event $\tilde{E}$ occurs with probability at least $1-\delta$. 
\end{lem}

\begin{proof}
	We first record three facts. First, any distribution supported on $[0,1]$ is $1/2$-sub-Gaussian. Second, the rewards of an arm from different time steps are i.i.d.\ and the realizations from different arms are independent from each other. Third, if arm $i$ is not pulled at time step $t$, all the statistics for arm $i$ at the time step $t$ (including sample mean, sample variance and concentration bound)  remain valid in time step $t+1$.
	By a direct application of Lemma~\ref{goodbound} to the sample mean $\hat{\mu}_i(t)$ and the  sample second  moment $\hat{M}_2(t):=\frac{1}{T_i(t)}\sum_{s=1}^{t-1} X_{s,i}^2\mathbbm{1}\{i\in\mathcal{J}_s\}$ of arm $i\in[N]$, and setting $\{Y_{s,i}\}_{s=1}^\infty \stackrel{\text{i.i.d.}}{\sim} \nu_i$, we have 
    \begin{align}\label{lilbound}
		&\;\mathbb{P}\big[\exists \, t\geq 1:\left|\hat{\mu}_i(t)-\mu_i\right|  \geq \tilde{\alpha}(T_i(t))\big]\\
		&=\mathbb{P}\left[\exists \, t\geq 1:\bigg|\frac{1}{t}\sum_{s=1}^{t} Y_{s,i}-\mu_i\bigg|  \geq \tilde{\alpha}(t)
 		\right] 
		\leq
		2\xi\Big(\frac{\delta}{4N}\Big),\quad\mbox{and}
		\\
		 &\;\mathbb{P}\left[\exists\, t\geq 1:\big| \hat{M}_2(t)-(\mu_i^2+\sigma_i^2)\big|\geq \tilde{\alpha}(T_i(t))\right]\\
		 &=\mathbb{P} \left[\exists \, t\geq 1:\bigg|\frac{1}{t}\sum_{s=1}^{t} Y_{s,i}^2-(\mu_i^2\!+\!\sigma_i^2)\bigg| \! \geq\! \tilde{\alpha}(t)
 		\right] 
		\!\leq\!
		2\xi\Big(\frac{\delta}{4N}\Big).
	\end{align}
	Since the rewards are in $[0,1]$, $|\hat{\mu}_i^2(t)-\mu_i^2|\le |\hat{\mu}_i(t)+\mu_i| \cdot |\hat{\mu}_i(t)-\mu_i|\le 2 |\hat{\mu}_i(t)-\mu_i| $. Using this and the 
	 triangle inequality, we obtain for every $t\ge1$,
	\begin{align}		
		\big|\tilde{\sigma}^2_i(t)-\sigma_i^2\big|
		&\leq \big|\hat{M}_2(t)-(\mu_i^2+\sigma_i^2)\big|+\big|\hat{\mu}_i^2(t)-\mu_i^2\big|\\*
		&
		\leq \tilde{\alpha}(T_i(t))+2\tilde{\alpha}(T_i(t))
		=\tilde{\beta}(T_i(t)).
	\end{align}
	Therefore, by a union bound, with probability at least 
	\begin{align}
		 1-4N\cdot \xi \Big(\frac{\delta}{4N}\Big)
		&=1-4N\cdot\frac{2+\epsilon}{\epsilon}\left(\frac{\delta}{4N\log (1+\epsilon)}\right)^{1+\epsilon}\\
		&=1-(4N)^{-\epsilon}\xi(\delta)
		\geq 1-\xi(\delta)
	\end{align}
	event $\tilde{E}$ occurs.
\end{proof}

\begin{restatable}[LIL-Based Upper Bound]{thm}{thmupperbound-LIL}\label{thm:lil_up_bd}
	Given an instance $ (\nu,\bar{\sigma}^2) $ and confidence parameter $ \delta $, 
	with probability at least $1-\xi(\delta)$,  
	VA-LUCB with the LIL-based confidence bounds succeeds and terminates in 
	\begin{equation*}
		O\left(H_{\mathrm{VA}}^{(1)}\ln\frac{N}{\delta}+H_{\mathrm{VA}}^{(3)}\right) 
	\end{equation*} 
	time steps,where $H_{\mathrm{VA}}^{(1)}$ and $H_{\mathrm{VA}}^{(3)}$ are defined in \eqref{lilhardness1} and \eqref{lilhardness2} respectively.
\end{restatable}
\begin{proof}
	For a suboptimal arm $i\in\mathcal{S}$, when $\tilde{\alpha}(T_i(t))\leq \Delta_i/4$,
	$ \hat{\mu}_i(t)+\tilde{\alpha}(T_i(t))\leq \mu_i+2\tilde{\alpha}(T_i(t))\leq \mu_i+\frac{\Delta_i}{2}\leq \bar{\mu},$
	which indicates arm $i$ is not in $\mathcal{N}_t$. The same holds for $i^\star$.
	For a feasible arm $i\in\mathcal{F}$, when $\tilde{\beta}(T_i(t))\leq \Delta_i^\mathrm{v}/2$,
	$ \tilde{\sigma}_i^2(t)+\tilde{\beta}(T_i(t))\leq \mu_i+2\tilde{\beta}(T_i(t))\leq \bar{\sigma}^2,$
	which indicates arm $i$ is not in $\partial{\mathcal{F}}_t$. The same holds for the infeasible arms $i\in\bar{\mathcal{F}}^c$.
	By a direct computation (see, for example, \cite[Eqn.~(4)]{Jamieson2014}),
	\begin{align}
		&\min\Big\{t\!\in\!\mathbb{N}:\tilde{\alpha}(t)\leq \frac{\Delta_i}{4}\Big\}\leq \frac{2\gamma}{\Delta_i^2}\ln\left(\frac{8N\ln(\gamma\Delta_i^{-2})}{\delta}\right), \;\mbox{and}\\
		&\min\Big\{t\!\in\!\mathbb{N}:\tilde{\beta}(t)\!\leq\! \frac{\Delta_i^\mathrm{v}}{2}\Big\}
		\!\leq\! \frac{2\gamma}{(\frac{2}{3}\Delta_i^\mathrm{v})^2}\ln\!\left(\frac{8N\ln(\gamma(\frac{2}{3}\Delta_i^\mathrm{v})^{-2})}{\delta}\right),
	\end{align}
	where $\gamma=\gamma_\epsilon=8(1+\sqrt{\epsilon})^2(1+\epsilon)^2$. According to Lemma~\ref{useful} (which also holds even if the confidence radii have been changed), at least one of pulled arms belongs to the set $(\partial\mathcal{F}_t\backslash\mathcal{S}_t)\cup(\mathcal{F}_t\cap \mathcal{N}_t)$, so after 
	\begin{align}\label{tildet}
		\tilde{t}:=& 
		\max\left\{\frac{2\gamma}{\Delta_{i^\star}^2}\ln\left(\frac{8N\ln(\gamma\Delta_{i^\star}^{-2})}{\delta}\right),\right.\\
		&\qquad\left.\frac{2\gamma}{(\frac{2}{3}\Delta_{i^\star}^\mathrm{v})^2}\ln\left(\frac{8N\ln(\gamma(\frac{2}{3}\Delta_{i^\star}^\mathrm{v})^{-2})}{\delta}\right)\right\}\\
		&+
		\sum_{i\in\mathcal{F}\cap\mathcal{S}}\frac{2\gamma}{\Delta_i^2}\ln\left(\frac{8N\ln(\gamma\Delta_i^{-2})}{\delta}\right)\\
		&
		+
		\sum_{i\in\bar{\mathcal{F}}^c\cap\mathcal{R}}\frac{2\gamma}{(\frac{2}{3}\Delta_i^\mathrm{v})^2}\ln\left(\frac{8N\ln(\gamma(\frac{2}{3}\Delta_i^\mathrm{v})^{-2})}{\delta}\right)\\
		&+
		\sum_{i\in\bar{\mathcal{F}}^c\cap\mathcal{S}}
		\min\left\{\frac{2\gamma}{\Delta_i^2}\ln\left(\frac{8N\ln(\gamma\Delta_i^{-2})}{\delta}\right),\right.\\
		&\left.\qquad\qquad\frac{2\gamma}{(\frac{2}{3}\Delta_i^\mathrm{v})^2}\ln\left(\frac{8N\ln(\gamma(\frac{2}{3}\Delta_i^\mathrm{v})^{-2})}{\delta}\right)\right\}
	\end{align} 
	time steps, the algorithm must have terminated. 
	Note that for $x\in(0,1)$, $\frac{1}{x^2}\ln\left(\frac{8N\ln(\gamma x^{-2})}{\delta}\right)$ decreases as $x$ increases and 
	\begin{align}
	    &\frac{1}{x^2}\ln\left(\frac{8N\ln(\gamma x^{-2})}{\delta}\right)\\
	    &=\frac{1}{x^2}\left(\ln\frac{8N}{\delta}+\ln\left(\ln(\gamma(1+\epsilon)+\ln x^{-2})\right)\right)\\
	    &\leq \frac{1}{x^2}\ln\frac{8N}{\delta}+\frac{1}{x^2}\left(\ln\ln_+(x^{-2})+\gamma\right)\\
	    &\leq \frac{c_1}{x^2}\ln\frac{N}{\delta}+\frac{\ln\ln_+(x^{-2})}{x^2}
	\end{align}
	where $c_1=\ln 8+\gamma$ is a known constant that only depends on $\epsilon$. Thus, $\tilde{t}$ can be   upper bounded by 
	\begin{equation}\label{lilupperbound}
	    2\gamma c_1 H_{\mathrm{VA}}^{(1)}\ln\frac{N}{\delta}+2\gamma H_{\mathrm{VA}}^{(3)}
	=O\left(H_{\mathrm{VA}}^{(1)}\ln\frac{N}{\delta}+H_{\mathrm{VA}}^{(3)}\right)
	\end{equation}
	where $H_{\mathrm{VA}}^{(1)}$ and $H_{\mathrm{VA}}^{(3)}$ are defined in \eqref{lilhardness1} and \eqref{lilhardness2} respectively.
\end{proof}

\section{Proof of Lower Bound}\label{lowerbound}
Let $\mathrm{KL}(\nu,\nu^\prime)$ denote the KL divergence between distributions $\nu$ and $\nu^\prime$, and 
\begin{align*}
	d(x, y) := x \ln\bigg( \frac{x}{ y} \bigg)+(1-x) \ln\bigg( \frac{1-x}{1-y} \bigg)
\end{align*}
denote the Kullback--Leibler (KL) divergence between the Bernoulli distributions $\mathrm{Bern}(x)$ and  $\mathrm{Bern}(y)$.
\begin{lem}[Pinsker's and reverse Pinsker's inequality~\cite{Friedrich2019}]
	\label{thm:pinsker}

	{Let $P$ and $Q$ be two distributions that are defined in the same finite space 
		$\mathcal{A}$ and have 
		the same support.} We have
	\begin{align*}
		\delta(P,Q)^2 \le \frac{1}{2} \mathrm{KL}(P,Q)
		\le \frac{1}{  \alpha_Q} \delta(P,Q)^2
	\end{align*}
	where
	$
	\delta(P,Q) := \sup\{| P(A)-Q(A)| \,:\, A \subset \mathcal{A}   \} 
	=\frac{1}{2} \sum_{x\in  \mathcal{A}} |P(x) - Q(x)|$
	is the total variational distance,
	and  
	$\alpha_Q := \min_{x\in \mathcal{A}: Q(x)>0} Q(x)
	$.
\end{lem}

\begin{lem}[Lemma 1 in~\cite{Kaufmann2016}] \label{lemma:lb_KLdecomp}
	For any $1\le j \le N$,
	{\setlength\abovedisplayskip{.5em}
		\setlength\belowdisplayskip{.4em}
		\begin{align*}
			& 
			\sum_{i=1}^N \mathbb{E}_{ \mathcal{G}_0} [T_{i}(\tau^{(0)})] \cdot d\big( \mu_i^{(0)}, \mu_i^{(j)} \big)
			\ge 
			\sup_{\mathcal{E} \in \mathcal{G}_0  } 
			d\big(    \mathbb{P}_{\mathcal{G}_0} (\mathcal{E})  ,~   \mathbb{P}_{\mathcal{G}_j} (\mathcal{E})   \big).
			\nonumber 
	\end{align*}}     
\end{lem} 

\begin{proof}[Proof of Theorem \ref{thm:low_bd}]
	Fix a $\delta$-PAC algorithm  $\pi$. We consider the instances containing arms with Bernoulli reward distributions. By simple algebra, an arm $ i\in[N] $ with reward distribution $ \nu_i=\mathrm{Bern}(\mu_i) $ and hence variance $\mu_i(1-\mu_i)$ is infeasible if and only if $ \mu_i\in(\underline{a},\bar{a}) $ and this arm is feasible otherwise.
	
	\noindent
	{\bf Step 1: Classification of instances.}
	Based on the values of $\{ \mu_i \}_{i=1}^N$, we have one of the following cases:
	\begin{itemize}
		\item[(i)] $\underline{a}<\mu_N\le \ldots \leq\mu_2\leq\mu_1<\bar{a} $,
		\item[(ii)] $0<\mu_N\le \ldots \leq\mu_2<\mu_1<1$, and $\bar{a}<\mu_1 $,
		\item[(iii)] $ \mu_i<\bar{a} $ for all $i\in[N]$, $ \{ i\in[N]: \mu_i =\underline{a} \} =\emptyset$, $ \{ i\in[N]: \mu_i <\underline{a} \} \ne \emptyset$, and arm $1=\mathop{\rm argmax}{ \{ \mu_i : \mu_i < \underline{a}\} } $.
	\end{itemize}
	which are shown in Figure \ref{cases}.
	\begin{figure*}[t]
		\centering
		\subfigure[Case (i)]{
			\includegraphics[width=.475\textwidth]{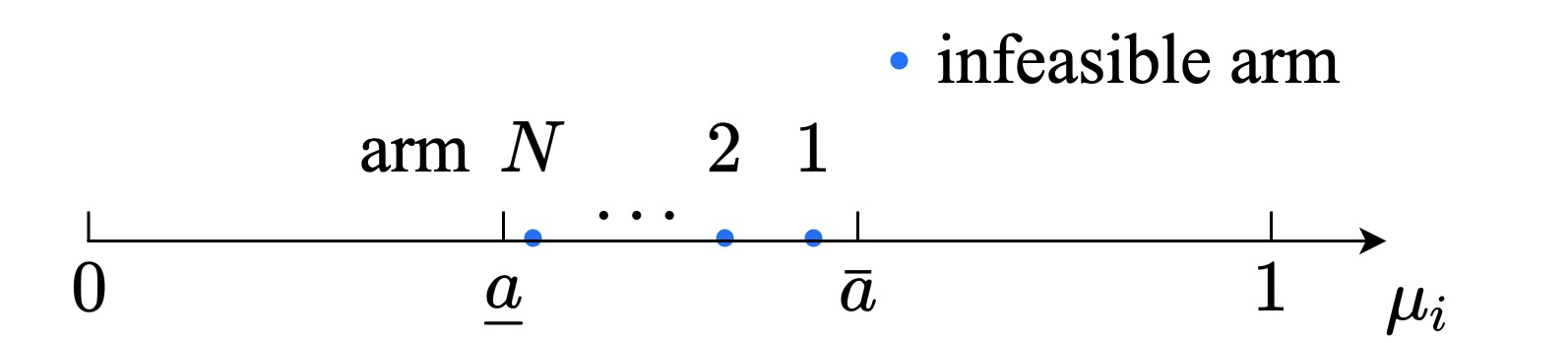}
		}\hspace{-.06in}
		\subfigure[Case (ii)]{
			\includegraphics[width=.475\textwidth]{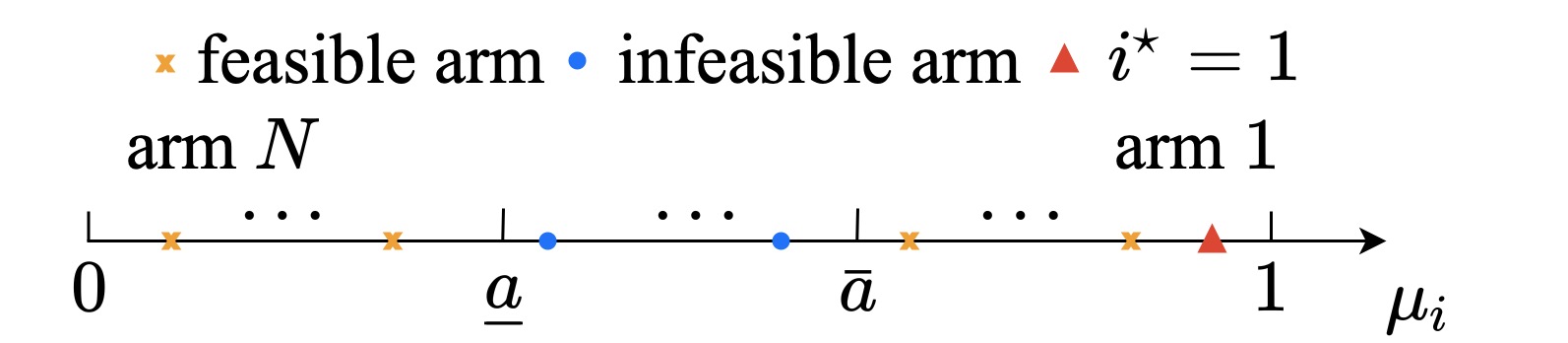}
		}\hspace{-.06in}\\
		\subfigure[Case (iii)]{
			\includegraphics[width=.475\textwidth]{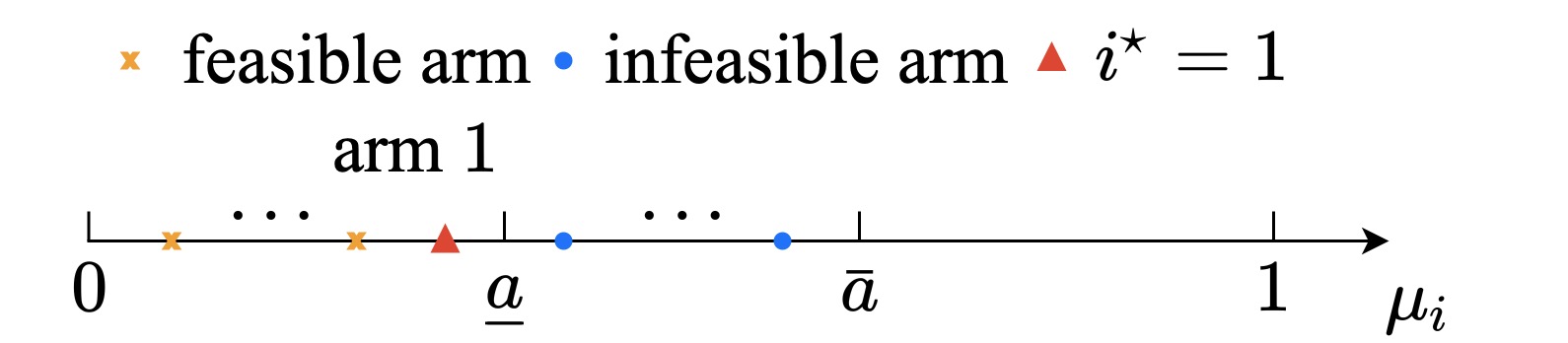}
		} 
		\caption{Illustrations of the three cases for the  proof of the lower bound}
		\label{cases}
	\end{figure*}
	There exists no feasible arm in Case (i), while there exists at least one feasible arm in Cases~(ii) and~(iii). 
	Then we construct instances by making small modifications to the reward distributions of the arms. These constructed instances are hard to distinguish from each of the cases above, in the sense that the constructed instances will lead to different conclusions towards the feasibility of the instance or the best feasible arm.  
	
	{\bf Step 2: Analysis of each case.}
	Subsequently, we analyze each case individually.
	In each case, we construct $N+1$ instances such that 
	under instance $j$~($0\le j \le N$), the stochastic reward of arm $i$~($1\le i \le N$) is drawn from distribution 
	$$ \nu_i^{(j)} :=  \nu_i \cdot \mathbbm{1}\{i\ne j \} + \nu_i^\prime \cdot \mathbbm{1}\{ i = j \}   ,$$
	where $ \nu_i^\prime $ will be specified in each case.
	Let $\mu_i^{(j)}$ denote the expectation of arm $i$ and $(\sigma_i^{(j)})^2$ denote the variance of arm $i$ under instance $j$. 
	Under instance $0\le j\le N$, we define several other notations as follows: 
	\begin{itemize}
		\item Let $X_{i,r}^{(j)}  $ be the random reward of arm $i$ at round $r$. Then $X_{i,r}^{(j)}  \in \{0, 1\} $.
		\item Let $i_r^{(j)}$ be the pulled arm at round $r$, and $\mathcal{G}_r^{(j)} = \{ (i_s^{(j)}, X_{  i_s^{(j)},s } )\}^r_{s=1}$ be the sequence of pulled arms and observed rewards up to and including round $r$. 
		\item Let $\tau^{(j)}$ denote the stopping time.
	\end{itemize} 
	For simplicity, we abbreviate  
	$\mathcal{G}_{\tau^{(j)}}^{(j)}$ as  
	$\mathcal{G}_j$. 
	
	\noindent
	\underline{\bf Case (i)}: $\underline{a}<\mu_N\le \ldots \leq\mu_2\leq\mu_1<\bar{a} $
	
	\noindent{\bf Construction of instances.}
	Fix any $0< b_1 < \underline{a} < \bar{a} < b_2<1$.
	We define $\nu_i^\prime = \mathrm{Bern}( \mu_i^\prime ) $ for arm $i\in[N]$ with
	\begin{align*}
		\mu_i^\prime = 
		\left\{
		\begin{array}{ll}
			b_1  & \mu_i < 1/2,
			\\ 
			b_2  & \mu_i \ge 1/2.
		\end{array}
		\right.
	\end{align*} 
	
	Therefore,
	\begin{itemize}
		\item under instance $0$, since $(\sigma_i^{(0)})^2 =\mu_i(1-\mu_i)>\bar{\sigma}^2$ for all arm $i\in[N]$, there is no feasible arm;  
		\item under instance $j$~($1\le j \le L$), we see that 
		\begin{align*}
			&(\sigma_i^{(j)})^2 =\mu_i(1-\mu_i) >\bar{\sigma}^2, \quad \forall\, i\ne j, 
			\text{ and }\\
			&(\sigma_j^{(j)})^2 =\mu_i^\prime(1-\mu_i^\prime)< \bar{\sigma}^2,
		\end{align*} 
		implying that $j$ is the unique optimal feasible arm.
	\end{itemize} 
	Since algorithm $ \pi $ is $ \delta $-PAC, we have $\mathbb{P}_{ \mathcal{G}_0 } [i_{\mathrm{out}} = j] < \delta$ and $\mathbb{P}_{ \mathcal{G}_j } [i_{\mathrm{out}} \ne j] < \delta$ for all $1\le j \le N$.

	\noindent{\bf Change measure.}
	Next, we lower bound $\mathbb{E}_{\mathcal{G}_0} [T_i( \tau^{(0)}  )]$ with the KL divergence by applying Lemma \ref{lemma:lb_KLdecomp}.

	Note that $d(a,b)$ equals to the KL divergence between $\mathrm{Bern}(a)$ and $\mathrm{Bern}(b)$ and
	\begin{align*}
		d( x, 1-x) \ge \ln \bigg( \frac{ 1 }{2.4x} \bigg) \quad \forall\, x \in (0,1].
	\end{align*}
	Let $\mathcal{E}_j:= \{ i_{\mathrm{out}} =  j  \}$, then
	\begin{align*}
		&\delta \ge \mathbb{P}_{ \mathcal{G}_0 } [ \{ i_{\mathrm{out}} \} \ne \emptyset] \ge \mathbb{P}_{ \mathcal{G}_0 } [   \mathcal{E}_j   ],
		\\
		&1-\delta \le \mathbb{P}_{ \mathcal{G}_j } [ i_{\mathrm{out}} = j]   =  \mathbb{P}_{ \mathcal{G}_j }  [\mathcal{E} _j ] .
	\end{align*}
	Since $\mu_i^{(0)} = \mu_i^{(j)}$ for all $i\ne j$  under instance $1\le j\le N$, we have
	\begin{align*}
		\mathbb{E}_{ \mathcal{G}_0} [T_{j}(\tau^{(0)})]
		&\ge 
		\frac{
			d(    \mathbb{P}_{\mathcal{G}_0} [\mathcal{E}_j] ,~   \mathbb{P}_{\mathcal{G}_j} [\mathcal{E}_j] )}{ d( \mu_j^{(0)}, \mu_j^{(j)} ) }
		\\
		&\ge 
		\frac{
			d(    \delta , 1-\delta ) }{ d( \mu_j, \mu_j^\prime  ) }
		\ge  \frac{ -\ln( 2.4\delta) }{ d( \mu_j, \mu_j^\prime ) }.
	\end{align*}
	Therefore,
	\begin{align*}
		&\mathbb{E}_{ \mathcal{G}_0 }[ \tau^{(0)} ]
		\ge 
		\sum_{j=1}^N \mathbb{E}_{ \mathcal{G}_0} [T_{j}(\tau^{(0)})]
		\ge 
		\ln \bigg( \frac{1}{2.4\delta} \bigg) \cdot  \sum_{i=1}^N \frac{ 1 }{ d ( \mu_i, \mu_i^\prime ) }
		\\
		&
		=
		\ln \bigg( \frac{1}{2.4\delta} \bigg) \cdot  
		\bigg( \sum_{ i: \mu_i<1/2} \frac{ 1 }{ d ( \mu_i, b_1 ) } +
		\sum_{ i: \mu_i\ge 1/2  } \frac{ 1 }{ d ( \mu_i, b_2 ) }
		\bigg)
		.
	\end{align*}  
	Since $(b_1 ,b_2)$ can be chosen arbitrarily from $ B= \{ (b_1, b_2 ): 0< b_1 < \underline{a} < \bar{a} < b_2<1 \}$, we have
	\begin{align*}
		&\mathbb{E}_{ \mathcal{G}_0 }[ \tau^{(0)} ]\\
		&
		\ge 
		\sup_{ (b_1, b_2 ) \in B   }
		\ln \bigg( \frac{1}{2.4\delta} \bigg) \\
		&\qquad\cdot  
		\bigg( \sum_{  i: \mu_i<1/2 } \frac{ 1 }{ d ( \mu_i, b_1 ) } +
		\sum_{ i: \mu_i\ge 1/2  } \frac{ 1 }{ d ( \mu_i, b_2 ) }
		\bigg)
		\\ 
		&=
		\ln \bigg( \frac{1}{2.4\delta} \bigg) \cdot  
		\bigg( \sum_{  i: \mu_i<1/2 } \frac{ 1 }{ d ( \mu_i, \underline{a} ) } +
		\sum_{ i: \mu_i\ge 1/2  } \frac{ 1 }{ d ( \mu_i, \bar{a} ) }
		\bigg) 
		.
	\end{align*}

	\noindent
	\underline{\bf Case (ii)}: $0<\mu_N\le \ldots \leq\mu_2<\mu_1<1$, and $\bar{a}<\mu_1 $.
	
	\noindent{\bf Construction of instances.}
	Fix any $ \underline{a}< b_1 < \bar{a} < \mu_1< b_2< 1 $.  We define $\nu_1^\prime=\mathrm{Bern}(b_1)$ and $\nu_i^\prime=\mathrm{Bern}(b_2)$ for all arms $i\ne 1 \vphantom{\big(}$. 
	Therefore, 
	\begin{itemize}
		\item under instance $0$, we see that
		\begin{align*}
			&\mu_1^{(0)} = \mu_1 > \mu_i = \mu_i^{(0)}  \quad \forall  \, i\ne 1,
			\text{ and }\\
			&
			(\sigma_1^{(0)} )^2= \mu_1 (1-\mu_1 ) <  \bar{a}(1-\bar{a}) = \bar{\sigma}^2;
		\end{align*}
		\item  under instance $1$, we see that
		\begin{align*} 
			(\sigma_1^{(1)})^2 = b_1 (1-b_1 ) >  \bar{a}(1-\bar{a}) = \bar{\sigma}^2;
		\end{align*}
		\item  under instance $j$~($2\le j \le L$), we see that 
		\begin{align*}
			&\mu_j^{(j)} = b_2  > \mu_i = \mu_{i}^{(j)} \quad \forall \, i\ne j,
			\text{ and }\\
			&
			(\sigma_j^{(j)} )^2= b_2(1- b_2 ) < \bar{a}(1-\bar{a}) = \bar{\sigma}^2.
		\end{align*} 
	\end{itemize} 
	Since arm $1$ is the unique best feasible arm under instance $0$,
	arm $1$ is not feasible under instance $1$,
	and arm $j$ is the unique best feasible arm under instance $j$~($2\le j \le N$), we have $\mathbb{P}_{ \mathcal{G}_0 } [i_{\mathrm{out}} \ne 1 ]< \delta$ and $\mathbb{P}_{ \mathcal{G}_1 } [i_{\mathrm{out}} = 1] < \delta$, and $\mathbb{P}_{ \mathcal{G}_j } [i_{\mathrm{out}} \ne j]< \delta$ for all $2 \le j \le N$.

	\noindent {\bf Change of measure.}  
	We again lower bound $\mathbb{E}_{\mathcal{G}_0} [T_i( \tau^{(0)}  )]$ with the KL divergence by applying Lemma~\ref{lemma:lb_KLdecomp}.
	Let $\mathcal{E}:= \{ i_{\mathrm{out}} \ne 1  \}$, then
	\begin{align*}
		&\delta \ge \mathbb{P}_{ \mathcal{G}_0 } [ i_{\mathrm{out}} \ne 1]= \mathbb{P}_{ \mathcal{G}_0 }[   \mathcal{E} ] ,
		1- \delta \le \mathbb{P}_{ \mathcal{G}_1 } [ i_{\mathrm{out}} \ne 1] = \mathbb{P}_{ \mathcal{G}_1 }  [ \mathcal{E} ] ,
		\\
		&1-\delta \le \mathbb{P}_{ \mathcal{G}_j } [ i_{\mathrm{out}} = j] \le \mathbb{P}_{ \mathcal{G}_j } [ i_{\mathrm{out}} \ne 1] =  \mathbb{P}_{ \mathcal{G}_j }  [ \mathcal{E} ]
	\end{align*} 
	for $\forall \, 2\le j \le L.$
	Since $\mu_i^{(0)} = \mu_i^{(j)}$ for all $i\ne j$  under instance $1\le j\le N$, we have
	\begin{align*}
		\mathbb{E}_{ \mathcal{G}_0}& [T_{j}(\tau^{(0)})]
		\ge 
		\frac{
			d\big(    \mathbb{P}_{\mathcal{G}_0} (\mathcal{E} )  ,~   \mathbb{P}_{\mathcal{G}_j} (\mathcal{E} )   \big)}{ d\big( \mu_j^{(0)}, \mu_j^{(j)} \big) }
		\ge 
		\frac{
			d(    \delta , 1-\delta ) }{ d\big(  \mu_j^{(0)}, \mu_j^{(j)}   \big) }
		\\
		&\ge  \frac{ -\ln( 2.4\delta) }{ d\big(  \mu_j^{(0)}, \mu_j^{(j)} \big) } 
		\ge 
		\left\{
		\begin{array}{ll}
			\displaystyle
			\frac{ -\ln( 2.4\delta) }{ d (  \mu_1, b_{1 }  ) },& j=1
			\\ \vspace{-.8em} \\
			\displaystyle
			\frac{ -\ln( 2.4\delta)     }{ d (  \mu_j, b_2  ) } ,&2\le j \le L
		\end{array}
		\right.
		.
	\end{align*}
	Therefore,
	\begin{align*}
		\mathbb{E}_{ \mathcal{G}_0 }[ \tau^{(0)} ]
		\ge 
		\ln \bigg( \frac{1}{2.4\delta} \bigg) \cdot \bigg( \frac{ 1 }{ d ( \mu_1, b_1 ) }  + \sum_{i=2}^N \frac{ 1 }{ d ( \mu_i, b_2 ) }  \bigg)  .
	\end{align*} 
	Since $(b_1,b_2)$ can be chosen arbitrarily from $ B= \{ (b_1, b_2 ): \underline{a}< b_1 < \bar{a} < \mu_1< b_2< 1  \}$, we have
	\begin{align*}
		&\mathbb{E}_{ \mathcal{G}_0 }[ \tau^{(0)} ]\\
		&\ge 
		\sup_{ (b_1, b_2 ) \in B   }
		\ln \bigg( \frac{1}{2.4\delta} \bigg) 
		\cdot  
		\bigg( \frac{ 1 }{ d ( \mu_1, b_1 ) }  + \sum_{i=2}^N \frac{ 1 }{ d ( \mu_i, b_2 ) }  \bigg) 
		\\ & 
		=
		\ln \bigg( \frac{1}{2.4\delta} \bigg) \cdot  
		\bigg( \frac{ 1 }{ d ( \mu_1, \bar{a} ) }  + \sum_{i=2}^N \frac{ 1 }{ d ( \mu_i, \mu_1 ) }  \bigg) 
		.
	\end{align*}   
	\noindent
	\underline{\bf Case (iii)}: 
	$\mu_i< \bar{a} $ for all $i\in[N]$, $ \{ i\in[N]: \mu_i =\underline{a} \} =\emptyset$, $ \{ i\in[N]: \mu_i <\underline{a} \} \ne \emptyset$, and arm $1=\mathop{\rm argmax}\ \{ \mu_i : \mu_i < \underline{a}\}  $.
	
	\noindent{\bf Construction of instances.}
	Fix any $ \mu_1 < b_1 < \underline{a}< b_2 < \bar{a} < b_3< 1 $.  
	We define $\nu_i^\prime = \mathrm{Bern}( \mu_i^\prime ) $ for arm $i\in[N]$ with
	\begin{align*}
		\mu_i^\prime = 
		\left\{
		\begin{array}{ll}
			b_1  ,& \mu_i <1/2 \text{ and } \mu_i \ne \mu_1,
			\\ 
			b_2  ,& \mu_i =\mu_1,
			\\ 
			b_3 , & \mu_i \ge 1/2.
		\end{array}
		\right.
	\end{align*} 
	Therefore, 
	arm $1$ is the unique best feasible arm under instance $0$,
	arm $1$ is not feasible  under instance $1$,
	and arm $j$ is the unique best feasible arm under instance $j$~($2\le j \le N$), we have $\mathbb{P}_{ \mathcal{G}_0 } [i_{\mathrm{out}} \ne 1 ]< \delta$, $\mathbb{P}_{ \mathcal{G}_1 } [i_{\mathrm{out}} = 1]< \delta$, and $\mathbb{P}_{ \mathcal{G}_j } [i_{\mathrm{out}} \ne j]< \delta$ for all $2 \le j \le N$.

	\noindent {\bf Change of measure.}  
	We again lower bound $\mathbb{E}_{\mathcal{G}_0} [T_i( \tau^{(0)}  )]$ with the KL divergence by applying Lemma~\ref{lemma:lb_KLdecomp}.
	Let $\mathcal{E}:= \{ i_{\mathrm{out}} \ne 1  \}$, then
	\begin{align*}
		&\delta \ge \mathbb{P}_{ \mathcal{G}_0 } [i_{\mathrm{out}} \ne 1]= \mathbb{P}_{ \mathcal{G}_0 } [ \mathcal{E}] ,
		1-\delta  \le \mathbb{P}_{ \mathcal{G}_1} [i_{\mathrm{out}} \ne 1]=  \mathbb{P}_{ \mathcal{G}_1 } [ \mathcal{E}] .
		\\
		&1-\delta  \le \mathbb{P}_{ \mathcal{G}_j } [ i_{\mathrm{out}} = j] \le \mathbb{P}_{ \mathcal{G}_j } [i_{\mathrm{out}} \ne 1] =  \mathbb{P}_{ \mathcal{G}_j }  [  \mathcal{E}] 
	\end{align*} 
	for $\forall\, 2\le j \le L. $
	Since $\mu_i^{(0)} = \mu_i^{(j)}$ for all $i\ne j$  under instance $1\le j\le N$, we have
	\begin{align*}
		&\mathbb{E}_{ \mathcal{G}_0} [T_{j}(\tau^{(0)})]
		\ge 
		\frac{
			d\big(    \mathbb{P}_{\mathcal{G}_0} [\mathcal{E}] ,  \mathbb{P}_{\mathcal{G}_j}[ \mathcal{E}]   \big)}{ d\big( \mu_j^{(0)}, \mu_j^{(j)} \big) }
		\ge 
		\frac{
			d(    \delta , 1-\delta ) }{ d\big(  \mu_j , \mu_j^\prime   \big) }
		\\
		&\ge  \frac{ -\ln( 2.4\delta) }{ d\big(  \mu_j , \mu_j^\prime \big) } 
		\ge 
		\left\{
		\begin{array}{ll}
			\displaystyle
			\frac{ -\ln( 2.4\delta) }{ d (  \mu_1, b_{1 }  ) },& \mu_j < 1/2 \text{ and } \mu_j\ne \mu_1
			\\ \vspace{-.8em} \\
			\displaystyle
			\frac{ -\ln( 2.4\delta)     }{ d (  \mu_j, b_2  ) } ,& \mu_j = \mu_1
			\\ \vspace{-.8em} \\
			\displaystyle
			\frac{ -\ln( 2.4\delta)     }{ d (  \mu_j, b_3  ) } ,& \mu_j \ge 1/2
		\end{array}
		\right.
		. 
	\end{align*}
	Therefore,
	\begin{align*}
		\mathbb{E}_{ \mathcal{G}_0 }[ \tau^{(0)} ]
		&\ge 
		\sum_{j=1}^N \mathbb{E}_{ \mathcal{G}_0} [T_{j}(\tau^{(0)})]
		\\
		&
		\ge  
		\ln \bigg( \frac{1}{2.4\delta} \bigg) \cdot  
		\bigg(  \frac{ 1 }{ d ( \mu_1, b_2 ) }  +
		\sum_{ \substack{ i: \mu_i<1/2,\\ \mu_i \ne \mu_1  }} \frac{ 1 }{ d ( \mu_i, b_1 ) } 
		\\
		&\quad+ \sum_{ i: \mu_i\ge 1/2  } \frac{ 1 }{ d ( \mu_i, b_3 ) } 
		\bigg). 
	\end{align*}  
	Since $(b_1 ,b_2,b_3)$ can be chosen arbitrarily from $ B= \{ (b_1, b_2,b_3 ):   \mu_1 < b_1 < \underline{a}< b_2 < \bar{a} < b_3< 1   \}$, we have
	\begin{align*}
		\mathbb{E}_{ \mathcal{G}_0 }[ \tau^{(0)} ] 
		&\ge 
		\sup_{ (b_1, b_2,b_3 ) \in B   }
		\ln \bigg( \frac{1}{2.4\delta} \bigg) \cdot  
		\bigg(  \frac{ 1 }{ d ( \mu_1, b_2 ) }  +
		\\
		&
		\qquad\qquad
		\sum_{ \substack{ i: \mu_i<1/2, \\ \mu_i \ne \mu_1} } \frac{ 1 }{ d ( \mu_i, b_1 ) } 
		+ \sum_{ i: \mu_i\ge 1/2  } \frac{ 1 }{ d ( \mu_i, b_3 ) } 
		\bigg)
		\\ & 
		=
		\ln \bigg( \frac{1}{2.4\delta} \bigg) \cdot  
		\bigg(  \frac{ 1 }{ d ( \mu_1, \underline{a} ) }  +
		\sum_{ \substack{i: \mu_i<1/2,\\ \mu_i \ne \mu_1 }} \frac{ 1 }{ d ( \mu_i, \mu_1 ) } \\
		&
		\quad
		+ \sum_{ i: \mu_i\ge 1/2  } \frac{ 1 }{ d ( \mu_i, \bar{a} ) } 
		\bigg)
		.
	\end{align*}

	\noindent{\bf Step 3: Simplification of the bounds with $H_{\mathrm{VA}}$.}
	We further lower bound the sample complexity in each case. 
	
	\noindent
	\underline{\bf Case (i)}: When $\underline{a}<\mu_N\le \ldots \leq\mu_2\leq\mu_1<\bar{a} $, we have
	\begin{align*}
		\mathbb{E}_{ \mathcal{G}_0 }[ \tau^{(0)} ]
		&
		\ge  
		\ln \bigg( \frac{1}{2.4\delta} \bigg) \cdot  
		\bigg( \sum_{  i: \mu_i<1/2 } \frac{ 1 }{ d ( \mu_i, \underline{a} ) } \\
		&
		\quad+
		\sum_{ i: \mu_i\ge 1/2  } \frac{ 1 }{ d ( \mu_i, \bar{a} ) }
		\bigg) 
		.
	\end{align*}  
	By Theorem~\ref{thm:pinsker}, we have
	\begin{align*}
		&d( \mu_i,  \bar{a} ) \le \frac{ 1 }{ \underline{a} } \cdot ( \mu_i - \bar{a} )^2 \quad
		\text{ and }\quad
		d( \mu_i,  \underline{a} ) \le \frac{ 1 }{ \underline{a} } \cdot ( \mu_i - \underline{a} )^2,\\
		&\text{ where } \quad\underline{a}<1/2 < \bar{a}=1-\underline{a} 
		.
	\end{align*}
	Let $\sigma_a^2$ be the variance of $\mathrm{Bern}(a)$ and $\sigma_b^2$ be the variance of $\mathrm{Bern}(b)$ for any $ a,b\in(0,1) $.
	Then
	\begin{align*}
		\sigma_a^2 - \sigma_b^2  &= a(1-a) - b(1-b) = a-b - a^2 + b^2 \\
		&= (a-b)(1-a-b), \\
		( \sigma_a^2 - \sigma_b^2 )^2 &= (a-b)^2 (1-a-b)^2.
	\end{align*}
	Note that $\bar{a}( 1-\bar{a} ) = \bar{\sigma}^2$. For $\mu_i\ge 1/2$,
	\begin{align*}
		d( \mu_i,  \bar{a} ) &\le \frac{ ( \mu_i - \bar{a} )^2  \cdot (1- \mu_i -\bar{a} )^2 }{  \underline{a} (1- \mu_i -\bar{a} )^2 }  
		= \frac{  ( \sigma_i^2 - \bar{\sigma}^2 )^2 }{ \underline{a} (1- \mu_i -\bar{a} )^2 }
		\\
		&\le  \frac{  ( \sigma_i^2 - \bar{\sigma}^2 )^2 }{ \underline{a} ( \bar{a} -1/2)^2 }
		= \frac{  ( \sigma_i^2 - \bar{\sigma}^2 )^2 }{ \underline{a} ( 1/2 - \underline{a} )^2 };
	\end{align*}
	for $\mu_i < 1/2$,
	\begin{align*}
		d( \mu_i,  \underline{a} ) &\le \frac{ ( \mu_i - \underline{a} )^2  \cdot (1- \mu_i -\underline{a} )^2 }{  \underline{a} (1- \mu_i -\underline{a} )^2 }  
		= \frac{  ( \sigma_i^2 - \bar{\sigma}^2 )^2 }{ \underline{a} (1- \mu_i -\underline{a} )^2 }
		\\
		&\le  \frac{  ( \sigma_i^2 - \bar{\sigma}^2 )^2 }{ \underline{a} ( 1/2 - \underline{a} )^2 }.
	\end{align*} 
	Since there is no feasible arm,  $\bar{\mathcal{F}}^c\cap\mathcal{R} = [N]$ and $\bar{\mathcal{F}}^c\cap\mathcal{S}=\mathcal{F}=\emptyset$.
	Notice an obvious fact 
	\begin{equation}\label{simplefact}
		(1/2- \underline{a} )^2 = \underline{a}^2 - \underline{a} + 1/4 = 1/4 - \bar{\sigma}^2,
	\end{equation}
	thus
	\begin{align*}
		\mathbb{E} [ \tau ]
		&
		\ge  
		\underline{a}(1/4 - \bar{\sigma}^2 ) \ln \bigg( \frac{1}{2.4\delta} \bigg) \cdot  
		\sum_{ i \in [N] } \frac{ 1 }{  ( \sigma_i^2 - \bar{\sigma}^2 )^2  }  
		\\
		&= H_{\mathrm{VA}} \ln \bigg( \frac{1}{2.4\delta} \bigg)  \cdot \underline{a}(1/4 - \bar{\sigma}^2 ) 
		.
	\end{align*}

	\noindent
	\underline{\bf Case (ii)}: When $0<\mu_N\le \ldots \leq\mu_2<\mu_1<1$, and $\bar{a}<\mu_1 $, we have 
	\begin{align*}
		\mathbb{E}_{ \mathcal{G}_0 }[ \tau^{(0)} ]
		&
		\ge  
		\ln \bigg( \frac{1}{2.4\delta} \bigg) \cdot  
		\bigg( \frac{ 1 }{ d ( \mu_1, \bar{a} ) }  + \sum_{j=2}^N \frac{ 1 }{ d ( \mu_j, \mu_1 ) }  \bigg) 
		.
	\end{align*}   
	We first apply Theorem~\ref{thm:pinsker} to see that
	\begin{align}
		d( \mu_i,  \mu_1 ) &\le \frac{ ( \mu_i - \mu_1 )^2  }{  1-\mu_1 } ,\quad \forall\, i\ne 1 ,
		\nonumber \\
		d( \mu_1,  \bar{a} ) &\le \frac{ ( \mu_1- \bar{a} )^2  \cdot (1- \mu_1 -\bar{a} )^2 }{  \underline{a} (1- \mu_1-\bar{a} )^2 }  
		= \frac{  ( \sigma_1^2 - \bar{\sigma}^2 )^2 }{ \underline{a} (1- \mu_1 -\bar{a} )^2 }
		\\
		&
		\le  \frac{  ( \sigma_1^2 - \bar{\sigma}^2 )^2 }{ \underline{a} ( 2\bar{a} -1)^2 }
		= \frac{  ( \sigma_1^2 - \bar{\sigma}^2 )^2 }{ 4\underline{a} ( 1/2 - \underline{a} )^2 }.  \label{eq:bai_fix_conf_low_bd_simp_1}
	\end{align}
	Since $\bar{\mathcal{F}}^c\cap\mathcal{R} $ is empty, $ \mathcal{S}= [ N]\setminus\{ i^\star \} $ and hence 
	\begin{align}
		\Delta_1& = \Delta_{i^\star}  = \min_{i \in \mathcal{S} } \Delta_i
		\cdot   \mathbbm{1}\{ \mathcal{S}\ne \emptyset  \} +   \infty \cdot   \mathbbm{1}\{ \mathcal{S}= \emptyset  \} 
		\\
		&= \min_{i \in[N] \setminus\{ i^\star\}  } \Delta_i .\label{eq:bai_fix_conf_low_bd_simp_2}
	\end{align} 
	Lastly,
	\begin{align*}
		\mathbb{E} [ \tau  ]
		&
		\overset{(a)}{\ge}  
		\ln \bigg( \frac{1}{2.4\delta} \bigg) \cdot  \frac{ 4\underline{a} ( 1/2 - \underline{a} )^2 }{  ( \sigma_1^2 - \bar{\sigma}^2 )^2 }
		\\*
		&\quad+
		(1-\mu_1 )\ln \bigg( \frac{1}{2.4\delta} \bigg) \cdot  \sum_{i=2}^N \frac{ 1 }{ ( \mu_i - \mu_1 )^2  } 
		\\
		&
		\overset{(b)}{\ge}  
		\ln \bigg( \frac{1}{2.4\delta} \bigg) \cdot  \frac{ 4\underline{a} ( 1/4 - \bar{\sigma}^2 )  }{  (  \Delta_1^{\mathrm{v}} )^2 }
		\\*
		&\quad
		+
		\frac{ 1-\mu_1}{2} \ln \bigg( \frac{1}{2.4\delta} \bigg) \cdot  \sum_{i=1}^N \frac{ 1 }{ \Delta_i^2  }  
		\\&
		\ge 
		\ln \bigg( \frac{1}{2.4\delta} \bigg) \cdot  \frac{ \min\{ 4\underline{a}( 1/4 - \bar{\sigma}^2 ),  (1-\mu_1)/8\} }{ \max\{ ( \Delta_1^{\mathrm{v}} )^2, (\Delta_1/2)^2\} }
		\\&\quad 
		+ \frac{ 1-\mu_1}{ 8 } \ln \bigg( \frac{1}{2.4\delta} \bigg) \cdot  \bigg( \sum_{i\in \mathcal{F}\cap\mathcal{S}  }  \frac{ 1 }{ (\Delta_i/2)^2  } 
		\\
		&\qquad+  \sum_{i\in   \bar{\mathcal{F}}^c\cap\mathcal{S}}  \frac{ 1 }{ \max\{ (\Delta_i/2)^2, (\Delta_i^{\mathrm{v}})^2 \} } 
		\bigg)
		\\&
		= H_{\mathrm{VA}} \ln \bigg( \frac{1}{2.4\delta} \bigg) \cdot  \min \bigg\{   4\underline{a} \bigg(\frac{1}{4} - \bar{\sigma}^2\bigg) ,  \frac{ 1-\mu_1}{ 8 }
		\bigg\}  
		.
	\end{align*}
	We derive $(a)$ by applying the lower bounds on the KL divergences in \eqref{eq:bai_fix_conf_low_bd_simp_1}, and $(b)$ follows from \eqref{simplefact} and \eqref{eq:bai_fix_conf_low_bd_simp_2}. 
	
	\noindent
	\underline{\bf Case (iii)}: When
	$ \mu_i<\bar{a} $ for all $i\in[N]$, $ \{ i\in[N]: \mu_i =\underline{a} \} =\emptyset$, $ \{ i\in[N]: \mu_i <\underline{a} \} \ne \emptyset$, and arm $1=\mathop{\rm argmax} \{ \mu_i : \mu_i < \underline{a}\}  $, we have  
	\begin{align*}
		\mathbb{E}_{ \mathcal{G}_0 }[ \tau^{(0)} ]
		&
		\ge  
		\ln \bigg( \frac{1}{2.4\delta} \bigg) \cdot  
		\bigg(  \frac{ 1 }{ d ( \mu_1, \underline{a} ) }  +
		\sum_{ \substack{i: \mu_i<1/2, \\ \mu_i \ne \mu_1 } } \frac{ 1 }{ d ( \mu_i, \mu_1 ) } 
		\\
		&\quad
		+ \sum_{ i: \mu_i\ge 1/2  } \frac{ 1 }{ d ( \mu_i, \bar{a} ) } 
		\bigg)
		.
	\end{align*}     
	Similar to the analysis of Cases (i) and (ii), we have
	\begin{align}
		d( \mu_i,  \bar{a} ) & \le 
		\frac{  ( \sigma_i^2 - \bar{\sigma}^2 )^2 }{ \underline{a} ( 1/4 - \bar{\sigma}^2)  } \quad \forall\,  \mu_i\ge 1/2,\\
		d( \mu_i,  \underline{a} ) &\le \frac{  ( \sigma_i^2 - \bar{\sigma}^2 )^2 }{ \underline{a} ( 1/4 - \bar{\sigma}^2)  }  \quad \forall\,  \mu_i < 1/2,
		\nonumber \\
		d( \mu_1,  \underline{a} )& \le \frac{ ( \mu_1- \underline{a} )^2  \cdot (1- \mu_1 -\underline{a} )^2 }{  \underline{a} (1- \mu_1-\underline{a} )^2 }  
		= \frac{  ( \sigma_1^2 - \bar{\sigma}^2 )^2 }{ \underline{a} (1- \mu_1 -\underline{a} )^2 }
		\\
		&\le  \frac{  ( \sigma_1^2 - \bar{\sigma}^2 )^2 }{ \underline{a} ( 1- 2\underline{a} )^2 }
		= \frac{  ( \sigma_1^2 - \bar{\sigma}^2 )^2 }{ 4\underline{a} ( 1/2 - \underline{a} )^2 }.  \nonumber 
	\end{align}  
	Note that $i^{\star}=1$ and $\mu_{i^{\star}} = \mu_1 < \underline{a}$.
	For $\mu_i < \mu_1$, arm $i$ is feasible and we also have
	\begin{align*}
		d( \mu_i,  \underline{a} ) \le \frac{  (\mu_i - \mu_1)^2 }{ \underline{a}}
		= \frac{  (\mu_i - \mu_{i^{\star}})^2 }{ \underline{a}}.
	\end{align*}
	Therefore,
	\begin{align*}
		\mathbb{E} [ \tau  ]
		&
		\ge  
		\ln \bigg( \frac{1}{2.4\delta} \bigg) \cdot  \frac{ 4\underline{a} ( 1/2 - \underline{a} )^2 }{  ( \sigma_1^2 - \bar{\sigma}^2 )^2 }
		+
		\underline{a} \ln \bigg( \frac{1}{2.4\delta} \bigg) \\
		&\quad
		\cdot  
		\Bigg( \sum_{ i \in \mathcal{F}\cap\mathcal{S}  } 
		\frac{ 1 }{ (\mu_i - \mu_{i^{\star}})^2 } 
		+
		\sum_{ i\notin \mathcal{F} } \frac{  1/4 - \bar{\sigma}^2   }{  ( \sigma_i^2 - \bar{\sigma}^2 )^2 }   \Bigg). 
	\end{align*}  
	By definition, $\bar{\mathcal{F}}^c\cap\mathcal{S}$ is empty and hence $\Delta_{1} = \min_{i \in \mathcal{S}} \Delta_i = \min_{i \in {\mathcal{F}\cap\mathcal{S}}} \Delta_i$.
	Combined with \eqref{simplefact}, the above analysis yields
	\begin{align*}
		\mathbb{E} [ \tau  ]
		&
		\ge   
		\underline{a} \ln \bigg( \frac{1}{2.4\delta} \bigg) \cdot  
		\Bigg(  
		\frac{ 4 ( 1/4 - \bar{\sigma}^2) }{  ( \Delta_1^{\mathrm{v}} )^2 }
		+
		\sum_{   i \in \mathcal{F}\cap\mathcal{S}  } 
		\frac{ 1 }{ \Delta_i^2 } 
		\\*
		&\quad
		+
		\sum_{ i\in\bar{\mathcal{F}}^c\cap \mathcal{R}} \frac{  1/4 - \bar{\sigma}^2   }{  ( \Delta_i^{\mathrm{v}})^2 }   \Bigg)
		\\&
		\ge   
		\underline{a}  \ln \bigg( \frac{1}{2.4\delta} \bigg) \cdot  
		\Bigg(  
		\frac{ 4 ( 1/4 - \bar{\sigma}^2) }{  ( \Delta_1^{\mathrm{v}} )^2 }
		+
		\sum_{    i \in \mathcal{F}  } 
		\frac{ 1/2 }{ \Delta_i^2 } 
		\\*
		&\quad
		+
		\sum_{ i\in \bar{\mathcal{F}}^c\cap \mathcal{R} } \frac{  1/4 - \bar{\sigma}^2   }{  ( \Delta_i^{\mathrm{v}})^2 }   \Bigg)
		\\&
		\ge   
		\underline{a}\ln \bigg( \frac{1}{2.4\delta} \bigg) \cdot  
		\Bigg(  
		\frac{ \min\{ 4 ( 1/4 - \bar{\sigma}^2), 1/8 \} }{  \max\{ ( \Delta_1^{\mathrm{v}} )^2,  (\Delta_1/2)^2 \} }
		\\
		&\quad
		+
		\sum_{i \in \mathcal{F}\cap\mathcal{S}  } 
		\frac{ 1/8  }{ ( \Delta_i/2)^2 } 
		+
		\sum_{ i\in\bar{\mathcal{F}}^c\cap \mathcal{R}} \frac{  1/4 - \bar{\sigma}^2   }{  ( \Delta_i^{\mathrm{v}})^2 }   \Bigg)
		\\& 
		\ge  
		H_{\mathrm{VA}}
		\ln \bigg( \frac{1}{2.4\delta} \bigg) \cdot \underline{a} 
		\cdot
		\min\bigg\{  \frac{1}{8}, \frac{1}{4} - \bar{\sigma}^2  \bigg\}.
	\end{align*}

	\noindent
	{\bf Step 4: Conclusion.} 
	
	\noindent
	\underline{\bf Case (i)}: When $\underline{a}<\mu_N\le \ldots \leq\mu_2\leq\mu_1<\bar{a} $, we have
	\begin{align*}
		\mathbb{E} [ \tau ]
		&
		\ge    
		H_{\mathrm{VA}} \ln \bigg( \frac{1}{2.4\delta} \bigg)  \cdot  \underline{a}(1/4 - \bar{\sigma}^2 ) 
		.
	\end{align*}
	
	\noindent
	\underline{\bf Case (ii)}: When $0<\mu_N\le \ldots \leq\mu_2<\mu_1<1$, and $\bar{a}<\mu_1 $, we have  
	\begin{align*}
		\mathbb{E} [ \tau  ]
		& 
		\ge
		H_{\mathrm{VA}} \ln \bigg( \frac{1}{2.4\delta} \bigg) \cdot  \min \bigg\{   4\underline{a} \bigg(\frac{1}{4} - \bar{\sigma}^2\bigg) ,  \frac{ 1-\mu_1}{ 8 }
		\bigg\}  \vphantom{ \Bigg\} }.
	\end{align*}

	\noindent
	\underline{\bf Case (iii)}: When
	$ \mu_i<\bar{a} $ for all $i\in[N]$, $ \{ i\in[N]: \mu_i =\underline{a} \} =\emptyset$, $ \{ i\in[N]: \mu_i <\underline{a} \} \ne \emptyset$, and arm $1=\mathop{\rm argmax} \{ \mu_i : \mu_i < \underline{a}\}     $, we have   
	\begin{align*}
		\mathbb{E} [ \tau  ]
		& 
		\ge  
		H_{\mathrm{VA}}
		\ln \bigg( \frac{1}{2.4\delta} \bigg) \cdot  \underline{a} 
		\cdot
		\min\bigg\{  \frac{1}{8}, \frac{1}{4} - \bar{\sigma}^2  \bigg\}.
	\end{align*}   
	In either case, we have  
	\begin{align*}
		\mathbb{E} [ \tau  ]
		& 
		\ge  
		H_{\mathrm{VA}}
		\ln \bigg( \frac{1}{2.4\delta} \bigg)  
		\cdot
		\min \bigg\{  \underline{a} \Big(\frac{1}{4} - \bar{\sigma}^2 \Big),\ \frac{\underline{a}}{ 8} ,\ \frac{ 1-\mu_{i^{\star}}}{ 8 }
		\bigg\} ,
	\end{align*}     
	which completes the proof of the lower bound.
\end{proof}

\begin{proof}[Proof of Corollary~\ref{cor:almost_opt}]
	The only statement that requires proof is the fact that the average sample complexity of VA-LUCB $\mathbb{E}[\tau^{\text{VA-LUCB}}]$ is  $O\big(H_\mathrm{VA}\ln (H_\mathrm{VA}/{\delta})\big)$. From Lemma~\ref{stop}, the expected time steps can be upper bounded by\footnote{The power of $t$ in the summation is $2$. The number $2$ can be traced back to choice/design of the power of $t$ (which is $4$) in the confidence radii $\alpha$ and $\beta$ in~\eqref{radius}. In fact, the power $4$ in \eqref{radius} can be replaced by any number slightly greater than $3$ so that the infinite summation in~\eqref{eqn:summation} still converges. By doing so, the sample complexity of the upper bound remains unchanged, but the empirical performance will be improved. One can replace the $4$ with $2$ and the empirical performance will be improved significantly. However, this comes at the expense of the loss of tightness in the expected sample complexity result (cf.\ Corollary~\ref{cor:almost_opt}).  }
	\begin{align}
		\hspace{-.1in}t^\star\!+\!\sum_{t=t^{\star}+1}^\infty \frac{5\delta}{t^2}\leq t^\star\!+\!\frac{5\delta}{t^\star}\leq 152\, H_\mathrm{VA}\ln\frac{H_\mathrm{VA}}{\delta}\!+\! 1. \label{eqn:summation}
	\end{align}
	Note the average sample complexity is at most twice the number of time steps, thus \begin{align}
		\tau_\delta^{\star}\le\mathbb{E}[\tau^{\text{VA-LUCB}}] &\le 304\, H_{\mathrm{VA}}\ln\frac{H_\mathrm{VA}}{\delta}+2\\
		&
		=O\left(H_{\mathrm{VA}}\ln\frac{H_\mathrm{VA}}{\delta}\right),
	\end{align} which completes the proof. 
\end{proof}

\section{Experimental Details}\label{expdetail}
\subsection{Experiment Design for the Fourth Term of $ H_\mathrm{VA} $}\label{design of fourth}
We complete the description of the experiment design in Section~\ref{experiments}, i.e., 
for the fourth term  $ \max\{\Delta_{i}/{2},\Delta_{i}^{\mathrm{v}}\}^{-2}  $,\newline
(a). Under the condition that $ \Delta_{i}/{2}\leq \Delta_{i}^{\mathrm{v}} $, when $ \Delta_{i^\star} (\geq 2\Delta_{i^\star}^{\mathrm{v}})$ and $  \Delta_{i}  $ for all $i\in\bar{\mathcal{F}}^c\cap\mathcal{S} $ increase, $ H_{\mathrm{VA}} $ and the sample complexity will stay the same.\newline
(b). Under the condition that $ \Delta_{i}/{2}\leq \Delta_{i}^{\mathrm{v}} $,when $ \Delta_{i}^{\mathrm{v}} $ increases, $ H_{\mathrm{VA}} $ and the sample complexity will decrease.\newline
(c). Under the condition that $ \Delta_{i}/{2}\geq \Delta_{i}^{\mathrm{v}} $, as $ \Delta_{i^\star} (\leq 2\Delta_{i^\star}^{\mathrm{v}})$ and $  \Delta_{i}  $ for all $i\in\bar{\mathcal{F}}^c\cap\mathcal{S} $ increase, $ H_{\mathrm{VA}} $  and the sample complexity will decrease.\newline
(d). Under the condition that $ \Delta_{i}/{2}\geq \Delta_{i}^{\mathrm{v}} $, as $ \Delta_{i}^{\mathrm{v}} $ increase, $ H_{\mathrm{VA}} $ and the sample complexity will stay unchanged.

\subsection{Specific Parameters for Each Instance}\label{specific para}
There are $4$ cases for the first and fourth term in $H_\mathrm{VA}$ respectively, as well as one case for the second and third term respectively. In each case, there are $ 11 $ instances, indexed by $ j\in\{0,1,2\ldots,10\} $. Each instance consists of  $N=20$ arms, including $ i^\star $ (if it exists), $ i^{\star\star} $ (if it exists), and the other $18$ arms with exactly the same parameters. Beta distribution are adopted as the reward distributions for the arms because they are supported on $[0,1]$ and due to  their flexibility in assigning the expectations and the variances for the arms. To be more specific,
given a Beta distribution $ B(\alpha,\beta) $ with expectation $ a $ and variance $ b $, where $ \alpha,\beta>0,a(1-a)>b $, the four parameters are related according to the following equations:
\begin{equation}\label{beta}
	\begin{aligned}
		&a=\frac{\alpha}{\alpha+\beta},\, &b=&\frac{\alpha\beta}{(\alpha+\beta)^2(\alpha+\beta+1)}\\
		&\alpha=\frac{a^2(1-a)-ab}{b},\, &\beta=&\frac{1-a}{a}\alpha=\frac{a(1-a)^2-(1-a)b}{b}.
	\end{aligned}
\end{equation}
Thus, when the expectation and the variance are given, the two parameters of the Beta distribution $\alpha$ and $\beta$ can be readily computed.
To demonstrate the effects on the sample complexity of the four terms more clearly, each instance is designed to consist of the arms which are associated with the term to be examined.
The parameters for each instance indexed by $ j\in\{0,1\ldots,10\} $ in each case are described below. Recall the definition of $H_{\mathrm{VA}}$ in \eqref{eqn:hardness}.
Since the {first term} $ \min\{{\Delta_{i^\star}}/{2},\Delta_{i^\star}^{\mathrm{v}}\}^{-2}  $ involves  arm $ i^\star $, arm $ i^{\star\star} $, Case 1 is comprised of the best feasible arm $i^\star$ and feasible and subptimal arms (including $i^{\star\star}$, i.e., $H_{\mathrm{VA}}= \min\{\Delta_{i^\star}/{2},\Delta_{i^\star}^{\mathrm{v}}\}^{-2} 
	+\sum_{i\in\mathcal{F}\cap\mathcal{S}}(\Delta_i/2)^{-2}
	$. We control the mean gap of $i^\star$ by changing the mean of $i^{\star\star}$ and the variance gap by changing the variance of $i^\star$.

\noindent
\underline{\bf Case  1(a)}: $ {\Delta_{i^\star}}/{2}\leq \Delta_{i^\star}^{\mathrm{v}} $ and 
\begin{align}
	H_\mathrm{VA}&
	= \frac{1}{\min\{ {\frac{\Delta_{i^\star}}{2}},\Delta_{i^\star}^{\mathrm{v}}\}^2 }
	\!+\!\left(\frac{\Delta_{i^{\star\star}}}{2}\right)^{-2} \!+\!\sum_{i\in\mathcal{F}\cap\mathcal{S}\backslash\{i^{\star\star}\}}\left(\frac{\Delta_i}{2}\right)^{-2}
	\\
	&= \left(\frac{\Delta_{i^\star}}{2} \right)^{-2}+\left(\frac{\Delta_{i^{\star\star}}}{2}\right)^{-2}+\sum_{i\in\mathcal{F}\cap\mathcal{S}\backslash\{i^{\star\star}\}}\left(\frac{\Delta_i}{2}\right)^{-2}
\end{align}
	The parameters that are varied are $ \Delta_{{i^\star}} $ and $ \Delta_{i^{\star\star}} $. See Table~\ref{Case1a} for details.

\noindent
\underline{\bf Case  1(b)}: $ {\Delta_{i^\star}}/{2}\leq \Delta_{i^\star}^{\mathrm{v}} $ and 
\begin{align}
	H_\mathrm{VA}&= \frac{1}{\min\{\frac{\Delta_{i^\star}}{2},\Delta_{i^\star}^{\mathrm{v}}\}^2 }
	\!+\!\left(\frac{\Delta_{i^{\star\star}}}{2}\right)^{-2}
	\!+\!\sum_{i\in\mathcal{F}\cap\mathcal{S}\backslash\{i^{\star\star}\}}\left(\frac{\Delta_i}{2}\right)^{-2}\\
	&
	= \left(\frac{\Delta_{i^\star}}{2} \right)^{-2}+\left(\frac{\Delta_{i^{\star\star}}}{2}\right)^{-2}+\sum_{i\in\mathcal{F}\cap\mathcal{S}\backslash\{i^{\star\star}\}}\left(\frac{\Delta_i}{2}\right)^{-2} 
\end{align}
	The parameter that is varied is $ \Delta_{{i^\star}}^{\mathrm{v}} $. See Table~\ref{Case1b} for details.

\begin{table*}[t]
	\begin{minipage}{0.49\linewidth}
		\centering
		\begin{tabular}{cccccc}
			\hline
			\multicolumn{6}{c}{$ \bar{\sigma}^2=0.25$ }      \\ \hline
			\multicolumn{1}{c|}{$ \mu_{i^\star} $} & \multicolumn{1}{c|}{$ 0.7 $} & \multicolumn{1}{c|}{$ \mu_{i^{\star\star}} $} & \multicolumn{1}{c|}{$ \mu_{i^\star} - \Delta_{i^\star}$} & \multicolumn{1}{c|}{$ \mu_{i} $} & $ 0.2  $\\ \hline
			\multicolumn{1}{c|}{$ \Delta_{i^\star} $} & \multicolumn{1}{c|}{$ 0.01\times1.2^{j} $}  & \multicolumn{1}{c|}{$ \Delta_{i^{\star\star}} $}  & \multicolumn{1}{c|}{ $ \Delta_{i^\star} $}  & \multicolumn{1}{c|}{$ \Delta_{i} $}  &  $ 0.5 $ \\ \hline
			\multicolumn{1}{c|}{$ \sigma_{i^\star}^2 $} & \multicolumn{1}{c|}{$ 0.09 $}  & \multicolumn{1}{c|}{$ \sigma_{i^{\star\star}}^2 $}  & \multicolumn{1}{c|}{$ 0.09 $}  & \multicolumn{1}{c|}{$ \sigma_{i}^2 $}  & $ 0.09 $  \\ \hline
			\multicolumn{1}{c|}{$ \Delta_{i^\star}^{\mathrm{v}} $} & \multicolumn{1}{c|}{$ 0.16 $}  & \multicolumn{1}{c|}{$ \Delta_{i^{\star\star}}^{\mathrm{v}} $}  & \multicolumn{1}{c|}{$ 0.16 $}  & \multicolumn{1}{c|}{$ \Delta_{i}^{\mathrm{v}} $}  &  $ 0.16 $ \\ \hline
		\end{tabular}
		\caption{Case 1(a)}
		\label{Case1a}
	\end{minipage}
	\
	\begin{minipage}{0.49\linewidth}
		\centering
		\begin{tabular}{cccccc}
			\hline
			\multicolumn{6}{c}{$ \bar{\sigma}^2=0.25,N=20 $ }      \\ \hline
			\multicolumn{1}{c|}{$ \mu_{i^\star} $} & \multicolumn{1}{c|}{$ 0.55 $} & \multicolumn{1}{c|}{$ \mu_{i^{\star\star}} $} & \multicolumn{1}{c|}{$0.53$} & \multicolumn{1}{c|}{$ \mu_{i} $} &$  0.15  $\\ \hline
			\multicolumn{1}{c|}{$ \Delta_{i^\star} $} & \multicolumn{1}{c|}{$0.02  $}  & \multicolumn{1}{c|}{$ \Delta_{i^{\star\star}} $}  & \multicolumn{1}{c|}{ $ 0.02 $}  & \multicolumn{1}{c|}{$ \Delta_{i} $}  &  $ 0.4 $ \\ \hline
			\multicolumn{1}{c|}{$ \sigma_{i^\star}^2 $} & \multicolumn{1}{c|}{$ \bar{\sigma}^2-\Delta_{i^\star}^{\mathrm{v}}$}  & \multicolumn{1}{c|}{$ \sigma_{i^{\star\star}}^2 $}  & \multicolumn{1}{c|}{$ 0.09 $}  & \multicolumn{1}{c|}{$ \sigma_{i}^2 $}  & $ 0.09 $  \\ \hline
			\multicolumn{1}{c|}{$ \Delta_{i^\star}^{\mathrm{v}} $} & \multicolumn{1}{c|}{$ 0.01\times 1.2^{j}$}  & \multicolumn{1}{c|}{$ \Delta_{i^{\star\star}}^{\mathrm{v}} $}  & \multicolumn{1}{c|}{$ 0.16 $}  & \multicolumn{1}{c|}{$ \Delta_{i}^{\mathrm{v}} $}  &  $ 0.16 $ \\ \hline
		\end{tabular}
		\caption{Case 1(b)}
		\label{Case1b}
	\end{minipage}
\end{table*}

\noindent
\underline{\bf Case  1(c)}: $ {\Delta_{i^\star}}/{2}\geq \Delta_{i^\star}^{\mathrm{v}} $ and 
\begin{align}
	H_\mathrm{VA}&= 
	\frac{1}{\min\{\frac{\Delta_{i^\star}}{2},\Delta_{i^\star}^{\mathrm{v}}\}^2} 
	\!+\!\left(\frac{\Delta_{i^{\star\star}}}{2}\right)^{-2} 
	\!+\!\sum_{i\in\mathcal{F}\cap\mathcal{S}\backslash\{i^{\star\star}\}}\left(\frac{\Delta_i}{2}\right)^{-2}
	\\
	&= (\Delta_{i^\star}^{\mathrm{v}})^{-2}+\left(\frac{\Delta_{i^{\star\star}}}{2}\right)^{-2}+\sum_{i\in\mathcal{F}\cap\mathcal{S}\backslash\{i^{\star\star}\}}\left(\frac{\Delta_i}{2}\right)^{-2}
\end{align}
	See Table~\ref{Case1c} for details.

\noindent
\underline{\bf Case  1(d)}: $ {\Delta_{i^\star}}/{2}\geq \Delta_{i^\star}^{\mathrm{v}} $ and 
\begin{align}
	H_\mathrm{VA}&= \frac{1}{\min\{\frac{\Delta_{i^\star}}{2},\Delta_{i^\star}^{\mathrm{v}}\}^2 }
	\!+\!\left(\frac{\Delta_{i^{\star\star}}}{2}\right)^{-2}\!+\!\sum_{i\in\mathcal{F}\cap\mathcal{S}\backslash\{i^{\star\star}\}}\left(\frac{\Delta_i}{2}\right)^{-2}\\
	&= \left(\Delta_{i^\star}^{\mathrm{v}} \right)^{-2}+\left(\frac{\Delta_{i^{\star\star}}}{2}\right)^{-2}+\sum_{i\in\mathcal{F}\cap\mathcal{S}\backslash\{i^{\star\star}\}}\left(\frac{\Delta_i}{2}\right)^{-2}
\end{align} 
	The parameters that are varied are $ \Delta_{{i^\star}} $ and $ \Delta_{i^{\star\star}} $.
	See Table~\ref{Case1d} for details.

\begin{table*}[t]
	\begin{minipage}{0.49\linewidth}
		\centering
		\begin{tabular}{cccccc}
			\hline
			\multicolumn{6}{c}{$ \bar{\sigma}^2=0.25 ,N=20$ }      \\ \hline
			\multicolumn{1}{c|}{$ \mu_{i^\star} $} & \multicolumn{1}{c|}{$ 0.55 $} & \multicolumn{1}{c|}{$ \mu_{i^{\star\star}} $} & \multicolumn{1}{c|}{$ 0.15$} & \multicolumn{1}{c|}{$ \mu_{i} $} & $ 0.15  $\\ \hline
			\multicolumn{1}{c|}{$ \Delta_{i^\star} $} & \multicolumn{1}{c|}{$0.4 $}  & \multicolumn{1}{c|}{$ \Delta_{i^{\star\star}} $}  & \multicolumn{1}{c|}{ $0.4$}  & \multicolumn{1}{c|}{$ \Delta_{i} $}  &  $ 0.4 $ \\ \hline
			\multicolumn{1}{c|}{$ \sigma_{i^\star}^2 $} & \multicolumn{1}{c|}{$ \bar{\sigma}^2-\Delta_{i^\star}^{\mathrm{v}}$}  & \multicolumn{1}{c|}{$ \sigma_{i^{\star\star}}^2 $}  & \multicolumn{1}{c|}{$ 0.09 $}  & \multicolumn{1}{c|}{$ \sigma_{i}^2 $}  & $ 0.09 $  \\ \hline
			\multicolumn{1}{c|}{$ \Delta_{i^\star}^{\mathrm{v}} $} & \multicolumn{1}{c|}{$ 0.01\times 1.2^{j} $}  & \multicolumn{1}{c|}{$ \Delta_{i^{\star\star}}^{\mathrm{v}} $}  & \multicolumn{1}{c|}{$ 0.16 $}  & \multicolumn{1}{c|}{$ \Delta_{i}^{\mathrm{v}} $}  &  $ 0.16 $ \\ \hline
		\end{tabular}
		\caption{Case 1(c)}
		\label{Case1c}
	\end{minipage}
	\
	\begin{minipage}{0.49\linewidth}
		\centering
		\begin{tabular}{cccccc}
			\hline
			\multicolumn{6}{c}{$ \bar{\sigma}^2=0.04$ }      \\ \hline
			\multicolumn{1}{c|}{$ \mu_{i^\star} $} & \multicolumn{1}{c|}{$ 0.7 $} & \multicolumn{1}{c|}{$ \mu_{i^{\star\star}} $} & \multicolumn{1}{c|}{$ \mu_{i^\star} - \Delta_{i^\star}$} & \multicolumn{1}{c|}{$ \mu_{i} $} & $ 0.3 $ \\ \hline
			\multicolumn{1}{c|}{$ \Delta_{i^\star} $} & \multicolumn{1}{c|}{$ 0.02\times 1.1^{j} $}  & \multicolumn{1}{c|}{$ \Delta_{i^{\star\star}} $}  & \multicolumn{1}{c|}{ $ \Delta_{i^\star} $}  & \multicolumn{1}{c|}{$ \Delta_{i} $}  &  $ 0.4 $ \\ \hline
			\multicolumn{1}{c|}{$ \sigma_{i^\star}^2 $} & \multicolumn{1}{c|}{$ 0.03 $}  & \multicolumn{1}{c|}{$ \sigma_{i^{\star\star}}^2 $}  & \multicolumn{1}{c|}{$ 0.03 $}  & \multicolumn{1}{c|}{$ \sigma_{i}^2 $}  & $ 0.03 $  \\ \hline
			\multicolumn{1}{c|}{$ \Delta_{i^\star}^{\mathrm{v}} $} & \multicolumn{1}{c|}{$ 0.01 $}  & \multicolumn{1}{c|}{$ \Delta_{i^{\star\star}}^{\mathrm{v}} $}  & \multicolumn{1}{c|}{$ 0.01 $}  & \multicolumn{1}{c|}{$ \Delta_{i}^{\mathrm{v}} $}  &  $ 0.01 $ \\ \hline
		\end{tabular}
		\caption{Case 1(d)}
		\label{Case1d}
	\end{minipage}
\end{table*}

\noindent
\underline{\bf Case  2}: To see the effect of  the \textbf{second term} $\sum_{i\in\mathcal{F}\cap\mathcal{S}}({\Delta_i}/{2})^{-2}$, Case 2 is compriesed of  the best feasible arm and arms in $ \mathcal{F}\cap\mathcal{S} $, including $i^{\star\star}$. We set $ {\Delta_{i^\star}}/{2}\leq\Delta_{i^\star}^{\mathrm{v}}$ and $\bar{\mathcal{F}}^c=\emptyset$. Therefore, 
\begin{align}
	H_\mathrm{VA}&=\frac{1}{\min\{\frac{\Delta_{i^\star}}{2},\Delta_{i^\star}^{\mathrm{v}}\}^2 }
	\!+\!\left(\frac{\Delta_{i^{\star\star}}}{2}\right)^{-2} 
	\!+\!\sum_{i\in\mathcal{F}\cap\mathcal{S}\backslash\{i^{\star\star}\}}\left(\frac{\Delta_i}{2}\right)^{-2}
	\\
	&
	= \left(\frac{\Delta_{i^\star}}{2}\right)^{-2} +\sum_{i\in\mathcal{F}\cap\mathcal{S}}\left(\frac{\Delta_i}{2}\right)^{-2}
\end{align}
 See Table~\ref{Case2} for details.

\noindent
\underline{\bf Case  3}:
As for the \textbf{third term }$\sum_{i\in\bar{\mathcal{F}}^c\cap\mathcal{R}}(\Delta_{i}^{\mathrm{v}})^{-2}$, $ \bar{\mathcal{F}}^c\cap\mathcal{R}$ is nonempty. In Case 3, we set $ \mathcal{F}=\emptyset$, i.e. it is an infeasible instance and there are $ 20 $ infeasible arms with the same parameters. Hence, $ H_\mathrm{VA}=\sum_{i\in\bar{\mathcal{F}}^c\cap\mathcal{R}}(\Delta_{i}^{\mathrm{v}})^{-2} $. See Table~\ref{Case3} for details.

\begin{table*}[t]
	\begin{minipage}{0.49\linewidth}
		\centering
		\begin{tabular}{cccccc}
			\hline
			\multicolumn{6}{c}{$ \bar{\sigma}^2=0.25,N=20$ }      \\ \hline
			\multicolumn{1}{c|}{$ \mu_{i^\star} $} & \multicolumn{1}{c|}{$ 0.7 $} & \multicolumn{1}{c|}{$ \mu_{i^{\star\star}} $} & \multicolumn{1}{c|}{$ \mu_{i^\star} - \Delta_{i^\star}$} & \multicolumn{1}{c|}{$ \mu_{i} $} & $ \mu_{i^\star} -\Delta_i $ \\ \hline
			\multicolumn{1}{c|}{$ \Delta_{i^\star} $} & \multicolumn{1}{c|}{$ \Delta_i$}  & \multicolumn{1}{c|}{$ \Delta_{i^{\star\star}} $}  & \multicolumn{1}{c|}{ $ \Delta_{i^\star} $}  & \multicolumn{1}{c|}{$ \Delta_{i} $}  & $ 0.02\times 1.2^{j} $\\ \hline
			\multicolumn{1}{c|}{$ \sigma_{i^\star}^2 $} & \multicolumn{1}{c|}{$ 0.09 $}  & \multicolumn{1}{c|}{$ \sigma_{i^{\star\star}}^2 $}  & \multicolumn{1}{c|}{$ 0.09 $}  & \multicolumn{1}{c|}{$ \sigma_{i}^2 $}  & $ 0.09 $  \\ \hline
			\multicolumn{1}{c|}{$ \Delta_{i^\star}^{\mathrm{v}} $} & \multicolumn{1}{c|}{$ 0.16 $}  & \multicolumn{1}{c|}{$ \Delta_{i^{\star\star}}^{\mathrm{v}} $}  & \multicolumn{1}{c|}{$ 0.16 $}  & \multicolumn{1}{c|}{$ \Delta_{i}^{\mathrm{v}} $}  &  $ 0.16 $ \\ \hline
		\end{tabular}
		\caption{Case 2}
		\label{Case2}
	\end{minipage}
	\
	\begin{minipage}{0.49\linewidth}
		\centering
		\begin{tabular}{cccccc}
			\hline
			\multicolumn{6}{c}{$ \bar{\sigma}^2=0.04,N=20$ }      \\ \hline
			\multicolumn{1}{c|}{$ \mu_{i^\star} $} & \multicolumn{1}{c|}{NA} & \multicolumn{1}{c|}{$ \mu_{i^{\star\star}} $} & \multicolumn{1}{c|}{NA} & \multicolumn{1}{c|}{$ \mu_{i} $} & $0.55 $ \\ \hline
			\multicolumn{1}{c|}{$ \Delta_{i^\star} $} & \multicolumn{1}{c|}{$+\infty $}  & \multicolumn{1}{c|}{$ \Delta_{i^{\star\star}} $}  & \multicolumn{1}{c|}{ $0$}  & \multicolumn{1}{c|}{$ \Delta_{i} $}  &  $ 0 $ \\ \hline
			\multicolumn{1}{c|}{$ \sigma_{i^\star}^2 $} & \multicolumn{1}{c|}{NA}  & \multicolumn{1}{c|}{$ \sigma_{i^{\star\star}}^2 $}  & \multicolumn{1}{c|}{NA}  & \multicolumn{1}{c|}{$ \sigma_{i}^2 $}  & $ \bar{\sigma}^2+\Delta_i^{\mathrm{v}}$  \\ \hline
			\multicolumn{1}{c|}{$ \Delta_{i^\star}^{\mathrm{v}} $} & \multicolumn{1}{c|}{NA}  & \multicolumn{1}{c|}{$ \Delta_{i^{\star\star}}^{\mathrm{v}} $}  & \multicolumn{1}{c|}{NA}  & \multicolumn{1}{c|}{$ \Delta_{i}^{\mathrm{v}} $}  &$ 0.01\times 1.2^{j} $ \\ \hline
		\end{tabular}
		\caption{Case 3}
		\label{Case3}
	\end{minipage}
\end{table*}

For the \textbf{fourth term} $\sum_{i\in\bar{\mathcal{F}}^c\cap\mathcal{S}}\max\{{\Delta_{i}}/{2},\Delta_{i}^{\mathrm{v}}\}^{-2} $, the  arms that are involved are  $ i^\star $ and the both infeasible and suboptimal arms.  In Case 4, $ i^\star $ is designed to be the unique feasible arm 
and the rest of the arms are set to be infeasible and  suboptimal arms with the same parameters, i.e., $H_{\mathrm{VA}} = \min\{{\Delta_{i^\star}}/{2},\Delta_{i^\star}^{\mathrm{v}}\}^{-2}+
	\sum_{i\in\bar{\mathcal{F}}^c\cap\mathcal{S}}\max\{{\Delta_i}/{2},\Delta_i^{\mathrm{v}}\}^{-2}.$  In Case 4(c) we set $ {\Delta_{i^\star}}/{2}\leq \Delta_{i^\star}^{\mathrm{v}} $ while in other cases $ {\Delta_{i^\star}}/{2}\geq \Delta_{i^\star}^{\mathrm{v}} $.

\noindent
\underline{\bf Case  4(a)}: $  {\Delta_{i}}/{2}\leq \Delta_{i}^{\mathrm{v}}  $ and
\begin{align}
	H_{\mathrm{VA}} &= \frac{1}{\min\{\frac{\Delta_{i^\star}}{2},\Delta_{i^\star}^{\mathrm{v}}\}^2} +
	\sum_{i\in\bar{\mathcal{F}}^c\cap\mathcal{S}}\frac{1}{\max\{\frac{\Delta_i}{2},\Delta_i^{\mathrm{v}}\}^2}\\*
	&
	=(\Delta_{i^\star}^{\mathrm{v}} )^{-2}+\sum_{i\in\bar{\mathcal{F}}^c\cap\mathcal{S}}(\Delta_{i}^{\mathrm{v}})^{-2} .
\end{align}
     The parameters that are varied are $ \Delta_i $ for all $ i\in[N]$. See Table~\ref{Case4a} for details.

\noindent
\underline{\bf Case  4(b)}: $  {\Delta_{i}}/{2}\leq \Delta_{i}^{\mathrm{v}}  $ and
\begin{align}
	H_{\mathrm{VA}} &= \frac{1}{\min\{\frac{\Delta_{i^\star}}{2},\Delta_{i^\star}^{\mathrm{v}}\}^2 }+
	\sum_{i\in\bar{\mathcal{F}}^c\cap\mathcal{S}}\frac{1}{\max\{\frac{\Delta_i}{2},\Delta_i^{\mathrm{v}}\}^2}\\
	&
	=\left(\Delta_{i^\star}^{\mathrm{v}} \right)^{-2}+\sum_{i\in\bar{\mathcal{F}}^c\cap\mathcal{S}}\left(\Delta_{i}^{\mathrm{v}}\right)^{-2} .
\end{align}
	 The parameters that are varied are $ \Delta_i^{\mathrm{v}} $ for all $ i\in\bar{\mathcal{F}}^c\cap\mathcal{S}$. See Table~\ref{Case4b} for details.

\begin{table*}[t]
	\begin{minipage}{0.49\linewidth}
		\centering
		\begin{tabular}{cccccc}
			\hline
			\multicolumn{6}{c}{$ \bar{\sigma}^2=0.04 ,N=20$ }      \\ \hline
			\multicolumn{1}{c|}{$ \mu_{i^\star} $} & \multicolumn{1}{c|}{$ 0.7 $} & \multicolumn{1}{c|}{$ \mu_{i^{\star\star}} $} & \multicolumn{1}{c|}{$ \mu_{i^\star} - \Delta_{i^\star}$} & \multicolumn{1}{c|}{$ \mu_{i} $} & $ \mu_{i^\star} - \Delta_{i}  $\\ \hline
			\multicolumn{1}{c|}{$ \Delta_{i^\star} $} & \multicolumn{1}{c|}{$  \Delta_i $}  & \multicolumn{1}{c|}{$ \Delta_{i^{\star\star}} $}  & \multicolumn{1}{c|}{ $ \Delta_{i^\star} $}  & \multicolumn{1}{c|}{$ \Delta_{i} $}  & $ 0.02\times 1.2^{j} $ \\ \hline
			\multicolumn{1}{c|}{$ \sigma_{i^\star}^2 $} & \multicolumn{1}{c|}{$ 0.03 $}  & \multicolumn{1}{c|}{$ \sigma_{i^{\star\star}}^2 $}  & \multicolumn{1}{c|}{$ 0.2 $}  & \multicolumn{1}{c|}{$ \sigma_{i}^2 $}  & $ 0.2 $  \\ \hline
			\multicolumn{1}{c|}{$ \Delta_{i^\star}^{\mathrm{v}} $} & \multicolumn{1}{c|}{$ 0.01 $}  & \multicolumn{1}{c|}{$ \Delta_{i^{\star\star}}^{\mathrm{v}} $}  & \multicolumn{1}{c|}{$ 0.16 $}  & \multicolumn{1}{c|}{$ \Delta_{i}^{\mathrm{v}} $}  &  $ 0.16 $ \\ \hline
		\end{tabular}
		\caption{Case 4(a)}
		\label{Case4a}
	\end{minipage}
	\
	\begin{minipage}{0.49\linewidth}
		\centering
		\begin{tabular}{cccccc}
			\hline
			\multicolumn{6}{c}{$ \bar{\sigma}^2=0.04,N=20 $ }      \\ \hline
			\multicolumn{1}{c|}{$ \mu_{i^\star} $} & \multicolumn{1}{c|}{$ 0.55 $} & \multicolumn{1}{c|}{$ \mu_{i^{\star\star}} $} & \multicolumn{1}{c|}{$ 0.53$} & \multicolumn{1}{c|}{$ \mu_{i} $} & $ 0.53 $\\ \hline
			\multicolumn{1}{c|}{$ \Delta_{i^\star} $} & \multicolumn{1}{c|}{$  0.02$}  & \multicolumn{1}{c|}{$ \Delta_{i^{\star\star}} $}  & \multicolumn{1}{c|}{ $ 0.02$}  & \multicolumn{1}{c|}{$ \Delta_{i} $}  &  $ 0.02$ \\ \hline
			\multicolumn{1}{c|}{$ \sigma_{i^\star}^2 $} & \multicolumn{1}{c|}{$ 0.03 $}  & \multicolumn{1}{c|}{$ \sigma_{i^{\star\star}}^2 $}  & \multicolumn{1}{c|}{$  \sigma_{i}^2$}  & \multicolumn{1}{c|}{$ \sigma_{i}^2 $}  & $  \bar{\sigma}^2+\Delta_{i}^{\mathrm{v}} $  \\ \hline
			\multicolumn{1}{c|}{$ \Delta_{i^\star}^{\mathrm{v}} $} & \multicolumn{1}{c|}{$ 0.01 $}  & \multicolumn{1}{c|}{$ \Delta_{i^{\star\star}}^{\mathrm{v}} $}  & \multicolumn{1}{c|}{$  \Delta_{i}^{\mathrm{v}}$}  & \multicolumn{1}{c|}{$ \Delta_{i}^{\mathrm{v}} $}  & $ 0.05\times 1.1^{j} $ \\ \hline
		\end{tabular}
		\caption{Case 4(b)}
		\label{Case4b}
	\end{minipage}
\end{table*}

\noindent
\underline{\bf Case  4(c)}: $  {\Delta_{i}}/{2}\geq \Delta_{i}^{\mathrm{v}}  $ and
\begin{align}
	H_{\mathrm{VA}} &= \frac{1}{\min\{\frac{\Delta_{i^\star}}{2},\Delta_{i^\star}^{\mathrm{v}}\}^2} +
	\sum_{i\in\bar{\mathcal{F}}^c\cap\mathcal{S}}\frac{1}{\max\{\frac{\Delta_i}{2},\Delta_i^{\mathrm{v}}\}^2}\\
	&
	=\left(\frac{\Delta_{i^\star}}{2}\right)^{-2}+\sum_{i\in\bar{\mathcal{F}}^c\cap\mathcal{S}}\left(\frac{\Delta_{i}}{2}\right)^{-2}  .
\end{align}
    The parameters that are varied are $ \Delta_i$ for all $i\in[N] $. See Table~\ref{Case4c} for details.

\noindent
\underline{\bf Case  4(d)}: $  {\Delta_{i}}/{2}\geq \Delta_{i}^{\mathrm{v}}  $ and
\begin{align}
	H_{\mathrm{VA}} &= \frac{1}{\min\{\frac{\Delta_{i^\star}}{2},\Delta_{i^\star}^{\mathrm{v}}\}^2 }+
	\sum_{i\in\bar{\mathcal{F}}^c\cap\mathcal{S}}\frac{1}{\max\{\frac{\Delta_i}{2},\Delta_i^{\mathrm{v}}\}^2}\\
	&
	=(\Delta_{i^\star}^{\mathrm{v}} )^{-2}+\sum_{i\in\bar{\mathcal{F}}^c\cap\mathcal{S}}\left(\frac{\Delta_{i}}{2}\right)^{-2}  .
\end{align}
 The parameters that are varied are $ \Delta_i^{\mathrm{v}}$ for all $ i\in\bar{\mathcal{F}}^c\cap\mathcal{S}$. See Table~\ref{Case4d} for details.

\begin{table*}[t]
	\begin{minipage}{0.49\linewidth}
		\centering
		\begin{tabular}{cccccc}
			\hline
			\multicolumn{6}{c}{$ \bar{\sigma}^2=0.2 ,N=20$ }      \\ \hline
			\multicolumn{1}{c|}{$ \mu_{i^\star} $} & \multicolumn{1}{c|}{$ 0.7 $} & \multicolumn{1}{c|}{$ \mu_{i^{\star\star}} $} & \multicolumn{1}{c|}{$ \mu_{i^\star} - \Delta_{i^\star}$} & \multicolumn{1}{c|}{$ \mu_{i} $} & $ \mu_{i^\star} - \Delta_{i}  $\\ \hline
			\multicolumn{1}{c|}{$ \Delta_{i^\star} $} & \multicolumn{1}{c|}{$  \Delta_i $}  & \multicolumn{1}{c|}{$ \Delta_{i^{\star\star}} $}  & \multicolumn{1}{c|}{ $ \Delta_{i^\star} $}  & \multicolumn{1}{c|}{$ \Delta_{i} $}  &  $ 0.09\times 1.1^{j} $ \\ \hline
			\multicolumn{1}{c|}{$ \sigma_{i^\star}^2 $} & \multicolumn{1}{c|}{$ 0.04 $}  & \multicolumn{1}{c|}{$ \sigma_{i^{\star\star}}^2 $}  & \multicolumn{1}{c|}{$ 0.21$}  & \multicolumn{1}{c|}{$ \sigma_{i}^2 $}  & $ 0.21$  \\ \hline
			\multicolumn{1}{c|}{$ \Delta_{i^\star}^{\mathrm{v}} $} & \multicolumn{1}{c|}{$ 0.16 $}  & \multicolumn{1}{c|}{$ \Delta_{i^{\star\star}}^{\mathrm{v}} $}  & \multicolumn{1}{c|}{$ 0.01$}  & \multicolumn{1}{c|}{$ \Delta_{i}^{\mathrm{v}} $}  &  $ 0.01 $ \\ \hline
		\end{tabular}
		\caption{Case 4(c)}
		\label{Case4c}
	\end{minipage}
	\
	\begin{minipage}{0.49\linewidth}
		\centering
		\begin{tabular}{cccccc}
			\hline
			\multicolumn{6}{c}{$ \bar{\sigma}^2=0.04 ,N=20$ }      \\ \hline
			\multicolumn{1}{c|}{$ \mu_{i^\star} $} & \multicolumn{1}{c|}{$ 0.7 $} & \multicolumn{1}{c|}{$ \mu_{i^{\star\star}} $} & \multicolumn{1}{c|}{$ 0.3$} & \multicolumn{1}{c|}{$ \mu_{i} $} & $ 0.3$\\ \hline
			\multicolumn{1}{c|}{$ \Delta_{i^\star} $} & \multicolumn{1}{c|}{$  0.4$}  & \multicolumn{1}{c|}{$ \Delta_{i^{\star\star}} $}  & \multicolumn{1}{c|}{ $ 0.4$}  & \multicolumn{1}{c|}{$ \Delta_i $}  &  $ 0.4$ \\ \hline
			\multicolumn{1}{c|}{$ \sigma_{i^\star}^2 $} & \multicolumn{1}{c|}{$ 0.03 $}  & \multicolumn{1}{c|}{$ \sigma_{i^{\star\star}}^2 $}  & \multicolumn{1}{c|}{$ \sigma_{i}^2$}  & \multicolumn{1}{c|}{$ \sigma_{i}^2 $}  & $ \bar{\sigma}^2+\Delta_{i}^{\mathrm{v}} $  \\ \hline
			\multicolumn{1}{c|}{$ \Delta_{i^\star}^{\mathrm{v}} $} & \multicolumn{1}{c|}{$ 0.01 $}  & \multicolumn{1}{c|}{$ \Delta_{i^{\star\star}}^{\mathrm{v}} $}  & \multicolumn{1}{c|}{$  \Delta_{i}^{\mathrm{v}}$}  & \multicolumn{1}{c|}{$ \Delta_{i}^{\mathrm{v}} $}  & $ 0.01\times 1.2^{j} $\\ \hline
		\end{tabular}
		\caption{Case 4(d)}
		\label{Case4d}
	\end{minipage}
\end{table*}

\subsection{Additional Experimental Results for VA-LUCB}\label{experiment figure}
The plots of the time complexities of Case 4(b) and Case 4(c) with respect to the corresponding $ H_\mathrm{VA}\ln({H_\mathrm{VA}}/{\delta}) $ are shown in Figure \ref{addtionalplots}.
\begin{figure*}[t]
	\centering
	\subfigure[Case 4(b)]{
		\includegraphics[width=.45\textwidth]{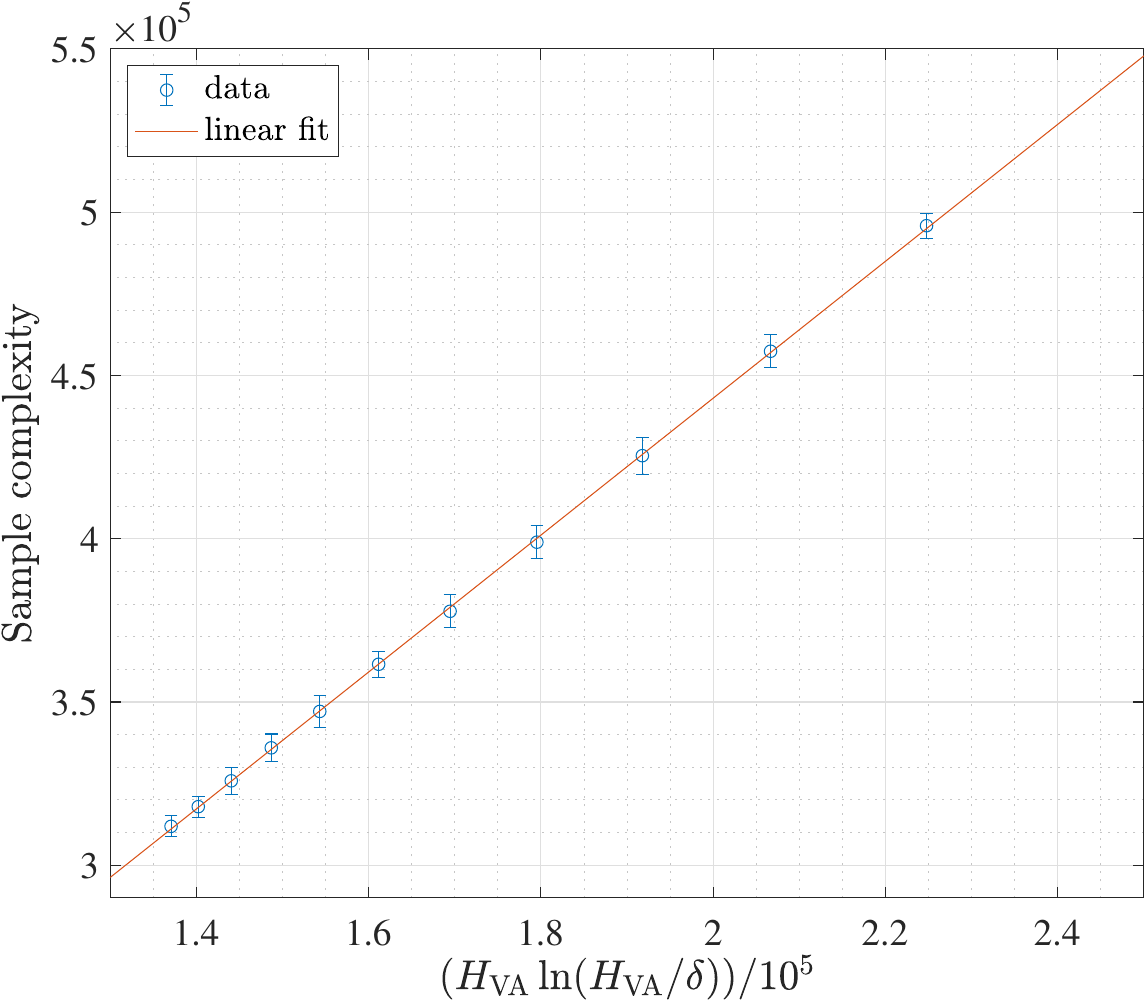}
	} 
	\subfigure[Case 4(c)]{
		\includegraphics[width=.45\textwidth]{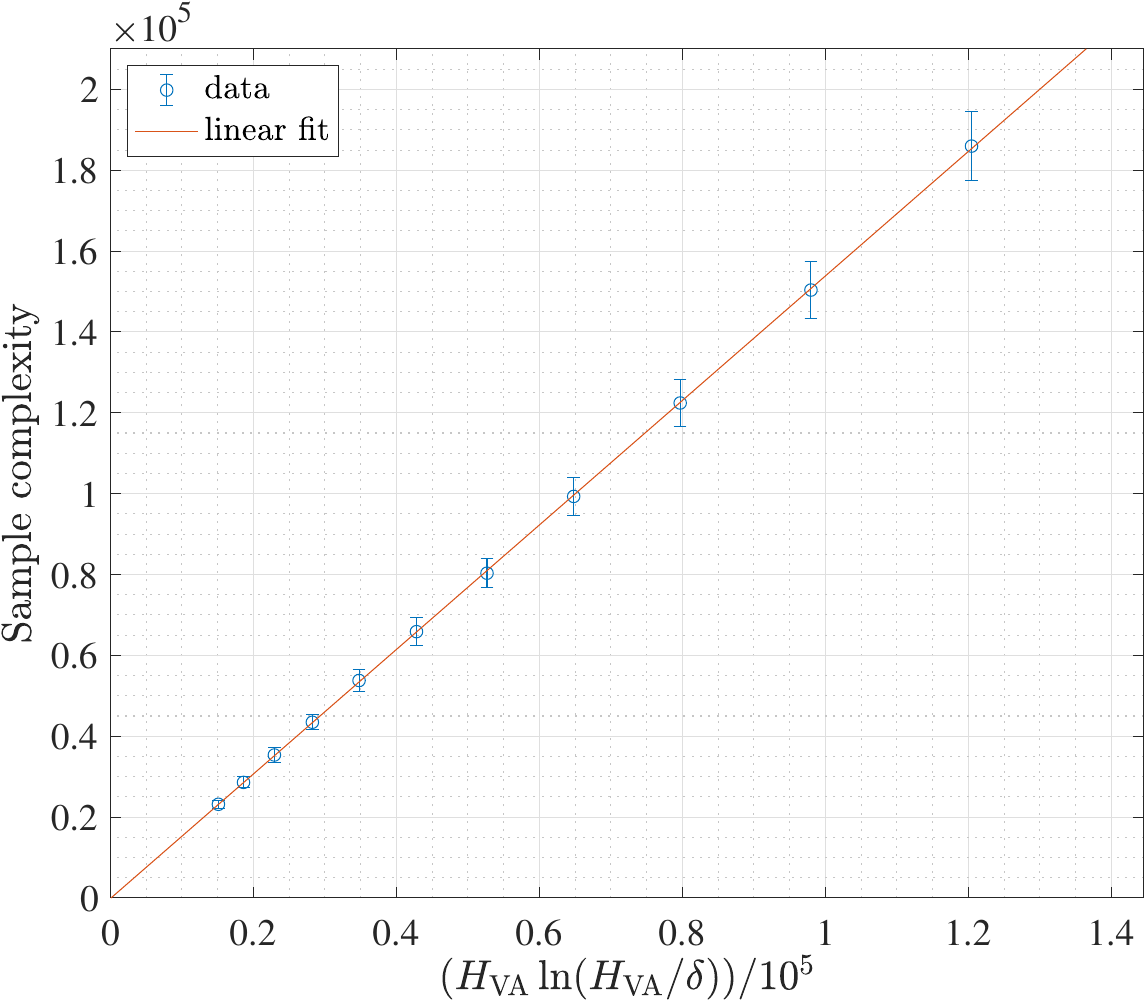}
	} 
	\caption{The time complexities of cases 4(b) and 4(c) with respect to $ H_{\mathrm{VA}}\ln({H_{\mathrm{VA}}}/{\delta}) $ }
	\label{addtionalplots}
\end{figure*}
For Cases 1(b), 4(a) and 4(d), the time complexities in each of these cases are expected to remain the same as the instances (and hence, hardness) vary. This is corroborated by the experimental results which are displayed in Table~\ref{3unchangedcases}.

We remark that some of the terms are correlated, e.g., $ \Delta_{i^\star} $ and $ \Delta_{i^{\star\star}} $. Hence, we sometimes have to make a compromise by changing the parameters of arms in other terms, like in Case 1(d) in which when  $   \Delta_{i^\star}$ increases, $ \Delta_{i^{\star\star}}  $ also increases. Thus, the decrease in sample complexity in this case results from $ \Delta_{i^{\star\star}}^{-2} $ and not from $ \min\{{\Delta_{i^\star}/2,\Delta_{i^\star}^{\mathrm{v}}}\}^{-2} $.
\begin{table*}[t]
	\centering
	\begin{tabular}{|c|cc|cc|cc|}
		\hline
		& \multicolumn{2}{c|}{Case 1(b)}           & \multicolumn{2}{c|}{Case 4(a)}           & \multicolumn{2}{c|}{Case 4(d)}           \\ \hline
		instance & \multicolumn{1}{c|}{TC}       & STD      & \multicolumn{1}{c|}{TC}       & STD      & \multicolumn{1}{c|}{TC}       & STD      \\ \hline
		0        & \multicolumn{1}{c|}{$ 5.648\times 10^5 $} & $ 4.185\times 10^4 $ & \multicolumn{1}{c|}{$3.026\times 10^5$} & $4.770\times 10^3 $& \multicolumn{1}{c|}{$2.895\times 10^5$} & $4.303\times 10^3$ \\ \hline
		1        & \multicolumn{1}{c|}{$5.585\times 10^5$} & $6.713\times 10^4$ & \multicolumn{1}{c|}{$3.028\times 10^5$} & $4.774\times 10^3$ & \multicolumn{1}{c|}{$2.896\times 10^5$} &$ 4.336\times 10^3 $\\ \hline
		2        & \multicolumn{1}{c|}{$5.531\times 10^5$} & $6.741\times 10^4$ & \multicolumn{1}{c|}{$3.031\times 10^5$} & $4.861\times 10^3 $& \multicolumn{1}{c|}{$2.896\times 10^5$} & $4.173\times 10^3$ \\ \hline
		3        & \multicolumn{1}{c|}{$5.666\times 10^5$} & $8.515\times 10^4$ & \multicolumn{1}{c|}{$3.029\times 10^5$} & $4.426\times 10^3$ & \multicolumn{1}{c|}{$2.892\times 10^5$} & $4.209\times 10^3$ \\ \hline
		4        & \multicolumn{1}{c|}{$5.787\times 10^5$} & $7.152\times 10^4$ & \multicolumn{1}{c|}{$3.029\times 10^5$} & $4.932\times 10^3$ & \multicolumn{1}{c|}{$2.896\times 10^5$} & $4.019\times 10^3$ \\ \hline
		5        & \multicolumn{1}{c|}{$5.733\times 10^5$} & $7.558\times 10^4$ & \multicolumn{1}{c|}{$3.030\times 10^5$} & $5.125\times 10^3 $& \multicolumn{1}{c|}{$2.894\times 10^5$} & $4.224\times 10^3$ \\ \hline
		6        & \multicolumn{1}{c|}{$5.578\times 10^5$} & $7.112\times 10^4$ & \multicolumn{1}{c|}{$3.031\times 10^5$} & $4.887\times 10^3$ & \multicolumn{1}{c|}{$2.894\times 10^5$} & $4.471\times 10^3$ \\ \hline
		7        & \multicolumn{1}{c|}{$5.596\times 10^5$} & $4.398\times 10^4$ & \multicolumn{1}{c|}{$3.031\times 10^5$} &$5.059\times 10^3$ & \multicolumn{1}{c|}{$2.897\times 10^5$} & $4.134\times 10^3$ \\ \hline
		8        & \multicolumn{1}{c|}{$5.630\times 10^5$} & $5.354\times 10^4$ & \multicolumn{1}{c|}{$3.031\times 10^5$} & $5.074\times 10^3$ & \multicolumn{1}{c|}{$2.896\times 10^5$} & $4.022\times 10^3$ \\ \hline
		9        & \multicolumn{1}{c|}{$5.726\times 10^5$} & $7.683\times 10^4$ & \multicolumn{1}{c|}{$3.031\times 10^5$} & $4.891\times 10^3$ & \multicolumn{1}{c|}{$2.895\times 10^5$} & $4.157\times 10^3 $\\ \hline
		10       & \multicolumn{1}{c|}{$5.767\times 10^5$} & $4.815\times 10^4 $& \multicolumn{1}{c|}{$3.029\times 10^5$} & $4.839\times 10^3 $& \multicolumn{1}{c|}{$2.895\times 10^5$} & $4.150\times 10^3$ \\ \hline
	\end{tabular}
	\caption{The time complexities of Cases 1(b), 4(a) and 4(d). ``TC'' and ``STD'' are short for ``Time Complexity'' and ``Standard Deviation'' respectively. The time complexities are almost constant across instances. }
	\label{3unchangedcases}
\end{table*}



\bibliographystyle{IEEEtran}
\bibliography{reference_VALUCB}

\begin{thebibliography}{10}
\providecommand{\url}[1]{#1}
\csname url@samestyle\endcsname
\providecommand{\newblock}{\relax}
\providecommand{\bibinfo}[2]{#2}
\providecommand{\BIBentrySTDinterwordspacing}{\spaceskip=0pt\relax}
\providecommand{\BIBentryALTinterwordstretchfactor}{4}
\providecommand{\BIBentryALTinterwordspacing}{\spaceskip=\fontdimen2\font plus
\BIBentryALTinterwordstretchfactor\fontdimen3\font minus
  \fontdimen4\font\relax}
\providecommand{\BIBforeignlanguage}[2]{{%
\expandafter\ifx\csname l@#1\endcsname\relax
\typeout{** WARNING: IEEEtran.bst: No hyphenation pattern has been}%
\typeout{** loaded for the language `#1'. Using the pattern for}%
\typeout{** the default language instead.}%
\else
\language=\csname l@#1\endcsname
\fi
#2}}
\providecommand{\BIBdecl}{\relax}
\BIBdecl

\bibitem{David2018}
Y.~David, B.~Sz{\"o}r{\'e}nyi, M.~Ghavamzadeh, S.~Mannor, and N.~Shimkin,
  ``{PAC} bandits with risk constraints,'' in \emph{Proceedings of the
  International Symposium on Artificial Intelligence and Mathematics (ISAIM)},
  2018.

\bibitem{lattimore2020bandit}
T.~Lattimore and C.~Szepesv{\'a}ri, \emph{Bandit algorithms}.\hskip 1em plus
  0.5em minus 0.4em\relax Cambridge University Press, 2020.

\bibitem{Cassel2018}
A.~Cassel, S.~Mannor, and A.~Zeevi, ``A general approach to multi-armed bandits
  under risk criteria,'' in \emph{Proceedings of the 31st Conference on
  Learning Theory}, ser. Proceedings of Machine Learning Research,
  vol.~75.\hskip 1em plus 0.5em minus 0.4em\relax PMLR, 2018, pp. 1295--1306.

\bibitem{lee2020learning}
J.~Lee, S.~Park, and J.~Shin, ``Learning bounds for risk-sensitive learning,''
  in \emph{Advances in Neural Information Processing Systems}, vol.~33, 2020,
  pp. 13\,867--13\,879.

\bibitem{chang2021unifying}
J.~Q.~L. Chang and V.~Y.~F. Tan, ``A unifying theory of {Thompson} sampling for
  continuous risk-averse bandits,'' in \emph{Proceedings of the 36th AAAI
  Conference on Artificial Intelligence (AAAI)}, 2022.

\bibitem{Even2006action}
E.~Even-Dar, S.~Mannor, and Y.~Mansour, ``Action elimination and stopping
  conditions for the multi-armed bandit and reinforcement learning problems,''
  \emph{Journal of Machine Learning Research}, vol.~7, p. 1079–1105, 2006.

\bibitem{Audibert2010}
J.-Y. Audibert, S.~Bubeck, and R.~Munos, ``Best arm identification in
  multi-armed bandits,'' in \emph{23th Conference on Learning Theory}, 2010,
  pp. 41--53.

\bibitem{jamieson14lil}
K.~Jamieson, M.~Malloy, R.~Nowak, and S.~Bubeck, ``lil'{UCB}: An optimal
  exploration algorithm for multi-armed bandits,'' in \emph{Proceedings of the
  27th Conference on Learning Theory}, ser. Proceedings of Machine Learning
  Research, vol.~35.\hskip 1em plus 0.5em minus 0.4em\relax Barcelona, Spain:
  PMLR, 2014, pp. 423--439.

\bibitem{Kalyanakrishnan2012}
S.~Kalyanakrishnan, A.~Tewari, P.~Auer, and P.~Stone, ``{PAC} subset selection
  in stochastic multi-armed bandits.'' in \emph{Proceedings of the 29th
  International Conference on Machine Learning}.\hskip 1em plus 0.5em minus
  0.4em\relax PMLR, 2012, pp. 227--234.

\bibitem{Kaufmann2016}
E.~Kaufmann, O.~Capp\'{e}, and A.~Garivier, ``On the complexity of best-arm
  identification in multi-armed bandit models,'' \emph{Journal of Machine
  Learning Research}, vol.~17, no.~1, p. 1–42, 2016.

\bibitem{russo2016simple}
D.~Russo, ``Simple {Bayesian} algorithms for best arm identification,'' in
  \emph{29th Annual Conference on Learning Theory}, ser. Proceedings of Machine
  Learning Research, vol.~49.\hskip 1em plus 0.5em minus 0.4em\relax PMLR,
  23--26 Jun 2016, pp. 1417--1418.

\bibitem{Jamieson2014}
K.~Jamieson and R.~Nowak, ``Best-arm identification algorithms for multi-armed
  bandits in the fixed confidence setting,'' in \emph{48th Annual Conference on
  Information Sciences and Systems (CISS)}.\hskip 1em plus 0.5em minus
  0.4em\relax IEEE, 2014, pp. 1--6.

\bibitem{howard2021time}
S.~R. Howard, A.~Ramdas, J.~McAuliffe, and J.~Sekhon, ``Time-uniform,
  nonparametric, nonasymptotic confidence sequences,'' \emph{The Annals of
  Statistics}, vol.~49, no.~2, pp. 1055--1080, 2021.

\bibitem{Sani2013}
A.~Sani, A.~Lazaric, and R.~Munos, ``Risk-aversion in multi-armed bandits,'' in
  \emph{Proceedings of the 25th International Conference on Neural Information
  Processing Systems}.\hskip 1em plus 0.5em minus 0.4em\relax Curran Associates
  Inc., 2012, p. 3275–3283.

\bibitem{Vakili2016}
S.~Vakili and Q.~Zhao, ``Risk-averse multi-armed bandit problems under
  mean-variance measure,'' \emph{IEEE Journal of Selected Topics in Signal
  Processing}, vol.~10, no.~6, pp. 1093--1111, 2016.

\bibitem{Zhu2020}
Q.~Zhu and V.~Y.~F. Tan, ``Thompson sampling algorithms for mean-variance
  bandits,'' in \emph{Proceedings of the 37th International Conference on
  Machine Learning}.\hskip 1em plus 0.5em minus 0.4em\relax PMLR, 2020, pp.
  11\,599--11\,608.

\bibitem{chang2021risk}
J.~Q.~L. Chang, Q.~Zhu, and V.~Y.~F. Tan, ``Risk-constrained {Thompson}
  sampling for {CVaR} bandits,'' \emph{arXiv preprint arXiv:2011.08046}, 2020.

\bibitem{Zimin2014}
A.~Zimin, R.~Ibsen-Jensen, and K.~Chatterjee, ``Generalized risk-aversion in
  stochastic multi-armed bandits,'' \emph{arXiv preprint arXiv:1405.0833},
  2014.

\bibitem{LA2020}
L.~A. Prashanth, K.~Jagannathan, and R.~Kolla, ``Concentration bounds for
  {CV}a{R} estimation: The cases of light-tailed and heavy-tailed
  distributions,'' in \emph{Proceedings of the 37th International Conference on
  Machine Learning}, vol. 119.\hskip 1em plus 0.5em minus 0.4em\relax PMLR,
  2020, pp. 5577--5586.

\bibitem{Kagrecha2019}
A.~Kagrecha, J.~Nair, and K.~Jagannathan, ``Distribution oblivious, risk-aware
  algorithms for multi-armed bandits with unbounded rewards,'' in
  \emph{Proceedings of the 33rd International Conference on Neural Information
  Processing Systems}, vol.~32.\hskip 1em plus 0.5em minus 0.4em\relax Curran
  Associates Inc., 2019, pp. 11\,272--11\,281.

\bibitem{kagrecha2022statistically}
------, ``Statistically robust, risk-averse best arm identification in
  multi-armed bandits,'' \emph{IEEE Transactions on Information Theory},
  vol.~68, no.~8, pp. 5248--5267, 2022.

\bibitem{David2016}
Y.~David and N.~Shimkin, ``Pure exploration for max-quantile bandits,'' in
  \emph{Machine Learning and Knowledge Discovery in Databases}.\hskip 1em plus
  0.5em minus 0.4em\relax Springer, 2016, pp. 556--571.

\bibitem{Kagrecha2020}
A.~Kagrecha, J.~Nair, and K.~Jagannathan, ``Constrained regret minimization for
  multi-criterion multi-armed bandits,'' \emph{arXiv preprint
  arXiv:2006.09649}, 2020.

\bibitem{baudry2020thompson}
D.~Baudry, R.~Gautron, E.~Kaufmann, and O.~Maillard, ``Optimal {Thompson}
  sampling strategies for support-aware {CVaR} bandits,'' in \emph{Proceedings
  of the 38th International Conference on Machine Learning}, vol. 139, 18--24
  Jul 2021, pp. 716--726.

\bibitem{EvenDar2006}
E.~Even-Dar, M.~Kearns, and J.~Wortman, ``Risk-sensitive online learning,'' in
  \emph{International Conference on Algorithmic Learning Theory}.\hskip 1em
  plus 0.5em minus 0.4em\relax Springer, 2006, pp. 199--213.

\bibitem{Maillard2013}
O.-A. Maillard, ``Robust risk-averse stochastic multi-armed bandits,'' in
  \emph{International Conference on Algorithmic Learning Theory}.\hskip 1em
  plus 0.5em minus 0.4em\relax Springer, 2013, pp. 218--233.

\bibitem{Chang2020}
H.~S. Chang, ``An asymptotically optimal strategy for constrained multi-armed
  bandit problems,'' \emph{Mathematical Methods of Operations Research},
  vol.~91, no.~3, pp. 545--557, 2020.

\bibitem{Wu2016}
Y.~Wu, R.~Shariff, T.~Lattimore, and C.~Szepesv{\'a}ri, ``Conservative
  bandits,'' in \emph{Proceedings of the 33rd International Conference on
  Machine Learning}, vol.~48.\hskip 1em plus 0.5em minus 0.4em\relax PMLR,
  2016, pp. 1254--1262.

\bibitem{Amani2019}
S.~Amani, M.~Alizadeh, and C.~Thrampoulidis, ``Linear stochastic bandits under
  safety constraints,'' in \emph{Proceedings of the 33rd International
  Conference on Neural Information Processing Systems}, vol.~32, 2019, pp.
  9256--9266.

\bibitem{tan2022}
V.~Y.~F. Tan, L.~A. Prashanth, and K.~Jagannathan, ``A survey of risk-aware
  multi-armed bandits,'' in \emph{Proceedings of 31st International Joint
  Conference on Artificial Intelligence}, Vienna, Austria, July 2022.

\bibitem{auer16}
P.~Auer, C.-K. Chiang, R.~Ortner, and M.~Drugan, ``Pareto front identification
  from stochastic bandit feedback,'' in \emph{Proceedings of the 19th
  International Conference on Artificial Intelligence and Statistics}, ser.
  Proceedings of Machine Learning Research, vol.~51.\hskip 1em plus 0.5em minus
  0.4em\relax Cadiz, Spain: PMLR, 09--11 May 2016, pp. 939--947.

\bibitem{turgay18a}
E.~Turgay, D.~Oner, and C.~Tekin, ``Multi-objective contextual bandit problem
  with similarity information,'' in \emph{Proceedings of the 21st International
  Conference on Artificial Intelligence and Statistics}, ser. Proceedings of
  Machine Learning Research, vol.~84.\hskip 1em plus 0.5em minus 0.4em\relax
  PMLR, 09--11 Apr 2018, pp. 1673--1681.

\bibitem{Zuluaga}
M.~Zuluaga, A.~Krause, and M.~P{{{\"u}}}schel, ``$\epsilon$-{PAL}: An active
  learning approach to the multi-objective optimization problem,''
  \emph{Journal of Machine Learning Research}, vol.~17, no. 104, pp. 1--32,
  2016.

\bibitem{samuels19a}
J.~Katz-Samuels and C.~Scott, ``Top feasible arm identification,'' in
  \emph{Proceedings of the 22nd International Conference on Artificial
  Intelligence and Statistics}, ser. Proceedings of Machine Learning Research,
  vol.~89.\hskip 1em plus 0.5em minus 0.4em\relax PMLR, 16--18 Apr 2019, pp.
  1593--1601.

\bibitem{Lu}
P.~Lu, C.~Tao, and X.~Zhang, ``Variance-dependent best arm identification,'' in
  \emph{Proceedings of the 37th Conference on Uncertainty in Artificial
  Intelligence}, ser. Proceedings of Machine Learning Research, vol. 161.\hskip
  1em plus 0.5em minus 0.4em\relax PMLR, 27--30 Jul 2021, pp. 1120--1129.

\bibitem{faella2020rapidly}
M.~Faella, A.~Finzi, and L.~Sauro, ``Rapidly finding the best arm using
  variance,'' in \emph{24th European Conference on Artificial
  Intelligence}.\hskip 1em plus 0.5em minus 0.4em\relax IOS Press, 2020, pp.
  2585--2591.

\bibitem{Bhat2019}
S.~P. Bhat and L.~A. Prashanth, ``Concentration of risk measures: a
  {Wasserstein} distance approach,'' in \emph{Proceedings of the 33rd
  International Conference on Neural Information Processing Systems},
  vol.~32.\hskip 1em plus 0.5em minus 0.4em\relax Curran Associates, Inc.,
  2019, pp. 11\,762--11\,771.

\bibitem{simchowitz17the}
M.~Simchowitz, K.~Jamieson, and B.~Recht, ``The simulator: Understanding
  adaptive sampling in the moderate-confidence regime,'' in \emph{Proceedings
  of the 2017 Conference on Learning Theory}, ser. Proceedings of Machine
  Learning Research, vol.~65.\hskip 1em plus 0.5em minus 0.4em\relax PMLR,
  2017, pp. 1794--1834.

\bibitem{tanczos2017a}
E.~T\'anczos, R.~Nowak, and B.~Mankoff, ``A {KL-LUCB} algorithm for large-scale
  crowdsourcing,'' in \emph{Advances in Neural Information Processing Systems},
  vol.~30.\hskip 1em plus 0.5em minus 0.4em\relax Curran Associates, Inc.,
  2017.

\bibitem{Garivier2016}
A.~Garivier and E.~Kaufmann, ``Optimal best arm identification with fixed
  confidence,'' in \emph{29th Conference on Learning Theory}, vol.~49.\hskip
  1em plus 0.5em minus 0.4em\relax PMLR, 2016, pp. 998--1027.

\bibitem{mcdiarmid1989}
C.~McDiarmid, ``On the method of bounded differences,'' \emph{Surveys in
  Combinatorics}, vol. 141, no.~1, pp. 148--188, 1989.

\bibitem{duchi}
J.~Duchi, \emph{Lecture Notes for Statistics 311/Electrical Engineering
  377}.\hskip 1em plus 0.5em minus 0.4em\relax Stanford University, 2016.

\bibitem{honorio2014}
J.~Honorio and T.~Jaakkola, ``Tight bounds for the expected risk of linear
  classifiers and {PAC-Bayes} finite-sample guarantees,'' in \emph{17th
  International Conference on Artificial Intelligence and Statistics
  (AISTATS)}, 2014, pp. 384--392.

\bibitem{Friedrich2019}
F.~G\"otze, H.~Sambale, and A.~Sinulis, ``{Higher order concentration for
  functions of weakly dependent random variables},'' \emph{Electronic Journal
  of Probability}, vol.~24, pp. 1 -- 19, 2019.

\end{thebibliography}

\begin{IEEEbiographynophoto}{Yunlong Hou}  received the B.S.\ degree from Beijing Normal University in 2020. He is currently pursuing the Ph.D.\  degree at  the Department of Mathematics, National University of Singapore (NUS). His research interests focus on machine learning, e.g., online learning.
\end{IEEEbiographynophoto}

\begin{IEEEbiographynophoto}{Vincent Y.\ F.\ Tan} (S'07-M'11-SM'15)  was born in Singapore in 1981. He is currently an Associate Professor in the Department of Mathematics and the  Department of Electrical and Computer Engineering at the National University of Singapore (NUS). He received the B.A.\ and M.Eng.\ degrees in Electrical and Information Sciences from Cambridge University in 2005 and the Ph.D.\ degree in Electrical Engineering and Computer Science (EECS) from the Massachusetts Institute of Technology (MIT)  in 2011.  His research interests include network information theory, machine learning, and statistical signal processing.

Dr.\ Tan received the MIT EECS Jin-Au Kong outstanding doctoral thesis prize in 2011, the NUS Young Investigator Award in 2014, the Singapore National Research Foundation (NRF) Fellowship (Class of 2018) and the NUS Young Researcher Award in 2019. He was also an IEEE Information Theory Society Distinguished Lecturer for 2018/9. He is currently serving as a Senior Area Editor of the {\em IEEE Transactions on Signal Processing} and an Associate Editor of Machine Learning for the {\em IEEE Transactions on Information Theory}. He is a member of the IEEE Information Theory Society Board of Governors.
\end{IEEEbiographynophoto}

\begin{IEEEbiographynophoto}{Zixin Zhong} was born in China in 1995. She is currently a postdoctoral fellow at the Department of Computing Science of University of Alberta (UofA). She is supervised by Prof.\ Csaba Szepesv{\'a}ri. Dr.\ Zhong received her PhD degree from the Department of Mathematics of National University of Singapore (NUS) in October 2021. Dr.\ Zhong was privileged to be supervised by Prof.\ Vincent Y.\ F.\ Tan and Prof.\ Wang Chi Cheung during her Ph.D.\ study, and she worked with them as a research fellow between June 2021 and July 2022.

Dr.\ Zhong's research interests are in reinforcement learning, online machine learning and, in particular,  multi-armed bandits. Her work has been presented at top machine learning (ML) conferences including ICML and AISTATS, and also in top  journals such as the {\em Journal of Machine Learning Research} (JMLR) and the {\em IEEE Transactions on Wireless Communications} (TWC). She also serves as a reviewer for several conferences and journals including   AISTATS, ICLR, ICML, NeurIPS, TIT, TSP, and TMLR. She was selected as a top reviewer for NeurIPS 2022.
\end{IEEEbiographynophoto}

\vfill

\end{document}